\documentclass{article}

\newif\ifdraft

\newif\ifapxAppended
\apxAppendedtrue

\newif\ifonecolumn
\onecolumntrue

\usepackage{amsthm} 
\usepackage{amssymb}
\usepackage{amsmath}
\usepackage{stmaryrd}
\usepackage{mathtools}
\usepackage{MnSymbol}
\usepackage{relsize} 
\usepackage{listings} 
\usepackage[style=british]{csquotes} 
\usepackage{wrapfig}
\usepackage{tikz}
\usetikzlibrary{trees}
\usetikzlibrary{calc}
\usepackage{calculation}
\usepackage{xr-hyper} 
\usepackage{widetext}
\usepackage{bussproofs} 
\usepackage{paralist}
\makeatletter
\newcommand{\algrule}[1][.2pt]{\par\vskip.1\baselineskip\hrule height #1\par\vskip.1\baselineskip}
\makeatother

\ifdraft
\usepackage[draft,nompar]{commenting}
\newcommand\lochanged[1]{{\color{red!60!yellow} #1}}
\else
\usepackage[final]{commenting}
\newcommand\lochanged[1]{}
\fi
\declareauthor{lo}{lo}{red!60!yellow}
\authorcommand{lo}{comment}
\declareauthor{cm}{cm}{green!50!black}
\authorcommand{cm}{comment}
\setdefaultauthor{cm}
\declareauthor{fz}{fz}{purple}
\authorcommand{fz}{comment}
\declareauthor{anon}{anon}{blue}
\authorcommand{anon}{comment}

\usepackage{hyperref}

\usepackage[capitalise]{cleveref} 

\Crefname{theorem}{Thm.}{Thm.}
\Crefname{corollary}{Cor.}{Corollary}
\crefname{proposition}{Prop.}{Propositions}
\Crefname{claim}{Claim}{Claims}
\Crefname{definition}{Def.}{Definitions}
\Crefname{fact}{Fact}{Facts}
\Crefname{conjecture}{Conj.}{Conjectures}
\Crefname{example}{Ex.}{Ex.}
\Crefname{remark}{Rem.}{Remarks}
\Crefname{convention}{Convention}{Conventions}
\Crefname{lemma}{Lem.}{Lemmas}
\Crefname{assumption}{Ass.~}{Ass.~}
\Crefname{section}{Sec.}{Sec.}
\Crefname{appendix}{App.}{App.}
\Crefname{figure}{Fig.}{Fig.}
\Crefname{algorithm}{Alg.}{Alg.}

\usepackage{thmtools}
\usepackage{thm-restate}

\DeclareUnicodeCharacter{0301}{\'{e}}
\theoremstyle{definition}
\newtheorem{definition}{Definition}
\newtheorem{example}[definition]{Example}

\newtheorem*{assumption*}{Assumption}

\theoremstyle{plain}
\newtheorem{theorem}[definition]{Theorem}
\newtheorem{proposition}[theorem]{Proposition}
\newtheorem{lemma}[theorem]{Lemma}

\theoremstyle{remark}
\newtheorem{remark}[definition]{Remark}

\newcommand\aref[3][blue]{%
  \begingroup%
  \hypersetup{linkcolor=#1}%
  \hyperlink{#2}{#3}%
  \endgroup
}

\usepackage{url}

\usepackage{pgfplots}

\usepackage[normalem]{ulem}
\newcommand\btrfunction{{TR function}}

\newcommand{\yhl}[1]{\colorbox{yellow!50}{\hspace{-3pt}#1\hspace{-3pt}}}
\newcommand{\ghl}[1]{\colorbox{green!20}{\hspace{-3pt}#1\hspace{-3pt}}}

\makeatletter
\newcommand\ztag[1]{%
\def\@currentlabel{\footnotesize #1}%
\gdef\tmp{%
\addtocounter{equation}{-1}%
\def\theequation{\footnotesize #1}}%
\aftergroup\aftergroup\aftergroup\aftergroup\aftergroup\aftergroup
\aftergroup\aftergroup\aftergroup\aftergroup\aftergroup\aftergroup
\aftergroup\aftergroup\aftergroup\aftergroup\aftergroup\aftergroup
\aftergroup\aftergroup\aftergroup\aftergroup\aftergroup\aftergroup
\aftergroup\aftergroup\aftergroup\aftergroup\aftergroup\aftergroup
\aftergroup
\tmp}

\newcommand{\Nat}{\mathbb{N}}
\newcommand{\Real}{\mathbb{R}}
\newcommand{\pReal}{\Real_{\geq 0}}

\newcommand{\Borel}{\mathcal{B}}

\newcommand{\charfn}[1]{\boldsymbol{1}_{#1}}
\newcommand{\inv}[1]{{#1}^{-1}}
\newcommand{\partialto}{\rightharpoonup}

\newcommand{\proj}[1]{\mathsf{proj}_{#1}}

\newcommand\dif{\textrm{d}}
\newcommand{\concat}{\mathbin{+\mkern-8mu+}}

\newcommand{\MyCase}[3]{
  \begin{cases}
    #1 & \text{if } #2,\\
    #3 & \text{otherwise.}
  \end{cases}
}

\newcommand{\interior}[1]{\mathring{#1}}

\newcommand{\domain}[1]{\mathsf{Dom}(#1)}
\newcommand{\support}[1]{\mathsf{Supp}(#1)}
\newcommand{\nsupport}[2]{\mathsf{Supp}^{#2}(#1)}
\newcommand{\comp}[1]{{#1}^c}
\newcommand{\closure}[1]{\overline{#1}}

\newcommand{\boundary}[1]{{\partial #1}}

\DeclarePairedDelimiter\norm{\lVert}{\rVert}%
\DeclarePairedDelimiter\bignorm{\big\lVert}{\big\rVert}%
\newcommand{\sign}{\mathrm{sign}}
\newcommand{\expint}[4]{
  \int_{#1}\ {#2}\ {#3}(\dif{#4})
}
\newcommand{\shortint}[3]{
  \int_{#1}\ {#2}\ \dif{#3}
}
\newcommand{\defn}[1]{\textbf{\em #1}}

\SetSymbolFont{stmry}{bold}{U}{stmry}{m}{n} 

\renewcommand\vec[1]{\boldsymbol{#1}}

\newcommand{\grad}[1]{{\nabla #1}}
\newcommand{\idmat}{\boldsymbol{I}}

\newcommand{\anbr}[1]{\langle #1\rangle}

\newcommand{\set}[1]{\{#1\}}

\DeclarePairedDelimiter\church{\llceil}{\rrceil}

\newcommand{\leb}{\mathsf{Leb}}
\newcommand{\measure}[1]{\mu_{#1}}
\newcommand{\pdf}[1]{\mathsf{pdf}_{#1}}
\newcommand{\pdfGau}{\varphi}

\newcommand{\Gau}{\mathcal{N}}
\newcommand{\Lap}{\mathcal{L}}
\newcommand{\Uni}{\mathcal{U}}

\newcommand{\PCFReal}{\mathsf{R}}
\newcommand{\PCF}[1]{\underline{#1}}
\newcommand{\PCFIf}[3]{\mathsf{if}\big(#1, #2, #3\big)}
\newcommand{\Normal}{\mathsf{normal}}
\newcommand{\Score}[1]{\mathsf{score}(#1)}
\newcommand{\Y}[1]{\mathsf{Y}{#1}}
\newcommand{\Fail}{\mathsf{fail}}

\newcommand{\terms}{\Lambda}

\newcommand{\closedvalues}{\Lambda^0_v}
\newcommand{\sk}{\mathsf{SK}}

\newcommand{\pop}{\mathcal {F}}
\newcommand{\tyarrow}{\Rightarrow}

\newcommand{\contra}{R'}
\newcommand{\Ifleq}[3]{\mathsf{if}\big(#1\leq 0, #2, #3\big)}

\newcommand{\tow}{\text{otherwise }}

\newcommand{\terma}{M}
\newcommand{\termb}{N}
\newcommand{\termc}{L}
\newcommand{\typea}{\sigma}
\newcommand{\typeb}{\tau}

\newcommand{\traces}{\mathbb{T}}
\newcommand{\trace}{\boldsymbol{t}}

\newcommand{\tmeasure}{{\mu_{\traces}}}
\newcommand{\emptytrace}{{[]}}
\newcommand{\evalcon}{E}
\newcommand{\valuea}{V}
\newcommand{\redexa}{R}
\newcommand{\red}{\longrightarrow}
\newcommand{\redplus}{\red^+}
\newcommand{\config}[3]{\left\langle{#1,\allowbreak#2,\allowbreak#3}\right\rangle}
\newcommand{\weightfn}{\mathsf{weight}}
\newcommand{\valuefn}{\mathsf{value}}
\newcommand{\oper}[1]{\llbracket #1 \rrbracket}
\newcommand{\transkernel}[1]{k_{#1}}
\newcommand{\List}[1]{\mathsf{List}({#1})}
\newcommand{\Pair}[1]{\mathsf{Pair}({#1})}

\newcommand{\states}{\mathbb{S}}
\newcommand{\smeasure}{{\mu_{\states}}}
\newcommand{\validstates}{\states^{\textrm{valid}}}

\newcommand{\btra}{w}
\newcommand{\trunc}[1]{\btra_{\leq #1}}

\newcommand{\seqa}{\vec{q}}
\newcommand{\seqrange}[2]{\vec{q}^{#1 \dots #2}}
\newcommand{\seqi}[1]{\vec{q}_{#1}}

\newcommand{\sdist}{\pi}
\newcommand{\spdf}{\zeta}
\newcommand{\tdist}{\nu}
\newcommand{\tpdf}{\rho}

\newcommand{\stateterm}[1]{(\vec{q}^{(#1)},\vec{p}^{(#1)})}
\newcommand{\chain}{\set{\stateterm{i}}_{i\in\Nat}}

\newcommand{\HMCint}{\Psi}

\newcommand{\len}[1]{|#1|}
\newcommand{\BTint}{\Psi_{\mathsf{NP}}}
\newcommand{\NPRInt}{\Psi_{\mathsf{NP-R}}}
\newcommand{\NPDisInt}{\Psi_{\mathsf{NP-Dis}}}
\newcommand{\extend}{\mathsf{extend}}

\DeclareFixedFont{\ttb}{T1}{txtt}{bx}{n}{9} 
\DeclareFixedFont{\ttm}{T1}{txtt}{m}{n}{9}  

\definecolor{deepblue}{rgb}{0,0,0.5}
\definecolor{deepred}{rgb}{0.6,0,0}
\definecolor{deepgreen}{rgb}{0,0.5,0}

\newcommand\pythonstyle{\lstset{
  language=Python,
  basicstyle=\scriptsize\ttm,
  otherkeywords={append,range,min,len,sum,True,False,pop},             
  keywordstyle=\ttb\color{deepblue},
  commentstyle=\ttm\color{white!55!black},
  emph={domain,cdfN,pdfN,grad,score,normal,uniform,sample,observe,from,log},          
  emphstyle=\ttb\color{deepred},    
  ndkeywords={extend,NPint,NPHMC,NPHMCstep,validstate,HMCint,eNPHMC,eNPHMCstep,accept,supported},
  ndkeywordstyle=\ttb\color{deepgreen},
  frame=trBL,                         
  showstringspaces=false,            %
  breaklines=true,
  basewidth=0.55em
}}

\lstnewenvironment{python}[1][]
{
\pythonstyle
\lstset{#1}
}
{}

\newcommand\pythoninline[1]{{\pythonstyle\lstinline!#1!}}

\usepackage{microtype}
\usepackage{graphicx}
\usepackage{subfigure}
\usepackage{booktabs} 

\usepackage[accepted]{icml2021}

\icmltitlerunning{Nonparametric Hamiltonian Monte Carlo}

\begin{document}

\twocolumn[
\icmltitle{Nonparametric Hamiltonian Monte Carlo}



\icmlsetsymbol{equal}{*}

\begin{icmlauthorlist}
\icmlauthor{Carol Mak}{oxford}
\icmlauthor{Fabian Zaiser}{oxford}
\icmlauthor{Luke Ong}{oxford}
\end{icmlauthorlist}

\icmlaffiliation{oxford}{Department of Computer Science, University of Oxford, United Kingdom}

\icmlcorrespondingauthor{Carol Mak}{pui.mak@cs.ox.ac.uk}

\icmlkeywords{Inference algorithm, MCMC, Tree}

\vskip 0.3in
]



\printAffiliationsAndNotice{}  

\begin{abstract}

Probabilistic programming uses programs to express generative models whose posterior probability is then computed by built-in inference engines. A challenging goal is to develop general purpose inference algorithms that work out-of-the-box for arbitrary programs in a universal probabilistic programming language (PPL). The densities defined by such programs, which may use stochastic branching and recursion, are (in general) \emph{nonparametric}, in the sense that they correspond to models on an infinite-dimensional parameter space. However standard inference algorithms, such as the Hamiltonian Monte Carlo (HMC) algorithm, target distributions with a fixed number of parameters. This paper introduces the \emph{Nonparametric Hamiltonian Monte Carlo} (NP-HMC) algorithm which generalises HMC to nonparametric models. Inputs to NP-HMC are a new class of measurable functions called ``\emph{tree representable}'', which serve as a language-independent representation of the density functions of probabilistic programs in a universal PPL. We provide a correctness proof of NP-HMC, and empirically demonstrate significant performance improvements over existing approaches on several nonparametric examples.
\end{abstract}

\lo{CONVENTIONS:

1. Use $\dif$ (Roman) in integrals

2. Changed ``bold blue path'' to ``blue path''

3. Nonparametric (not non-parametric)

4. initialise / formalise (not initialize / formalize) [British English, for consistency]
}

\fz{
5. Use \texttt{\textbackslash{}cref\{...\}} instead of \texttt{Figure/ Algorithm/ etc.~\textbackslash{}ref\{...\}} if possible, for consistency
}

\cm{
6. Use \texttt{Assumption \textbackslash{}aref\{ass:\#1\}\{\#1\}} for referencing Assumption \#1.
}

\section{Introduction}
\label{sec:intro}

\changed[lo]{Probabilistic programming is a general purpose means of expressing probabilistic models as programs, and automatically performing Bayesian inference.
Probabilistic programming systems enable data scientists and domain experts to focus on designing good models;
the task of developing efficient inference engines can be left to experts in Bayesian statistics, machine learning and programming languages.
To realise the full potential of probabilistic programming, it is essential to automate the inference of latent variables in the model, conditioned on the observed data.}

\changed[lo]{Church \cite{DBLP:conf/uai/GoodmanMRBT08} introduced \emph{universal probabilistic programming}, the idea of writing probabilistic models in a Turing-complete functional programming language.}
\changed[lo]{Typically containing only a handful of basic programming constructs} such as branching and recursion, universal probabilistic programming languages (PPLs) can nonetheless specify all computable probabilistic models \cite{VakarKS19}.
In particular, \defn{nonparametric models}---models with an unbounded number of random variables---can be described naturally in universal PPLs using recursion.
These include probabilistic models with an unknown number of components, like
Bayesian nonparametric models \cite{https://doi.org/10.1111/1467-9868.00095},
variable selection in regression \cite{article},
signal processing \cite{DBLP:conf/aistats/MurrayLKBS18};
and models that are defined on infinite-dimensional spaces, such as probabilistic context free grammars \cite{Manning99},
birth-death models of evolution \cite{DBLP:conf/uai/KudlickaMRS19} and
statistical phylogenetics \cite{Ronquist2020.06.16.154443}.
Examples of practical universal PPL include Anglican \cite{DBLP:conf/aistats/WoodMM14}, Venture \cite{DBLP:journals/corr/MansinghkaSP14}, Web PPL \cite{dippl}, Hakaru \cite{narayanan2020symbolic}, Pyro \cite{pyro}, Turing \cite{DBLP:conf/aistats/GeXG18} and Gen \cite{DBLP:conf/pldi/Cusumano-Towner19}.

However, because universal PPLs are expressively complete, it is a challenging problem to design and implement general purpose inference engines for them.
The parameter space of a nonparametric model is a union of spaces of varying dimensions.
\changed[fz]{To approximate the posterior via an Markov chain Monte Carlo (MCMC) algorithm, the transition kernel will have to efficiently switch between a potentially unbounded number of configurations of different dimensions.}
This difficulty explains why there are so few \emph{general purpose} MCMC algorithms for universal PPLs \cite{DBLP:journals/jmlr/WingateSG11,DBLP:conf/aistats/WoodMM14,DBLP:conf/pkdd/TolpinMPW15,HurNRS15}.
\changed[lo]{We believe it is also the reason why these algorithms struggle with nonparametric models, as we show in \cref{sec:experiements}.
A case in point is the widely used universal PPL Pyro.
Even though it allows the specification of nonparametric models, its HMC and No-U-Turn Sampler \cite{HoffmanG14} inference engines do not support them reliably:}
\changed[fz]{in one of our benchmark tests, they produced a wrong posterior (\cref{fig:walk-pyro}).}



In this paper, we introduce the \emph{Nonparametric Hamiltonian Monte Carlo} (NP-HMC) algorithm, which generalises the Hamiltonian Monte Carlo (HMC) algorithm \cite{DUANE1987216} to nonparametric models.
The input to NP-HMC is what we call a \emph{tree representable} (TR) function, which is a large class of measurable functions of type \changed[fz]{$\btra: \bigcup_{n \in \Nat} \Real^n \to \Real_{\ge 0}$},
designed to be a language-independent representation for the density functions of programs written in any universal PPL.
The parameter space of the standard HMC algorithm is $\Real^n$, a Euclidean space of a fixed dimension.
By contrast, the parameter space of NP-HMC is $\bigcup_{n \in \Nat}\Real^n$.
\changed[lo]{The key innovation of NP-HMC is a method by which the dimension of the configuration of the current sample is incremented lazily, while preserving the efficacy of HMC by keeping the Hamiltonian approximately invariant.}
We prove that NP-HMC is correct, i.e., the induced Markov chain converges to the posterior distribution.
To evaluate the practical utility of NP-HMC, we compare an implementation of the algorithm against existing out-of-the-box MCMC inference algorithms on several challenging models with an unbounded number of random variables.
Our results suggest that NP-HMC is applicable to a large class of probabilistic programs written in universal PPLs,
offering significantly better performance than existing algorithms.

\paragraph{Notation}

We write $\seqa$ to mean a (possibly infinite) real-valued sequence;
$\seqrange{1}{i}$ \changed[lo]{the prefix of $\vec \seqa$ consisting of the first $i$ coordinates};
$\seqi{i}$ the $i$-th coordinate of $\seqa$; and
$\len{\seqa}$ the length of $\seqa$.
We write sequence as lists,
such as $[3.6,1.0,3,55, -4.2]$, and
the concatenation of sequences $\seqa$ and $\seqa'$ as $\seqa \concat \seqa'$.


We write $\Borel_n$ for the Borel $\sigma$-algebra of $\Real^n$;
$\Gau_n$ for the standard $n$-dimensional normal distribution with mean $\boldsymbol 0$ and covariance $\idmat$;
\changed[cm]{$\pdfGau_n(\vec{x} \mid \vec{\mu}, \vec{\Sigma})$}
for the density of $\vec{x} \in \Real^n$ in the $n$-dimensional normal distribution with mean $\vec{\mu}$ and covariance $\vec{\Sigma}$.
For brevity we write $\pdfGau_n(\vec{x})$ for $\pdfGau_n(\vec{x} \mid \vec{0}, \vec{\idmat})$ and $\pdfGau$ for $\pdfGau_1$.
\fz{Can we inline some of the definitions in this paragraph to where they're used? It seems like some are only used once.}

For any $\pReal$-valued function 
\changed[lo]{$f: \domain{f} \to \pReal$},
we write $\support{f} := \inv{f}(\Real_{>0})$ for the \defn{support} of $f$; and
$\nsupport{f}{n} := \support{f} \cap \Real^n$ for the support of $f$ in $\Real^n$.
We say $x \in X$ is $f$-\defn{supported} if $x \in \support{f}$.

\section{Tree Representable Functions}
\label{sec:tr-functions}


Conventional HMC \changed[fz]{samples from a distribution} with a density function $\btra: \Real^n \to \Real_{\ge 0}$ where the dimension of the target (parameter) space $\Real^n$ is fixed.
However this is too restrictive for probabilistic programs, because---with branching and recursion---the target space has a variable, even unbounded, number of dimensions.
\begin{python}[caption={A simple probabilistic program.},label={intro-program},captionpos=b,float]
q = sample(normal(0, 1))
sum = 0
while sum < q:
  sum += sample(normal(0, 1))
observe(sum, normal(q, 1))
\end{python}%
\begin{example}[Working]
Consider the probabilistic program in \cref{intro-program} where \pythoninline{sample(normal(0, 1))} denotes sampling from the standard normal distribution.
The dimension of the target space, i.e.~the number of samples drawn, can vary from run to run because of the branching behaviour of the while loop.
(Recall that by \defn{trace}, we mean the sequence of samples drawn in the course of a particular run, one for each random primitive encountered.)
\changed[fz]{The density function then records the weight of this trace, computed by multiplying the probability densities of all sampled values and the likelihoods of all observations during the run.}
For the above program, we could have a trace $[0.3, 0.5] \in \Real^2$ of length 2 or a trace $[1.0, 0.5, 0.5] \in \Real^3$ of length 3, and so on.
Hence we have to consider density functions of type $\btra: \bigcup_{n\in \Nat} \Real^n \to \Real_{\ge 0}$.
However, not every such function makes sense as the density of a probabilistic program.
For example, if $\btra([1.0, 0.5, 0.5]) > 0$ then the program can execute successfully with the trace $[1.0, 0.5, 0.5]$, but not with $[1.0, 0.5]$ or any other prefix.
\changed[fz]{In other words,} no proper prefix of $[1.0, 0.5, 0.5]$ is in $\support{\btra}$.
\end{example}

Thus we set our target space to be the \textbf{\em measure space of traces}
$\traces := \bigcup_{n\in\Nat} \Real^n$ equipped with the standard disjoint union $\sigma$-algebra
$\Sigma_\traces := \set{\bigcup_{n\in\Nat} U_n \mid U_n \in \Borel_n }$, with measure given by summing the respective (higher-dimensional) normals $\tmeasure(\bigcup_{n\in\Nat} U_n) := \sum_{n\in\Nat} \Gau_n(U_n)$.

We consider density functions that are measurable functions $\btra: \traces \to \pReal$ satisfying the \defn{prefix property}:
whenever $\vec{q} \in \nsupport{\btra}{n}$ then for all $k < n$,
we have $\vec{q}^{1\dots k} \not\in \nsupport{\btra}{k}$.
We call them \defn{tree representable (TR) functions} because any such function $\btra$ can be represented as a possibly infinite but finitely branching tree, which we call \emph{program tree}.
This is exemplified in \cref{fig:binary-tree} (left),
where
a circular node denotes an element of the input of type $\Real$;
a rectangular node gives the condition for
{$\vec{q} \in \nsupport{\btra}{n}$}
(with the left, but not the right, child satisfying the condition); and
a leaf node gives the result of the function on that branch.
Any \defn{branch} (i.e.~path from root to leaf) in a program tree of $w$ represents a set of finite sequences $[q_1,\dots,q_n]$ in $\support{\btra}$.
In fact, every program tree of a TR function $w$ specifies a countable partition of $\support{\btra}$ \changed[fz]{via its branches}.
The prefix property guarantees that \changed[fz]{for each TR function $w$, there are program trees of the form in \cref{fig:binary-tree} (left) representing $w$.}

\tikzset{position/.style={blue,font=\footnotesize}}
\tikzset{path/.style={draw=blue, line width=3pt}}
\tikzset{not path/.style={draw=black, line width=0.5pt}}
\tikzset{input/.style={circle,draw=black, line width=0.5pt,minimum size=12pt,font=\footnotesize},inner sep=2pt}
\tikzset{condition/.style={rectangle,draw=black, line width=0.5pt, minimum height=0.1pt,font=\footnotesize},inner sep=2pt}
\tikzset{leaf/.style={draw=none,font=\footnotesize}}
\tikzset{decision/.style={black!50!green}}

\begin{figure}[t]
  \parbox{.24\textwidth}{
    \centering
    \begin{tikzpicture}[scale=.55,sibling distance = 5em]
      \node [input] (q1) {$q_1$}
        child {
          node [condition] (c1) {$[q_1] \overset{?}{\in} \nsupport{\btra}{1}$}
          child {node [leaf] (r1) {$\btra([q_1])$}}
          child {
            node [input] (q2) {$q_2$}
            child {
              node [condition] (c2) {$[q_1,q_2] \overset{?}{\in} \nsupport{\btra}{2}$}
              child {node [leaf,xshift=-1em] (r2) {$\btra([q_1,q_2])$}}
              child {
                node [input] (q3) {$q_3$}
                child {
                  node [condition] (c3)
                  {$[q_1,q_2,q_3] \overset{?}{\in} \nsupport{\btra}{3}$}
                  child {node [leaf,xshift=-1em] (r3) {$\btra([q_1,q_2,q_3])$}}
                  child {node {$\vdots$}}
                }
              }
            }
          }
        };
        \node at ($(c1) + (220:1.3cm) $) [decision] {\scriptsize yes};
        \node at ($(c1) + (315:1.1cm) $) [decision] {\scriptsize no};
        \node at ($(c2) + (215:1.4cm) $) [decision] {\scriptsize yes};
        \node at ($(c2) + (315:1.1cm) $) [decision] {\scriptsize no};
        \node at ($(c3) + (215:1.4cm) $) [decision] {\scriptsize yes};
        \node at ($(c3) + (315:1.1cm) $) [decision] {\scriptsize no};
    \end{tikzpicture}
  }%
  \parbox{.25\textwidth}{
    \centering
    \begin{tikzpicture}[scale=.55,sibling distance = 5em]
      \node [input] (q1) {$q_1$}
        child {
          node [condition] (c1) {$q_1 \leq 0$}
          edge from parent [path]
          child {node [leaf,xshift=-1em] (r1) {$\pdfGau(0 \mid q_1,1)$}
          edge from parent [not path]}
          child {
            node [input] (q2) {$q_2$}
            child {
              node [condition] (c2) {$0<q_1\leq q_2$}
              child {node [leaf,xshift=-1.25em] (r2) {$\pdfGau(q_2 \mid q_1,1)$}}
              child {
                node [input] (q3) {$q_3$}
                edge from parent [not path]
                child {
                  node [condition, text width=2.1cm] (c3)
                  {$0 < q_1$ and \hspace{1cm} $q_2 < q_1 \leq q_2+q_3$}
                  child {node [leaf,xshift=-2em] (r3) {$\pdfGau(q_2+q_3 \mid q_1,1)$}}
                  child {node {$\vdots$}}
                }
              }
            }
          }
        };
        \node [position, right of=q1, node distance=17pt] {$0.3$};
        \node [position, right of=q2, node distance=17pt] {$0.5$};
    \end{tikzpicture}
  }%
  \caption{(left) Generic \btrfunction\ $\btra$; (right) \btrfunction\ $\btra$ for the probabilistic program in \cref{intro-program}.}
  \label{fig:binary-tree}
\end{figure}
We target \btrfunction{}s as densities for our new sampler NP-HMC because they are a naturally large class of functions.
In particular, every program of a universal PPL has a density function that is tree representable.\footnote{\label{fnote:TR function} With (additional) suitable assumptions about the computability of $w$, we can view any such tree as the abstract syntax tree of a program that computes $w$, but with any recursion unravelled (so that the tree is potentially infinite).}
(See
\ifapxAppended\cref{prop: all spcf terms have TR weight function} in \cref{appendix: SPCF}
\else the SM
\fi
for a formal account.)
For instance, the program in \cref{intro-program} has density\footnote{Notice that, even though the program samples from a normal distribution, $\btra$ does not factor in Gaussian densities \changed[fz]{from those sample statements, just the observe statement}, since they are already accounted for by $\measure{\traces}$.}
$\btra$ (\btra.r.t.~the \changed[cm]{stock} measure $\tmeasure$) given by
\[
  {\btra}{(\vec{q})} :=
    \begin{cases}
      {\pdfGau(\sum_{i=2}^{n} \vec{q}_i \mid \vec{q}_1,1)}
      &
      \begin{matrix}
        \text{if }
        \forall k < \len{\vec{q}} \text{\hspace{5em}} \\[-0.1em]
        {\sum_{i=2}^{k} \vec{q}_i < \vec{q}_1 \leq \sum_{i=2}^{n} \vec{q}_i},
      \end{matrix}
      \\
      0 & \text{otherwise}
    \end{cases}
\]
which is TR and \changed[lo]{it has a program tree as depicted in \cref{fig:binary-tree} (right)}.
Notice that every element in the support of $\btra$ belongs to a branch in this tree: for example, the trace $[0.3,0.5]$ belongs to the blue branch in \cref{fig:binary-tree} (right).

As we will explain in \cref{sec:np-hmc}, the prefix property (satisfied by TR functions) is
essential for the correctness of NP-HMC.

\section{Nonparametric HMC}
\label{sec:np-hmc}


\defn{Nonparametric Hamiltonian Monte Carlo (NP-HMC)} (\cref{fig:algorithms}) is a MCMC algorithm that, given a \lo{The qualification ``(computable)'' may prompt questions. In view of footnote~\ref{fnote:TR function}, I think we can / should drop it.} \btrfunction{} $\btra$, iteratively proposes \changed[lo]{a new sample $\seqa \in \bigcup_{n \in \Nat}\Real^n$} and accepts {it} with a suitable Hastings acceptance probability,
\changed[lo]{such that the invariant distribution of the induced Markov chain is}
\[ \tdist:{A}\mapsto \frac{1}{Z} \shortint{A}{\btra}{\tmeasure}
\]
with normalising constant $Z := \shortint{\traces}{\btra}{\tmeasure}$.
As the name suggests, NP-HMC is a generalisation of the standard HMC algorithm (\cref{rem:NP-HMC extends HMC}.i) to \emph{nonparametric models}, in the form of TR functions whose support is a subspace of $\traces$ of unbounded dimension.

In this section, we first explain our generalised algorithm, using a version (\cref{alg:np-hmc}) that is geared towards conceptual clarity, and defer discussions of more efficient variants.

\begin{figure*}[t]
  \parbox{.42\textwidth}{
  \parbox[t][5cm][c]{.4\textwidth}{
  \centering
  \begin{tikzpicture}[scale=0.8]
    \begin{axis}[
      axis lines = left,
      xlabel = $q_1$,
      ylabel = {Energy},
      xmin=-4, xmax=4,
      ymin=0,
      height=7cm,
      every axis plot/.append style={very thick},
      grid=both,
      grid style={line width=.1pt, draw=gray!10},
      major grid style={line width=.2pt,draw=gray!50},
    ]
    \addplot [
      domain=-4:0,
      samples=100,
      color=red,
    ]
    {-ln(1/sqrt(2*pi)*exp(-(1/2)*x^2))};
    \addlegendentry{$-\log(\pdfGau(0 \mid q_1,1))\cdot[q_1 \leq 0]$}
    \node[label={45:{[-3.1]}},circle,fill,inner sep=2pt] at (axis cs:-3.1,5.72393853) {};
    \node[label={45:{[-2.37]}},circle,fill,inner sep=2pt] at (axis cs:-2.37,3.72738853) {};
    \node[label={45:{[-0.86]}},circle,fill,inner sep=2pt] at (axis cs:-0.86,1.28873853) {};
    \end{axis}
  \end{tikzpicture}
  }
  \small
  \begin{center}
  \begin{tabular}{lllll}
  \toprule
  Time &0 &1 &2 &3 \\
  \midrule
  $\vec{q}$ &[-3.1] &[-2.37] &[-0.86] &[1.15] \\
  \midrule
  $\vec{p}$ &[1.2] &[3.29] &[3.94] &[5.26] \\
  \bottomrule
  \end{tabular}
  \end{center}
  \caption{The Hamiltonian dynamics of a particle on the surface $-\log(\pdfGau(0 \mid q_1,1))\cdot[q_1 \leq 0]$ on $\Real$.}
  \label{fig:hamdyn on R}
  }
  \hspace{.03\textwidth}
  \parbox{.55\textwidth}{
  \parbox[t][5cm][c]{.5\textwidth}{
  \centering
  \begin{tikzpicture}[scale=0.78]
    \begin{axis}[
      axis lines = left,
      xlabel = $q_1$,
      ylabel = $q_2$,
      zlabel = {Energy},
      xmin=-4, xmax=4,
      ymin=-4, ymax=4,
      zmin=0,
      view={15}{5},
      height=7cm,
      legend style={at={(1.05,1.05)}},
      every axis plot/.append style={thick},
      grid=both,
      grid style={line width=.1pt, draw=gray!10},
      major grid style={line width=.2pt,draw=gray!50},
      legend cell align={left},
    ]
    \addplot3 [
      mesh,
      domain=-4:0,
      y domain=-4:4,
      samples=40,
      color=red,
    ]
    {-ln(1/sqrt(2*pi)*exp(-(1/2)*x^2))};
    \addlegendentry{$-\log(\pdfGau(0 \mid q_1,1))\cdot[q_1 \leq 0]$}
    \addplot3 [
      mesh,
      domain=0:4,
      y domain=0:4,
      restrict z to domain=0:8,
      samples=40,
      color=blue
    ]{(y > x) ? -ln(1/sqrt(2*pi)*exp(-(1/2)*(y-x)^2)):-1};
    \addlegendentry{$-\log(\pdfGau(q_2 \mid q_1,1)) \cdot[0 < q_1 \leq q_2] $}
    \node[label={[fill=white, fill opacity=0.75, text opacity=1]90:{[-3.1, -1.61]}},circle,fill,inner sep=2pt] at (axis cs:-3.1,-1.61,5.72393853) {};
    \node[label={[fill=white, fill opacity=0.75, text opacity=1]90:{[-2.37, -0.39]}},circle,fill,inner sep=2pt] at (axis cs:-2.37,-0.39,3.72738853) {};
    \node[label={[fill=white, fill opacity=0.75, text opacity=1]90:{[-0.86, 0.82]}},circle,fill,inner sep=2pt] at (axis cs:-0.86,0.82,1.28873853) {};
    \node[label={[fill=white, fill opacity=0.75, text opacity=1]45:{[1.15, 2.04]}},circle,fill,inner sep=2pt] at (axis cs:1.15,2.04,1.31321053) {};
    \end{axis}
  \end{tikzpicture}
  }
  \small
  \begin{tabular}{lllll}
  \toprule
  Time &0 &1 &2 &3 \\
  \midrule
  $\vec{q}$ &[-3.1, -1.61]&[-2.37, -0.39]&[-0.86, 0.82]&[1.15, 2.04] \\
  \midrule
  $\vec{p}$ &[1.2, 3.04] &[3.29, 3.04] &[3.94, 3.04] &[5.26, 3.04] \\
  \bottomrule
  \end{tabular}
  \caption{The Hamiltonian dynamics of a particle on the updated surface $-\log(\pdfGau(0 \mid q_1,1))\cdot[q_1 \leq 0] -\log(\pdfGau(q_2 \mid q_1,1)) \cdot[0 < q_1 \leq q_2] $ on $\Real^2$.}
  \label{fig:hamdyn on R2}
  }
\end{figure*}

\paragraph{Idea}
\changed[fz]{We assume basic familiarity with the HMC algorithm; see \cite{Betancourt18} for the intuition behind it and \cite{Neal2011} for details.}
Like the HMC algorithm, NP-HMC views a sample $\seqa$ as a particle at position $\seqa$, with a randomly initialised momentum $\vec p$, moving on a frictionless surface derived from the density function $\btra$.
The key innovation lies in our treatment when the particle moves \emph{beyond} the surface, i.e.~outside the support of the density function.
This procedure is called $\extend$ (\cref{alg:extend}), which we will now illustrate.


Let's trace the movement of a particle
at position $\vec{q}=[-3.1]$ with a randomly chosen momentum $\vec{p} = [1.2]$,
on the \changed[lo]{surface (a line in 1D) determined} by the \btrfunction{} $\btra$ in \cref{fig:binary-tree} (right).
The first two steps taken by the particle are simulated according to the Hamiltonian dynamics on the surface $-\log(\pdfGau(0 \mid q_1,1))\cdot[q_1 \leq 0]$, which is derived from the restriction of $\btra$ to $\Real$.
The positions on the surface and states\footnote{As in HMC, a \emph{state} of the NP-HMC algorithm is a position-momentum pair $(\vec{q}, \vec{p})$ with $|\vec{q}|=|\vec{p}|$; but unlike HMC, $\vec{q}, \vec{p} \in \traces$.} of the particle at each step are given in \cref{fig:hamdyn on R}.

At the third time step, the particle is at the position $[1.15]$, which is no longer on the surface.
To search for a suitable state, NP-HMC increments the dimension of the \changed[lo]{current surface (line)} as follows.

First, the surface is extended to
$-\log(\pdfGau(0 \mid q_1,1))\cdot[q_1 \leq 0] -\log(\pdfGau(q_2 \mid q_1,1)) \cdot[0 < q_1 \leq q_2]$,
which is derived from the sum of the respective restrictions of $\btra$ to $\Real$ and to $\Real^2$,
as depicted in \cref{fig:hamdyn on R2}.
Since $\btra$ satisfies the prefix property, the respective supports of these restrictions,
namely $\set{[q_1,q_2]\in\Real^2 \mid q_1 \leq 0}$ and $\set{[q_1,q_2]\in\Real^2 \mid 0 < q_1 \leq q_2}$, are disjoint; and
hence the states of the particle on the previous surface $-\log(\pdfGau(0 \mid q_1,1))\cdot[q_1 \leq 0]$ can be reused on the updated surface.
\changed[lo]{Notice that the respective first coordinates of the particle's positions on the surface in \cref{fig:hamdyn on R2} are identical to the particle's positions on the surface (line) in \cref{fig:hamdyn on R}.}


Next, \changed[lo]{the initial state $(\vec{q} =[-3.1],\vec{p}=[1.2])$ is extended by appending a randomly chosen value to both the position and momentum components}, so that the particle is positioned on the updated surface with an initial momentum.
In our example, $-1.61$ and $3.04$ are sampled and the initial state of the particle becomes $(\vec{q} =[-3.1,-1.61],\vec{p}=[1.2,3.04])$ which is located on the surface as shown in \cref{fig:hamdyn on R2}.
The states at times $1$, $2$ and $3$ are updated accordingly and are given in the table in \cref{fig:hamdyn on R2}.
Notice that the particle at time $3$ is now positioned on the updated surface, and hence Hamiltonian dynamics can resume.
The rest of this section is devoted to \changed[lo]{formalising} our algorithm.

\begin{assumption*}
  Henceforth we assume that the input \btrfunction\ $\btra$ satisfies the following:
  \begin{compactenum}
    \item 
    [\hypertarget{ass:1}{1}]
    \emph{Integrability}, i.e.,~the normalising constant $Z < \infty$.
    (Otherwise, the inference problem would be ill-defined.)
    \item
    [\hypertarget{ass:2}{2}]
    The function $\btra$ is \emph{almost everywhere continuously differentiable} on $\traces$.
    (Since Hamiltonian dynamics exploits the derivative of $\btra$ in simulating the position of a particle,
    this ensures that the derivative exists ``often enough'' for the Hamiltonian integrator to be used.)
    \item
    [\hypertarget{ass:3}{3}]
    For \emph{almost} every infinite real-valued sequence $\seqa$,
    there is a $k$ such that $\btra$ is positive on $\seqrange{1}{k}$.
    (This ensures the $\extend$ subroutine in \cref{alg:extend} terminates almost surely.)
  \end{compactenum}
\end{assumption*}

\begin{remark}
  \label{rem:NP-HMC extends HMC}
  \begin{asparaenum}[(i)]
    \item  NP-HMC is a generalisation of HMC to nonparametric models. Precisely, NP-HMC specialises to standard HMC (with the leapfrog integrator) in case the \changed[fz]{density} function $w$ satisfies \changed[cm]{$\domain{w} = \Real^n$} for some $n$,
    in addition to Assumptions \aref{ass:1}{1}, \aref{ass:2}{2} and \aref{ass:3}{3}.
    \item
      All closed integrable \emph{almost surely terminating} programs of a universal PPL\footnote{\changed[cm]{An almost surely terminating program (as defined in
      \ifapxAppended\cref{appendix: SPCF}\else the SM\fi)
      almost never observes a value with zero probability density.}}
      induce densities that satisfy Assumptions \aref{ass:1}{1}, \aref{ass:2}{2} and \aref{ass:3}{3}.
      \changed[lo]{(See \ifapxAppended\cref{appendix: SPCF} and \cref{lemma: all AST SPCF term satisfy all assumptions} \else the Supplementary Materials (SM) \fi
      for an account.)}
  \end{asparaenum}
\end{remark}

\paragraph{Truncations}

The surfaces on which the particle is positioned are derived from \changed[lo]{a sum of appropriate restrictions} of the input \btrfunction{} $\btra$, defined as follows.
The $n$-th \defn{truncation} $\trunc{n}:\Real^n \to \pReal$ is
$\seqa \mapsto \sum_{k=1}^n \btra(\seqrange{1}{k})$,
which returns the cumulative sum of the weight on the prefixes of an $n$-dimensional trace $\seqa$.
Thanks to the prefix property, for each $\vec q$, at most one summand is non-zero.
\emph{So any real-valued $\trunc{n}$-supported sequence $\seqa\in\Real^n$ has a prefix in the support of $\btra$;
and any $\btra$-supported sequence of length $n$ is also $\trunc{n}$-supported}.

We define a family $U = \set{U_n}_{n\in\Nat}$ of \defn{potential energies}
where each $U_n:\Real^n \partialto \pReal$ is a partial function defined as $U_n := -\log \trunc{n}$
with domain $\domain{U_n} := \support{\trunc{n}}$.
\changed[cm]{These are the surfaces on which the particle is positioned.}

\paragraph{Proposal step}
The \defn{nonparametric (NP) integrator} $\BTint$ (\cref{alg:np-hmc integrator}) proposes a state by simulating the Hamiltonian motion of a particle at position $\seqa \in \Real^n$, with potential energy $U_n(\seqa)$ and a randomly chosen momentum $\vec{p} \in \Real^n$.\footnote{The Hamiltonian motion is almost always defined, by Ass.~\aref{ass:2}{2}.}
The simulation runs $L$ discrete update steps \changed[fz]{(also called \emph{leapfrog steps})} of size $\epsilon > 0$, or until
the particle leaves the domain of $U_n$ (i.e.~the support of $\trunc{n}$).
At that moment, the simulation stops and
NP-HMC ``extends'' the state $(\seqa,\vec{p})$ via the $\extend$ subroutine (\cref{alg:extend}), until the position of the particle falls into the domain of some potential energy \changed[lo]{(i.e.~support of some higher dimensional truncation)}.%
\footnote{This happens almost surely, thanks to Ass.~\aref{ass:3}{3}.}
Once the \changed[fz]{extended} position of the particle is settled, simulation resumes.
\changed[fz]{If no extension is necessary, the behaviour is the same as that of the standard HMC leapfrog integrator.}


\paragraph{Extend}
The heart of NP-HMC is the $\extend$ subroutine (\cref{alg:extend}).
Suppose
$\extend$ is called after
{$i$} position steps and {$i-1/2$} momentum steps are completed, i.e., at time $t = i \, \epsilon$.
If {$\seqa$} is in the domain of {$U_{\len{\seqa}}$}, $\extend$ leaves the state unchanged;
otherwise \changed[fz]{$\seqa$ is not long enough and the while loop extends it as follows.}
\begin{compactitem}
  \item Sample a pair {$(x_0,y_0)$} of real numbers from the standard normal distribution respectively.
  \item
    Trace the motion of a particle with {constant} potential energy $1$
    for {$i$} position and \changed[lo]{$i-1/2$} momentum steps,
    starting from {$(x_0,y_0)$},
    to obtain {$(x,y)$}.
    Notice that the momentum update is simply the identity, hence we only need to consider {$i$} position updates which takes {$x$ to $x_0 + t \, y_0$}.
  \item Append
    {$(x_0,y_0)$} to the initial state {$(\vec{q_0},\vec{p_0})$} and
    {$(x,y)$} to the current state {$(\seqa,\vec{p})$}.
\end{compactitem}
Thus the length of the position $\seqa$ is incremented,
and by \changed[lo]{Assumption} \aref{ass:3}{3}, this loop terminates almost surely at a position {$\seqa$} in the domain of the potential energy of dimension $|\seqa|$.


\begin{figure}[t!]
  \centering
  \vspace{-3mm}
  \begin{minipage}{\linewidth}
    \begin{algorithm}[H]
      \caption{NP-HMC Step}
      \label{alg:np-hmc}
      \begin{algorithmic}
        \STATE {\bfseries Input:}
          current \changed[fz]{sample $\vec{q_0}$},
          density function $\btra$,
          step size $\epsilon$,
          number of steps $L$
        \STATE \textbf{Output:}
          next \changed[fz]{sample $\vec{q}$}
        \algrule
        \STATE $\vec{p_0} \sim \Gau_{\len{\vec{q_0}}}$
          \hfill \COMMENT{Initialise}
        \STATE $U = \set{ U_n := \lambda \seqa. -\log(\trunc{n}(\seqa))}_{n\in\Nat}$
        \STATE \ghl{$((\seqa, \vec{p}), (\vec{q_0}, \vec{p_0})) =
          \BTint((\vec{q_0}, \vec{p_0}), U, \epsilon, L)$}
        \IF{$\Uni(0,1) < \min \Big(1,\frac
                                  {\changed[cm]{\trunc{\len{\seqa}}} (\seqa) \,{\pdfGau_{2\len{\seqa}}} (\seqa\concat\vec{p})}
                                  {\changed[cm]{\trunc{\len{\seqa}}} (\vec{q_0}) \,{\pdfGau_{2\len{\seqa}}} (\vec{q_0}\concat\vec{p_0})}\Big)$}
        \STATE {\bfseries return} $\seqrange{1}{k}$ where $\btra(\seqrange{1}{k}) > 0$
        \ELSE
        \STATE {\bfseries return} $\vec{q_0}^{1\dots k}$ \changed[fz]{where $\btra(\vec{q_0}^{1\dots k}) > 0$}
        \ENDIF
      \end{algorithmic}
    \end{algorithm}
  \end{minipage}
  \\[-2mm]
  \begin{minipage}{\linewidth}
    \begin{algorithm}[H]
      \caption{\ghl{NP Integrator $\BTint$}}
      \label{alg:np-hmc integrator}
      \begin{algorithmic}
        \STATE {\bfseries Input:}
          current state $(\vec{q_0}, \vec{p_0})$,
          family of potential energies $U = \set{U_n}_{n\in \Nat}$,
          step size $\epsilon$,
          number of steps {$L$}
        \STATE \changed[fz]{\textbf{Output:}
          new state $(\vec q, \vec p)$,
          extended initial state $(\vec q_0, \vec p_0)$}
        \algrule
        \STATE $(\seqa,\vec{p}) = (\vec{q_0}, \vec{p_0})$ \hfill \COMMENT{Initialise}
        \FOR{$i=0$ {\bfseries to} $L$}
          \STATE $\vec{p}=\vec{p}-\frac{\epsilon}{2}\grad{U_{\len{\vec{q}}}}(\seqa)$ \hfill\COMMENT{1/2 momentum step}
          \STATE $\seqa=\seqa+\epsilon \, \vec{p}$ \hfill\COMMENT{1 position step}
          \STATE \yhl{$((\seqa, \vec{p}), (\vec{q_0}, \vec{p_0})) = \extend((\seqa, \vec{p}), (\vec{q_0}, \vec{p_0}),\changed[lo]{i\,\epsilon}, U)$}
          \STATE $\vec{p}=\vec{p}-\frac{\epsilon}{2}\grad{U_{\len{\vec{q}}}}(\seqa)$ \hfill\COMMENT{1/2 momentum step}
        \ENDFOR
        \STATE {\bfseries return} $((\seqa, \vec{p}), (\vec{q_0}, \vec{p_0}))$
      \end{algorithmic}
    \end{algorithm}
  \end{minipage}
  \\[-2mm]
  \begin{minipage}{\linewidth}
    \begin{algorithm}[H]
      \caption{\yhl{$\extend$}}
      \label{alg:extend}
      \begin{algorithmic}
        \STATE {\bfseries Input:}
          current state $(\seqa, \vec{p})$,
          initial state $(\vec{q_0}, \vec{p_0})$,
          time $t$,
          family of potential energies $U = \set{U_n}_{n\in \Nat}$
        \STATE \changed[fz]{\textbf{Output:}
          extended state $(\vec q, \vec p)$,
          extended initial state $(\vec q_0, \vec p_0)$}
        \algrule
        \WHILE{$\seqa \notin \domain{U_{\len{\seqa}}}$}
          \STATE $x_0 \sim \Gau_1; y_0 \sim \Gau_1$ \hfill\COMMENT{sample from normal}
          \STATE $(x, y) = (x_0 + t \, y_0, y_0)$ \hfill\COMMENT{run for time t}
          \STATE $(\vec{q_0}, \vec{p_0}) = (\vec{q_0} \concat [x_0], \vec{p_0} \concat [y_0])$ \hfill\COMMENT{update initial}
          \STATE $(\seqa, \vec{p}) = (\seqa \concat [x], \vec{p} \concat [y])$ \hfill\COMMENT{update current}
        \ENDWHILE
        \STATE {\bfseries return} $((\seqa, \vec{p}), (\vec{q_0}, \vec{p_0}))$
      \end{algorithmic}
    \end{algorithm}
  \end{minipage}
  \caption{Pseudocode for Nonparametric Hamiltonian Monte Carlo} 
  \label{fig:algorithms}
\end{figure}


\paragraph{Putting them together}
A single NP-HMC iteration, as shown in \cref{alg:np-hmc}, produces a proposed sample $\vec{q}$ from the current sample $\vec{q_0}$, by applying the NP integrator $\BTint$ to the state $(\vec{q_0},\vec{p_0})$ with a freshly sampled momentum $\vec{p_0}$.
The proposed state $(\vec{q},\vec{p})$ is then accepted with probability given by the Hastings acceptance ratio.




\begin{remark}
  The prefix property of the input \btrfunction{} $\btra$ plays an important role in the NP-HMC inference algorithm.
  If \cref{alg:np-hmc} returns $\vec q$ as the next sample,
  then for any extension $\vec q'$ of $\vec q$, we have $w(\vec q') = 0$, and so, all such $q'$ are irrelevant for inference (\changed[fz]{because} $U_{|\vec q'|}(\vec{q'}) = U_{|\vec q|}(\vec{q})$).
  If the prefix property weren't satisfied,
  the algorithm would fail to account for the weight on such $\vec q'$.
\end{remark}
\fz{Is this remark really helpful or can it be deleted?}

\paragraph{Other extensions and efficiency considerations}

We have presented a version of NP-HMC that eschews runtime efficiency in favour of clarity.
An advantage of such a presentation
\changed[lo]{(in deliberately purified form)}
is that it becomes easy to see that the same method is just as applicable to such HMC variants as
reflective/refractive HMC \cite{AfsharD15} and
discontinuous HMC \cite{NishimuraDL20}; we call the respective extensions NP-RHMC and NP-DHMC.
For details, see
\ifapxAppended\cref{sec: hmc variants}%
\else the SM%
\fi%
.

\fz{I changed this.
Before we said we implemented NP-RHMC, which is only partially true.
I implemented it, but only for very simple branches of the form \texttt{sample(distribution) < constant}.
I barely tested this implementation.
The original RHMC paper implements it for general affine conditions, but this is difficult to implement and still insufficient for our experiments anyway (cf. stick-breaking).
Since NP-DHMC does not require computing intersections at all, I think it is just the superior algorithm.
Nevertheless, we should not claim to have fully implemented NP-RHMC.}

Several efficiency improvements to NP-HMC are possible.
If the density function is given by a probabilistic program, one can interleave its execution with $\extend$, by gradually extending $\vec q$ at every encountered \pythoninline{sample} statement.
Similarly, the truncations $w_{\le n}$ don't require an expensive summation to compute.
For details, see
\ifapxAppended\cref{sec:efficiency-improvements}%
\else the SM%
\fi%
.

Our implementation (\cref{sec:experiements}) also improves the $\extend$ function (\cref{alg:extend}): it not only extends a trace $\vec q$ if necessary, it also trims it to the unique prefix $\vec q'$ of $\vec q$ with positive $w(\vec q')$.
This version works better in our experiments.
For details, see
\ifapxAppended\cref{sec:efficiency-improvements}%
\else the SM%
\fi%
.

\lo{Because of space restriction, we may need to move this to the implementation / examples section in the appendix (also partly because we don't yet have a proof of correctness, but this is not so critical).}
\fz{I think it is important to mention this in the main text because NP-DHMC does not work in practice without this adaption.}

\section{Correctness}
\label{sec:correctness}


The {NP-HMC} algorithm is correct in the sense that
the {generated} Markov chain converges to the target distribution
$\tdist:{A}\mapsto{\frac{1}{Z} \shortint{A}{w}{\tmeasure}}$ \changed[fz]{with normalising constant $Z$}.
We present an outline of our proof here.
The full proof can be found in
\ifapxAppended\cref{appendix: correctness}.
\else the SM.
\fi


\paragraph{Invariant distribution}
By iterating \cref{alg:np-hmc}, the NP-HMC algorithm generates a Markov chain $\set{\vec{q}^{(i)}}_{i\in\Nat} $ on the target space $\traces$.
The first step to correctness is to show that the invariant distribution of this chain is indeed the target distribution.

This proof is non-trivial, since the length of the generated sample depends on the values of the random samples drawn in the $\extend$ subroutine (\cref{alg:extend}).
To work around this, we define an auxiliary algorithm, 
\changed[lo]{which induces the same Markov chain as NP-HMC,}
but does not increase the dimension dynamically.
Instead, it finds the smallest {$N$} such that all intermediate positions in the {$L$} leapfrog steps stay in the domain of {$U_N$},
and performs leapfrog steps as in standard HMC.
In this algorithm, all stochastic primitives are executed outside of the Hamiltonian dynamics simulation, and the simulation has a fixed dimension.
Hence we can proceed to identify the invariant distribution.
We then show (in
\ifapxAppended\cref{lemma: pi is the invariant distribution of np-hmc,thm: marginalised distribution is the target distribution}%
\else the SM%
\fi)
that the Markov chain generated by this auxiliary algorithm has the target distribution $\tdist$ as its invariant distribution,
and hence the same holds for NP-HMC (\cref{alg:np-hmc}).

\begin{restatable}{theorem}{invariant}
\label{thm: marginalised distribution is the target distribution}
Given Assumptions \aref{ass:1}{1}, \aref{ass:2}{2} and \aref{ass:3}{3}, the target distribution $\tdist$ is the invariant distribution of the Markov chain generated by iterating \cref{alg:np-hmc}.
\end{restatable}

\paragraph{Convergence}
In
\ifapxAppended\cref{sec:convergence}%
\else the SM%
\fi, we extend the proof of \citet{CancesLS07}
{to show that the chain converges for a small enough step size $\epsilon$},
as long as the following additional assumptions are met:
\begin{compactenum}[(C1)]
\item $w$ is continuously differentiable on a non-null set $A$ with measure-zero boundary.
\item \changed[fz]{$w|_{\support{w}}$} is bounded below by a positive constant.
\item For each $n$, the function $\frac{\grad{w_{\le n}}}{w_{\le n}}$ is uniformly bounded from above and below on $\support{w_{\le n}} \cap A$.
\item For each $n$, the function $\frac{\grad{w_{\le n}}}{w_{\le n}}$ is Lipschitz \changed[fz]{continuous} on $\support{w_{\le n}} \cap A$.
\end{compactenum}

\begin{restatable}{theorem}{convergence}
  \label{thm: np-hmc converges}
  If Assumptions (C1)--(C4) are satisfied in addition to Assumptions \aref{ass:1}{1}, \aref{ass:2}{2} and \aref{ass:3}{3},
  the Markov chain generated by iterating \cref{alg:np-hmc} converges to the target distribution $\tdist$.
\end{restatable}

\section{Experiments}
\label{sec:experiements}


We implemented the NP-HMC algorithm and its variants (NP-RHMC and NP-DHMC) in Python, using PyTorch \cite{pytorch} for automatic differentiation.
\changed[fz]{We implemented it from scratch rather than in an existing system because NP-DHMC needs additional information about each sample (does the density function depend discontinuously on it?), so it requires a deeper integration in the probabilistic programming system.}
In our empirical evaluation, we focus on the NP-DHMC algorithm because it inherits discontinuous HMC's efficient handling of discontinuities:
contrary to NP-RHMC, it does not need to compute the intersections of the particle's trajectory with the regions of discontinuity.
Our implementation also uses the efficiency improvements discussed in
\ifapxAppended \cref{sec:efficiency-improvements}%
\else the SM%
\fi.
The code for our implementation and experiments is available at \url{https://github.com/fzaiser/nonparametric-hmc} and archived as \cite{code}.

We compare our implementation with Anglican's \cite{DBLP:conf/aistats/WoodMM14} inference algorithms that are applicable out-of-the-box to nonparametric models:
lightweight Metropolis-Hastings (LMH), particle Gibbs (PGibbs) and random-walk lightweight Metropolis-Hastings (RMH).%
\footnote{
\changed[cm]{We also compared Interacting Particle MCMC (IPMCMC),} but it performed consistently worse than PGibbs in our experiments
\changed[cm]{and hence is omitted.}
}
NP-DHMC performs more computation per sample than its competitors because it evaluates the density function in each of the $L$ leapfrog steps, not just once like the other inference algorithms.
To equalise the computation budgets, we generate $L$ times as many samples for each competitor algorithm, and apply thinning (taking every $L$-th sample) to get a comparable sample size.


\begin{table}
\begin{center}\footnotesize
\caption{Total variation distance from the ground truth for the geometric distribution, averaged over 10 runs. Each run: $10^3$ NP-DHMC samples with $10^2$ burn-in, 5 leapfrog steps of size 0.1; and $5\times 10^3$ LMH, PGibbs and RMH samples.}
\label{table:geometric-tvd}
\vspace{0.1in}
\begin{tabular}{lllll}
\toprule
method &\textbf{ours} &LMH{} &PGibbs{} &RMH{} \\
\midrule
TVD &\textbf{0.0136} &0.0224 &0.0158 & 0.0196 \\
\bottomrule
\end{tabular}
\end{center}
\end{table}
\paragraph{Geometric distribution}
A classic example to illustrate recursion in a universal PPL is sampling from a geometric distribution with parameter $p$ by repeatedly flipping a biased coin with probability $p$ (see e.g.~Ch.~5 in \cite{intro-prob-prog}).
The pseudocode for it is:
\begin{minipage}{\columnwidth}
\begin{python}
def geometric():
  if sample(uniform(0, 1)) < p: return 1
  else: return 1 + geometric()
\end{python}
\end{minipage}
We tested our algorithm on this problem with $p = 0.2$.
Our implementation works well on this example and has no trouble jumping between traces of different length.
To quantify this, we computed the total variation distance to the ground truth for each approach and report it in \cref{table:geometric-tvd}.
The result is perhaps surprising given that the odds are ``stacked against'' NP-DHMC in this model:
it is a discrete model, so there is no gradient information, and there are no observations (only sampling from the prior).
These properties should favour the competitors, making the performance of NP-DHMC rather remarkable.

\paragraph{Random walk}
To better evaluate our algorithm on probabilistic programs with unbounded loops (such as the example from \cref{sec:tr-functions}), we considered the one-sided random walk on $\pReal$ described in \cite{MakOPW20}:
A pedestrian starts from a random point in $[0,3]$ and walks a uniformly random distance of at most 1 in either direction, until they pass 0.
Given a (noisily) measured total distance of 1.1 travelled, what is the posterior distribution of the starting point?
As this process has infinite expected running time, we need a stopping condition if the pedestrian is too far away from zero, (\pythoninline{distance < 10}), as shown in the following pseudocode:
\begin{minipage}{\columnwidth}
\begin{python}
start = sample(uniform(0,3))
position = start; distance = 0
while position > 0 and distance < 10:
  step = sample(uniform(-1, 1))
  position += step; distance += abs(step)
observe(distance, normal(1.1, 0.1))
return start
\end{python}
\end{minipage}

\begin{figure}
\centering
\footnotesize
\includegraphics[width=\columnwidth]{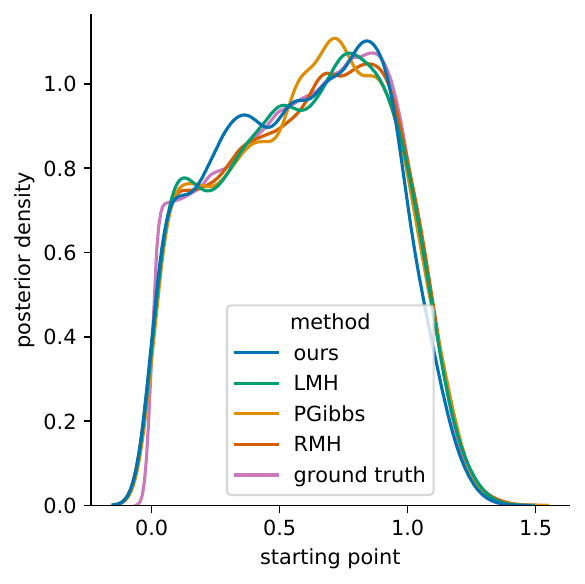}
\begin{tabular}{lllll}
\toprule
method & \textbf{ours} &LMH{} &PGibbs{} &RMH{} \\
\midrule
ESS &\textbf{679} &526 &310 & 508 \\
\bottomrule
\end{tabular}
\caption{Kernel density estimate (top) averaged over 10 runs and estimated effective sample size (bottom) averaged over 10 runs. Each run: $10^3$ NP-DHMC samples with $10^2$ burn-in, 50 leapfrog steps of size 0.1; and $5\times 10^4$ LMH, PGibbs and RMH samples.}
\label{fig:walk}
\end{figure}

This example is interesting and challenging because the true posterior is difficult to determine precisely.
Therefore we took $10^7$ importance samples (effective sample size $\approx 4.4 \times 10^5$) and considered those as the ground truth.
As one can see from \cref{fig:walk}, NP-DHMC comes closest.
Since it is not clear what measure to use for the distance from these ``ground truth'' samples, we instead computed the effective sample size\footnote{We used the standard ESS estimator (based on weighted samples) for the ground-truth importance samples, and an autocorrelation-based MCMC ESS estimator (Sec.~11.5 in \cite{Gelman2014}) for the rest. 
}
for each method (\cref{fig:walk}).
Our method does best in that regard as well.

\begin{figure}
\centering
\includegraphics[width=\columnwidth]{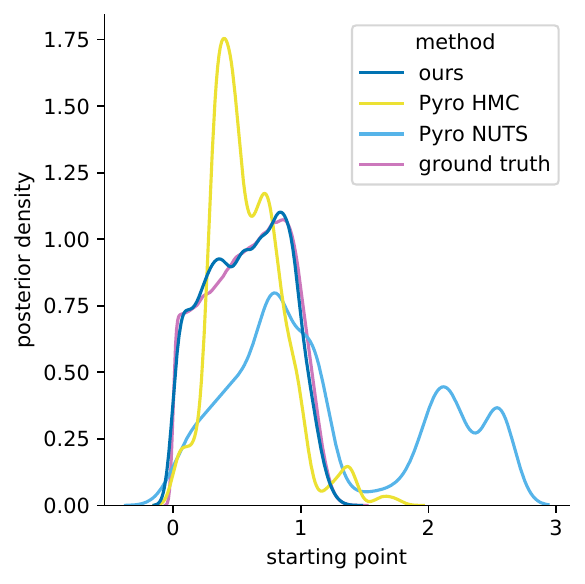}
\vspace{0.1in}
\vspace{-0.3in}
\caption{Kernel density estimate for the random walk example compared to Pyro averaged over 10 runs, some of which were discarded because of low acceptance rate ($< 0.1$). Each run: $10^3$ samples with $10^2$ burn-in, 50 leapfrog steps of size 0.1.}
\label{fig:walk-pyro}
\end{figure}
\changed[lo]{The popular PPL Pyro accepts nonparametric models as input.
We therefore tried to ascertain the performance of its HMC implementation on this example.}
We ran both Pyro's HMC sampler (with the same hyperparameters as ours) and Pyro's No-U-Turn sampler (NUTS) \cite{HoffmanG14}, which aims to automatically infer good hyperparameter settings.
The inferred posterior distributions (\cref{fig:walk-pyro}) are far away from the ground truth, and clearly wrong.
Pyro’s inference was very slow, sometimes taking almost a minute to produce a single sample.
\changed[lo]{For this reason, we didn't run more experiments with it.}

\fz{Let me know what you think about this. Also, does the plot take up too much space?}
\cm{I wonder if we could add the Pyro plots to Fig 5. Don't worry if it's too crowded. Another note about colours. I think it would be nice to use a different colour for the Pyro engines and keep the colour of ground truth consistent with Fig 5.}
\fz{I can do it, but I think it would be too crowded.}

\begin{figure*}
\centering
\includegraphics[width=0.9\textwidth]{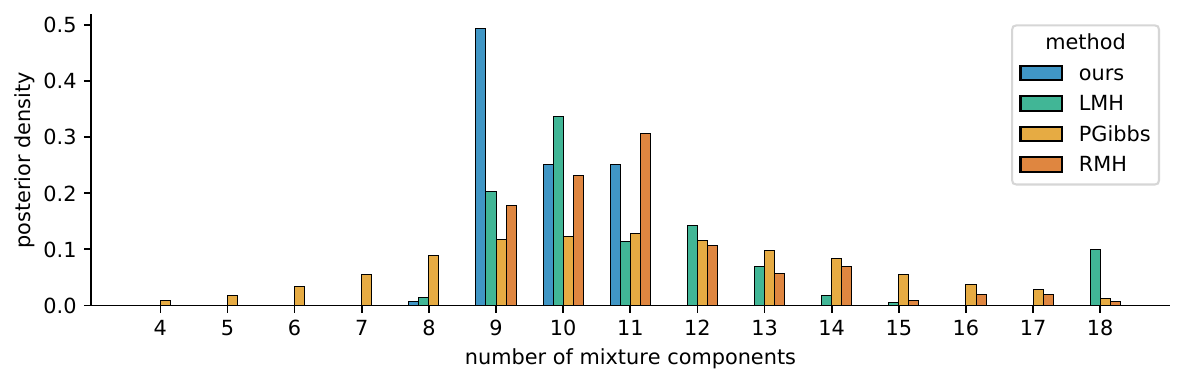}
\vspace{0.1in}
\vspace{-0.3in}
\caption{Histogram of the no.~of mixtures for the GMM example; correct posterior = 9, averaged over 10 runs. Each run: $10^3$ NP-DHMC samples with $10^2$ burn-in, 50 leapfrog steps of size 0.05; and $5\times 10^4$ LMH, PGibbs and RMH samples.}
\label{fig:gmm}
\end{figure*}

\begin{figure}[h]
\centering
\includegraphics[width=\columnwidth]{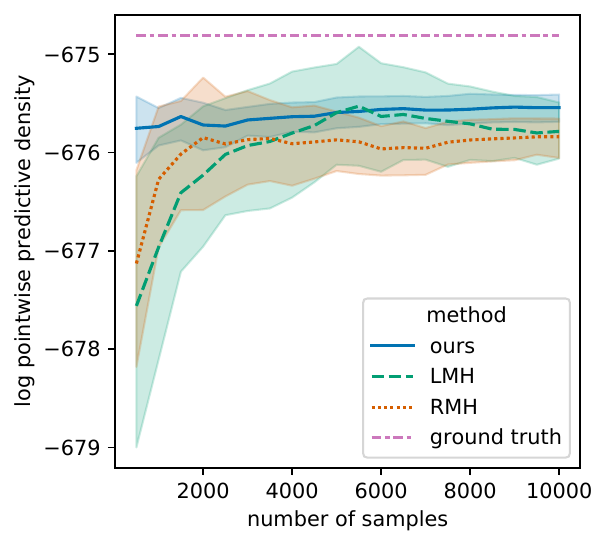}
\vspace{-0.2in}
\caption{LPPD for the GMM example, averaged over 10 runs. Each run: $10^3$ NP-DHMC samples with $10^2$ burn-in, 50 leapfrog steps of size 0.05; and $5\times 10^4$ LMH, PGibbs and RMH samples. The shaded area is one standard deviation. PGibbs (with final LPPD $-716.85 \pm 0.64$) is omitted to show the top contenders clearly.}
\label{fig:gmm-lppd-plot-zoomed}
\end{figure}
\begin{figure}[h]
\centering
\includegraphics[width=\columnwidth]{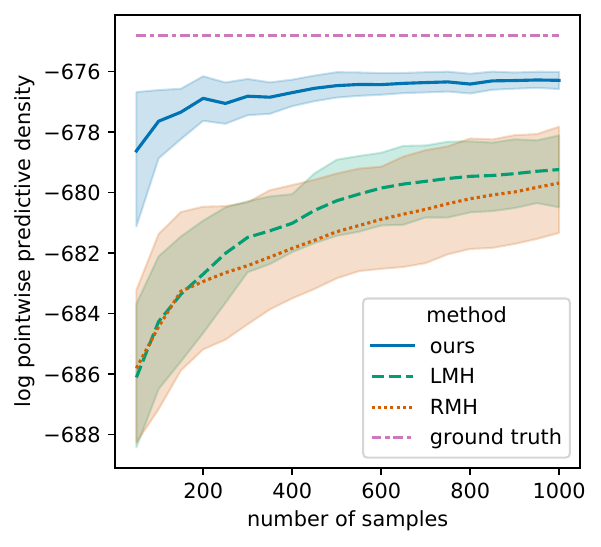}
\vspace{-0.2in}
\caption{LPPD for the DPMM example, averaged over 10 runs. Each run: $10^2$ NP-DHMC samples with $50$ burn-in, 20 leapfrog steps of size 0.05; and $2\times 10^3$ LMH, PGibbs and RMH samples. The shaded area is one standard deviation. PGibbs (with final LPPD $-725.96 \pm 9.83$) is omitted to show the top contenders clearly.}
\label{fig:dpmm-lppd-plot-zoomed}
\end{figure}

\paragraph{Gaussian mixture model}
We also considered the following mixture model adapted from \cite{DBLP:conf/icml/ZhouYTR20}, where the number of mixture components $K$ is unbounded.
\begin{align*}
K &\sim \mathrm{Poisson}(10) + 1 \\
\mu_k &\sim \mathrm{Uniform}([0, 100]^3) &&\text{for } k = 1,\dots,K \\
x_n &\sim \frac{1}{K}\sum_{k=1}^K \Gau_3(\mu_k, 10^2I_3) &&\text{for } n = 1,\dots,N
\end{align*}
This model samples parameters $\theta = (K\in \Nat, \mu_k \in [0,100]^3)$ and $N$ data points $X = \{x_1,\dots,x_N\} \subseteq \Real^3$.
Note that this model uses much higher standard deviations (10 instead of 0.1) for the Gaussian mixture components compared to \cite{DBLP:conf/icml/ZhouYTR20}.
This is to avoid typical problems of MCMC algorithms with multimodal distributions, which is an issue inherent to MCMC algorithms and tangential to this work.
We used this model to generate $N = 200$ training data points for a fixed $\theta^* = (K^* = 9, \mu_{1\dots K^*}^*)$.
We let our inference algorithms sample from $p(\theta \mid X)$ and compared the posterior on the number of mixture components $K$ with the other inference algorithms.
The histogram (\cref{fig:gmm}) shows that NP-DHMC usually finds the correct number of mixture components ($K^* = 9$).

In addition, we computed the log pointwise predictive density (LPPD) for a test set with $N' = 50$ data points $Y = \{y_1,\dots,y_{N'}\}$, generated from the same $\theta^*$ as the training data.
The LPPD is defined as $\sum_{i=1}^{N'} \log \int p(y_i \mid \theta) p(\theta \mid X) \dif\theta$ and can be approximated by $\sum_{i=1}^{N'} \log \frac{1}{M} \sum_{j=1}^M p(y_i \mid \theta_j)$ where $(\theta_j)_{j=1\dots M}$ are samples from $p(\theta \mid X)$ \cite{Gelman2014}.
The results (\cref{fig:gmm-lppd-plot-zoomed}) include the ``true'' LPPD of the test data under the point estimate $\theta^*$.
As we can see, NP-DHMC outperforms the other methods and has the lowest variance over multiple runs.

\paragraph{Dirichlet process mixture model}
Finally, we consider a classic example of nonparametric models: the Dirichlet process $\mathrm{DP}(\alpha, H)$ \cite{Ferguson73}, which is a stochastic process parametrised by concentration parameter $\alpha > 0$ and a probability distribution $H$.
For practical purposes, one can think of samples from $\mathrm{DP}(\alpha, H)$ as an infinite sequence $(w_k, h_k)_{n\in \Nat}$ where $w_n$ are weights that sum to 1 and $h_n$ are samples from $H$.
Conceptually, a DP Gaussian mixture model takes the form
\begin{align*}
(w_k, h_k)_{k\in \Nat} &\sim \mathrm{DP}(\alpha, H) \\
x_n &\sim \sum_{k=1}^\infty w_k \cdot \Gau(h_k, \Sigma) &&\text{for } n = 1,\dots,N
\end{align*}
Sampling from a DP is usually implemented by the stick-breaking method \cite{Sethuraman94}.
However, one cannot actually compute an infinite sequence.
A practical workaround is to cap the number of mixture components by a fixed $K$ and only consider $(w_k, h_k)_{k=1,\dots,K}$.
\footnote{For instance, see \url{https://pyro.ai/examples/dirichlet_process_mixture.html} (accessed: 2021-06-06), the Pyro tutorial on DP mixture models.}
\changed[lo]{This renders the model parametric, enabling the use of standard HMC.
However such a treatment is clearly unsatisfactory: if the data actually requires more mixture components than $K$, the model would be found wanting.}

We propose a different approach.
Instead of choosing a fixed $K$, we allow it to depend on the weights $w_k$:
we pick the minimal $K \in \Nat$ such that $\sum_{i=1}^K w_i > 1 - \epsilon$ for some $\epsilon > 0$.
With this restriction, we only discard insignificant mixture components (with a weight $w_k < \epsilon$) and allow as many mixture components as necessary to model the data accurately.
This model is not parametric anymore, but still tractable by our algorithm.

We implemented the above model with the parameters $\alpha = 5$, $\epsilon = 0.01$, and the remaining parameters as chosen in the previous GMM example, i.e.~$H = \mathrm{Uniform}([0,100]^3)$ and $\Sigma = 10^2I_3$.
We used the same training and test data as in the previous GMM example and present the LPPD results in \cref{fig:dpmm-lppd-plot-zoomed}.
As we can see, NP-DHMC outperforms the other methods and has the lowest variance over multiple runs.

\section{Related Work and Conclusion}
\label{sec:related work}



The standard MCMC algorithm for PPLs that is widely implemented (for example, in Anglican, Venture, and Web PPL) is the Lightweight Metropolis-Hastings (LMH) algorithm and its extensions \cite{DBLP:conf/aistats/YangHG14,DBLP:conf/pkdd/TolpinMPW15,DBLP:conf/aistats/RitchieSG16},
which performs single-site updates on the current sample and re-executes the program.
Unlike NP-HMC, where Hamiltonian motion is simulated on the resulting extended trace,
LMH suffers from a lack of predictive accuracy in its proposal
(as shown in \cref{sec:experiements}).

\cm{Due to space constraints, I don't think we need to talk about the MH algorithm from \citet{HurNRS15}.}

\changed[lo]{The Reversible Jump Markov chain Monte Carlo (RJMCMC) algorithm \cite{Green95} is similar to NP-HMC in that it is a trans-dimensional MCMC sampler.
However, NP-HMC is a \emph{general purpose} inference algorithm that works out-of-the-box when given an input density function, whereas RJMCMC additionally requires the user to specify a transition kernel.
Various RJMCMC transition kernels have been suggested for specific models, e.g.~split-merge proposal for infinite Gaussian mixture models.}

Some PPLs such as Hakaru, Pyro and Gen give users the flexibility to hand-code the proposal in a MCMC setting.
For instance, \citet{cusumanotowner2020automating} implement the split-merge proposal \cite{https://doi.org/10.1111/1467-9868.00095} of RJMCMC in Gen.
Though this line of research is orthogonal to ours,
PPLs such as Gen could play a useful role in the implementation of NP-HMC and similar extensions of inference algorithms to nonparametric models.

The HMC algorithm and its variants, notably the No-U-Turn Sampler,
are the workhorse inference methods in the influential PPL Stan \cite{stan}.
The challenges posed by stochastic branching in PPLs are the focus of reflective/refractive HMC \cite{AfsharD15}; discontinuous HMC \cite{NishimuraDL20}; mixed HMC \cite{DBLP:conf/nips/Zhou20}; and the first-order PPL in \cite{DBLP:conf/aistats/ZhouGKRYW19} which is equipped with an implementation of discontinuous HMC.
\changed[lo]{By contrast, our work is an attempt to tackle the language constructs of branching \emph{and} recursion.}

\changed[lo]{Unlike Monte Carlo methods, variational inference (VI) \cite{BleiKM17} solves the Bayesian inference problem by treating it as an optimisation problem.
When adapted to models expressed as probabilistic programs, the score function VI \cite{RanganathGB14} can in principle be applied to a large class of branching and recursive programs because only the variational density functions need to be differentiable.
Existing implementations of VI algorithms in probabilistic programming systems are however far from automatic:
in the main, the guide programs (that express variational distributions) still need to be hand-coded.}

Recently, \citet{DBLP:conf/icml/ZhouYTR20} introduced the Divide, Conquer, and Combine (DCC) algorithm, which is applicable to programs definable using branching and recursion.
\changed[fz]{As a} hybrid algorithm, DCC solves the problem of designing a proposal that can efficiently transition between configurations by performing local inferences on submodels, and returning an appropriately weighted combination of the respective samples.
Thanks to a judicious resource allocation scheme, it exhibits strong performance on multimodal distributions.

\anon{Zhou: This is not fully true though, as LMH forgets any proposed, not accepted samples whereas DCC ‘remembers’ them. Also DCC uses RMH actually.

I agree that it is possible to make proposals cleverer. However, at the end of the day, if we are not provided more information, we will need to rely on sth. similar to RMH to explore (prior + local movements).

I cannot comment for now as I do not fully follow how NP-HMC extends. But I think it might be good to compare your jumping scheme with LMH, in order to improve clarity. }

\subsection*{Conclusion}
\label{sec:conclusions}

\changed[lo]{We have presented the NP-HMC algorithm, the first extension of the HMC algorithm to nonparametric models.
We have proved that NP-HMC is correct. We have also empirically demonstrated that it enjoys significant performance improvements over state-of-the-art MCMC algorithms for universal PPLs on four nonparametric models, thereby illustrating that the key advantage of HMC---the proposal of moves to distant states with a high acceptance probability---has been preserved by NP-HMC.}

\paragraph{Acknowledgements}

We thank the reviewers for their insightful feedback and pointing out \changed[fz]{important related work}.
We are grateful to Hugo Paquet, Dominik Wagner and Yuan Zhou, who gave detailed comments on an early draft, and to Tom Rainforth and Arnaud Doucet for their helpful comments and advice.
\changed[fz]{We gratefully acknowledge support from the EPSRC}\fz{Not sure how to word this, but EPSRC is funding my DPhil.}\changed[cm]{ and the Croucher Foundation.}
\cm{Thanks to funding agencies and corporate sponsors that provided financial support?}



\bibliography{database.bib}
\bibliographystyle{icml2021}

\ifapxAppended
\clearpage

\onecolumn
\icmltitle{Nonparametric Hamiltonian Monte Carlo (Appendix)}
\vskip 0.3in

\appendix

\section{Statistical PCF}
\label{appendix: SPCF}

In this section, we present a simply-typed
statistical probabilistic programming language with (stochastic) branching and recursion, and its operational semantics.

This language serves two purposes for the NP-HMC algorithm.
First, it is a purified universal probabilistic programming language (PPL) widely considered \cite{DBLP:conf/icfp/BorgstromLGS16,VakarKS19,MakOPW20} which specifies tree-representable functions that satisfies Ass.~\aref{ass:1}{1}, \aref{ass:2}{2} and \aref{ass:3}{3} (\cref{prop: all spcf terms have TR weight function,lemma: all AST SPCF term satisfy all assumptions}) and hence NP-HMC can be applied.
Second, its (operational) semantics is used to prove correctness of NP-HMC in \cref{appendix: correctness}.

\subsection{Syntax}
\label{subsec:stat PCF}

SPCF is a simply-typed higher-order universal PPL with branching and recursion.
More formally, it is a statistical probabilistic version of call-by-value
PCF \cite{DBLP:journals/tcs/Scott93,DBLP:conf/fsttcs/Sieber90} with reals as the ground type.
The terms and part of the typing system of SPCF are presented in \cref{fig:SPCF syntax}.
Free variables and closed terms are defined in the usual way.
In the interest of readability, we sometimes use pseudocode (e.g.~\cref{intro-program}) in the style of Python to express SPCF terms.

There are two probabilistic constructs of SPCF:
the sampling construct $\Normal$ draws from $\Gau$, the standard Gaussian distribution with mean $0$ and variance $1$;
the scoring construct $\Score{\terma}$ enables conditioning on observed data by multiplying the weight of the current execution with the real number denoted by $\terma$.
Note this is not limiting as the standard uniform distribution with endpoints $0$ and $1$ can be described as $\PCF{\mathsf{cdfNormal}}(\Normal)$ where $\mathsf{cdfNormal}$ is the cumulative distribution function (cdf) of the standard normal distribution.
And any real-valued distribution with inverse cdf $f$ can be described as $\PCF{f}(\PCF{\mathsf{cdfNormal}}(\Normal))$.

\begin{remark}
  The main difference between our variant of SPCF and the others
  \cite{VakarKS19,MakOPW20} is that our sampling construct draws from the \emph{standard normal distribution} instead of the standard uniform distribution.
  This does not restrict nor extend our language and is only considered since the target (parameter) space of the standard HMC algorithm matches that of the support of a standard $n$-dimensional normal distribution.
\end{remark}

\begin{figure}[h]
  \defn{Types} (typically denoted $\typea,\typeb$) and
  \defn{terms} (typically $\terma,\termb,\termc $):
  \begin{align*}
    \typea,\typeb & ::=
    \PCFReal \mid \typea \tyarrow \typeb \\
    \terma,\termb,\termc & ::=
    y \mid
    \PCF{r} \mid
    \lambda y.\terma \mid
    \terma\,\termb \mid
    \PCFIf{\termc \leq 0}{\terma}{\termb}\mid
    \PCF{f}(\terma_1,\dots,\terma_\ell) \mid
    \Y{\terma} \mid
    \Normal \mid
    \Score{\terma}
  \end{align*}
  \defn{Typing system}:
  $$
    \AxiomC{\vphantom{$4$}}
    \UnaryInfC{$\Gamma \vdash \Normal:\PCFReal$}
    \DisplayProof
    \qquad
    \AxiomC{$\Gamma \vdash \terma:\PCFReal$}
    \UnaryInfC{$\Gamma \vdash \Score{\terma}:\PCFReal$}
    \DisplayProof
    \qquad
    \AxiomC{$
      \Gamma \vdash \terma : (\typea \tyarrow \typeb) \tyarrow (\typea \tyarrow \typeb)
    $}
    \UnaryInfC{$
      \Gamma \vdash \Y{\terma} : \typea \tyarrow \typeb
    $}
    \DisplayProof
  $$
  \caption{
    Syntax of SPCF, where $r \in \Real$, $x,y$ are variables, and
   $f:\Real^n \to \Real$ ranges over a set $\pop$ of partial, measurable
   primitive functions.
  }
  \label{fig:SPCF syntax}
\end{figure}

\subsection{Operational Semantics}
\label{subsec:operational semantics}

\begin{figure}[t]
  \defn{Values} (typically denoted $\valuea$),
  \defn{redexes} (typically $\redexa$) and
  \defn{evaluation contexts} (typically $\evalcon$):
  \begin{align*}
    \valuea & ::=
    \PCF{r} \mid \lambda y.M \\
    \redexa & ::=
              (\lambda y\ldotp\terma)\,\valuea \mid
              \Ifleq{\PCF{r}}{\terma}{\termb} \mid
              \PCF{f}(\PCF{r_1},\dots,\PCF{r_\ell}) \mid
              \Y{(\lambda y\ldotp\terma)}\mid
              \Normal \mid
              \Score{\PCF{r}}\\
    \evalcon & ::=
               [] \mid
               \evalcon\,\terma \mid
               (\lambda y.\terma)\,\evalcon \mid
               \Ifleq{\evalcon}{\terma}{\termb}\mid
               \PCF{f}(\PCF{r_1}, \dots, \PCF{r_{i-1}},\evalcon,\terma_{i+1},\dots,\terma_\ell) \mid
               \Y{\evalcon} \mid
              \Score{\evalcon}
  \end{align*}
  \noindent\defn{Redex contractions}:
  \begin{align*}
    \config{(\lambda y.\terma)\,\valuea}{w}{\trace} & \red
      \config{\terma[\valuea/y]}{w}{\trace} \\
    \config{\PCF{f}(\PCF{r_1},\dots,\PCF{r_\ell})}{w}{\trace} & \red
      \begin{cases}
        \config{\PCF{f(r_1,\dots,r_\ell)}}{w}{\trace}
        & \text{if } (r_1,\dots,r_\ell) \in \domain{f},\\
        \Fail & \text{otherwise.}
      \end{cases}
      \\
    \config{\Y{(\lambda y.\terma)}}{w}{\trace} & \red
      \config{\lambda z.\terma[\Y{(\lambda y.\terma)}/y]\,z}{w}{\trace}
      \tag{for fresh variable $z$} \\
    \config{\Ifleq{\PCF{r}}{\terma}{\termb}}{w}{\trace} & \red
      \begin{cases}
        \config{\terma}{w}{\trace}
        & \text{if } r \leq 0,\\
        \config{\termb}{w}{\trace} & \text{otherwise.}
      \end{cases}\\
    \config{\Normal}{w}{\trace} & \red
      \config{\PCF{r}}{w}{\trace\concat [r]} \tag{for some $r \in \Real$} \\
    \config{\Score{\PCF{r}}}{w}{\trace} & \red
      \begin{cases}
        \config{\PCF{r}}{r\cdot w}{\trace}
        & \text{if } r > 0,\\
        \Fail & \text{otherwise.}
      \end{cases}
  \end{align*}
  \noindent\defn{Evaluation contexts}:
  $$
    \AxiomC{$\config{\redexa}{w}{\trace} \red \config{\contra}{w'}{\trace'}$}
    \UnaryInfC{$\config{E[\redexa]}{w}{\trace} \red \config{E[\contra]}{w'}{\trace'}$}
    \DisplayProof
    \qquad
    \AxiomC{$\config{\redexa}{w}{\trace} \red\Fail$}
    \UnaryInfC{$\config{E[\redexa]}{w}{\trace} \red\Fail$}
    \DisplayProof
  $$
  \caption{
    Operational small-step semantics of SPCF}
  \label{fig:operational small-step}
\end{figure}

The small-step reduction of SPCF is standard (see \citet{DBLP:conf/icfp/BorgstromLGS16}).
We present it as a rewrite system of \defn{configurations}, which are triples of the form $\config{\terma}{w}{\trace}$ where $M$ is a closed SPCF term, $w \in \pReal$ is a \defn{weight}, and $\trace \in \traces$ a trace, as defined in \cref{fig:operational small-step}.
\fz{Shouldn't top and bottom of the evaluation context rules be flipped?}

In the rule for $\Normal$, a random value $r \in \Real$ is generated and recorded in the trace, while the weight remains unchanged:
even though the program samples from a normal distribution, the weight does not factor in Gaussian densities as they are already accounted for by $\measure{\traces}$.
In the rule for $\Score{\PCF{r}}$, the current weight is multiplied by $r\in\Real$: typically this reflects the likelihood of the current execution given some observed data.
Similarly to \cite{DBLP:conf/icfp/BorgstromLGS16} we reduce terms which cannot be reduced in a reasonable way (i.e.~scoring with nonpositive constants or evaluating functions outside their domain) to $\Fail$.

We write $\redplus$ for the transitive closure of $\red$, and
$\red^*$ for the reflexive and transitive closure of $\red$.

\subsubsection{Value and Weight Functions.}
\label{appendix: value and weight functions}

Recall the measure space of traces
$\traces := \bigcup_{n\in\Nat} \Real^n$ is equipped with the standard disjoint union $\sigma$-algebra
$\Sigma_\traces := \set{\bigcup_{n\in\Nat} U_n \mid U_n \in \Borel_n }$, with measure given by summing the respective (higher-dimensional) normals $\tmeasure(\bigcup_{n\in\Nat} U_n) := \sum_{n\in\Nat} \Gau_n(U_n)$.
Following \citet{DBLP:conf/icfp/BorgstromLGS16},
we write $\terms$ to denote the set of all SPCF terms and view it as $\bigcup_{n\in\Nat} (\sk_n \times \Real^n)$
where $\sk_n$ is the set of SPCF terms with exactly $n$ numerals place-holders.
The measurable space of terms
is equipped with the $\sigma$-algebra $\Sigma_{\terms}$ that is the Borel algebra of the
countable disjoint union topology of the product topology of the discrete topology on
$\sk_n$ and the standard topology on $\Real^n$.
Similarly the subspace $\closedvalues$ of closed values inherits the Borel algebra on $\terms$.

Let $\terma$ be a closed SPCF term.
Its \defn{value function} $\valuefn_{\terma} : \traces \to \closedvalues \cup \{\bot\}$ returns, given a trace, the output value of the program, if the program terminates in a value.
The \defn{weight function} $\weightfn_{\terma} : \traces \to \pReal$ returns the final weight of the corresponding execution. Formally:
\begin{align*}
  \valuefn_{\terma} (\trace) & :=
  \begin{cases}
  V & \hbox{if $\config{\terma}{1}{\emptytrace} \red^*
      \config{\valuea}{w}{\trace}$}\\
  \bot & \tow
  \end{cases}
  &
  \weightfn_{\terma} (\trace) & :=
  \begin{cases}
  w & \hbox{if
    $\config{\terma}{1}{\emptytrace} \red^* \config{\valuea}{w}{\trace}$}\\
  0 & \tow
  \end{cases}
\end{align*}
It follows already from \cite{DBLP:conf/icfp/BorgstromLGS16} that the functions $\valuefn_\terma$ and $\weightfn_\terma$ are measurable.

Finally, every closed SPCF term $\terma$ has an associated \defn{value measure}
\begin{align*}
  {\oper{\terma}}: {\Sigma_{\closedvalues}} &\longrightarrow{\pReal}\\
   U & \longmapsto
    \shortint{\inv{\valuefn_\terma}(U)}
    {\weightfn_\terma}
    {\tmeasure}
\end{align*}
This corresponds to the denotational semantics of SPCF in the $\omega$-quasi-Borel space model via computational adequacy \cite{VakarKS19}.

\begin{proposition}
  \label{prop: all spcf terms have TR weight function}
  Every closed SPCF term has a tree representable weight function.
\end{proposition}

\begin{proof}
  Assume $\terma$ is a closed SPCF term and $\seqa \in \nsupport{\weightfn_\terma}{n}$.
  The reduction of $\terma$ must be
  $\config{\terma}{1}{\emptytrace} \red^* \config{\valuea}{w}{\seqa}$ for some value $\valuea$ and weight $w > 0$.
  Assume for contradiction that there is some $k < n$ where
  $\config{\terma}{1}{\emptytrace} \red^* \config{\valuea'}{w'}{\seqrange{1}{k}}$ for some value $\valuea'$ and weight $w' > 0$.
  Since $\seqrange{1}{k}$ is a prefix of $\seqa$ and $\red$ is deterministic if the trace is given,
  we must have
  $
  \config{\terma}{1}{\emptytrace}
  \redplus \config{\valuea'}{w'}{\seqrange{1}{k}}
  \redplus \config{\valuea}{w}{\seqa},
  $
  which contradicts the fact that $\valuea'$ is a value.
\end{proof}

\subsection{Almost-sure Termination}


\begin{definition}\rm
  \label{def:ast}
  We say that a SPCF term $M$ \defn{terminates almost surely} if
    $\terma$ is closed and
    $\tmeasure(\set{\trace \in \traces \mid \exists V, w \,.\,
    \config{\terma}{1}{\emptytrace} \red^* \config{\valuea}{w}{\trace}}) =1$;
\end{definition}

The following proposition is used in \cref{prop: step 1 and 2 is probabilistic} to support the correctness proof.

\begin{proposition}
  \label{prop: AST SPCF term gives probability measure}
  The value measure $\oper{\terma}$ of
  a closed almost surely terminating SPCF term $\terma$ which does not contain $\Score{-}$ as a subterm
  is probabilistic.
\end{proposition}

One of the main contribution of \cite{MakOPW20} is to find a suitable class of primitive functions such that their main theorem (\cref{lemma: weight and value functions are diff ae}) holds.

For our purposes, we take the set of \defn{analytic functions} with co-domain $\Real$ as our class $\pop$ of primitive functions which, as shown in Example 3 of \cite{MakOPW20}, satisfies the conditions for which the following lemma holds.

\begin{lemma}[\citet{MakOPW20}, Theorem 3]
  \label{lemma: weight and value functions are diff ae}
  Let $\terma$ be an SPCF term which terminates almost surely.
  Then its weight function $\weightfn_{\terma}$ and value function $\valuefn_{\terma}$ are differentiable almost everywhere.
\end{lemma}

\begin{definition}\rm
  \label{def:integrable}
  We say that a SPCF term $M$ is \defn{integrable} if
      $\terma$ is closed and
      its value measure is finite,
      i.e.~$\oper{M}(\closedvalues) < \infty$;
\end{definition}

We conclude with the following lemma which shows that NP-HMC is an adequate inference algorithm for closed SPCF terms.

\begin{lemma}
  \label{lemma: all AST SPCF term satisfy all assumptions}
  The weight function of a closed integrable almost surely terminating SPCF term satisfies Assumptions \aref{ass:1}{1}, \aref{ass:2}{2} and \aref{ass:3}{3} of the NP-HMC algorithm.
\end{lemma}

\begin{proof}
  Let $\terma$ be a closed integrable almost surely terminating SPCF term,
  and $\btra$ be its weight function.
  $\btra$ is tree representable by \cref{prop: all spcf terms have TR weight function}.
  Integrability of $\btra$ (Assumption \aref{ass:1}{1}) is given as an assumption, and
  $\btra$ is almost everywhere continuously differentiable (Assumption \aref{ass:2}{2}) by \cref{lemma: weight and value functions are diff ae}.

  Assume for contradiction that Assumption \aref{ass:3}{3} does not hold.
  i.e.~There is a non-null set $U$ of infinite real-valued sequence where $\btra$ is zero on all prefixes of sequences in $U$.
  Let $U_p := \set{\seqrange{1}{k} \mid \seqa \in U, k \in \Nat}$ be the set of prefixes of sequences in $U$.
  Since $U$ is non-null, $U_p$ must also be non-null.
  Moreover, $\btra$ is zero on all traces in $U_p$.
  By the definition of weight function,
  $\seqa \in U_p$ implies
  $\config{\terma}{1}{\emptytrace} \not\red^* \config{\valuea}{w'}{\seqa}$
  for some $\valuea$ and $w'$.
  Hence, the probability of a non-terminating run of $\terma$ is non-zero and
  $\terma$ is not almost surely terminating.
\end{proof}

\begin{remark}
  The weight function as defined in \cref{appendix: value and weight functions} is the input density function of the target distribution to which an inference algorithm typically samples from.
  In this paper, we call this function the ``weight function'' when considering semantics following \cite{DBLP:conf/esop/CulpepperC17,VakarKS19,MakOPW20},
  and use the notion ``density'' when referring it in an inference algorithm similar to \cite{DBLP:conf/aistats/ZhouGKRYW19,DBLP:conf/icml/ZhouYTR20,cusumanotowner2020automating}.
\end{remark}

\section{Hamiltonian Monte Carlo Algorithm and its Variants}
\label{appendix: HMC}

Hamiltonian Monte Carlo (HMC) algorithm \cite{DUANE1987216,CancesLS07,Neal2011}
is a Markov chain Monte Carlo inference algorithm that
generates
samples from a continuous (finite) distribution $\tdist$ on
the \changed[lo]{measure space} $(\Real^n, \Borel_n, \leb_n)$, where $\Borel_n$ denotes the Borel $\sigma$-algebra.

\subsection{HMC Algorithm}

To generate a Markov chain $\set{\vec{q_i}}_{i\in\Nat}$ of samples from $\tdist$,
HMC simulates the \emph{Hamiltonian} motion of a particle on the negative logarithm of the density function of $\tdist$ with some auxiliary momentum.
Hence regions with high probability in $\tdist$ have low potential energy and are more likely to be visited by the simulated particle.
In each iteration, the particle is given some random momentum.
We formalise the algorithm here.

\subsubsection{Hamiltonian Dynamics}
\label{sec: Hamilton eq}

Say $\tpdf:\Real^n \to \Real$ is the (not necessarily normalized) probability density function of $\tdist$.
The simulated particle has two types of energies:
\defn{potential energy} $U:\Real^n \to\Real$ given by $U(\vec q) := -\log \tpdf(\vec q)$ and
\defn{kinetic energy} $K:\Real^n \to\Real$ given by $K(\vec p) := -\log \pdf{D}(\vec p)$ where
$D$ is some momentum distribution, typically a $n$-dimensional normal distribution.
Henceforth, we take
$K(\vec p) := \sum_{i=1}^n \frac{\vec{p}_i^2}{2}$.

The \defn{Hamiltonian} $H: \Real^n \times \Real^n \to \pReal $ of a system is defined quite simply to be the sum of the potential and kinetic energies, i.e.~
\[
  H(\vec{q},\vec{p}) := U(\vec{q}) + K(\vec{p}).
\]

The trajectories $\set{(\vec{q}^t,\vec{p}^t)}_{t \geq 0}$, where $\vec{q}^t$ and $\vec{p}^t$ are the position and momentum of the particle at time $t$ respectively, defined by the Hamiltonian $H$, can be determined by the \defn{Hamiltonian equations}:
\[
  \frac{\dif \vec{q}(t)}{\dif t} := \frac{\partial H}{\partial \vec{p}}(\vec{q}(t),\vec{p}(t))
  = \grad{K}(\vec{p}(t)) = \vec{p}(t) \qquad\text{and}\qquad
  \frac{\dif \vec{p}(t)}{\dif t} := -\frac{\partial H}{\partial \vec{q}}(\vec{q}(t),\vec{p}(t))
  = -\grad{U}(\vec{q}(t)).
\]
with initial conditions $(\vec{q}(0),\vec{p}(0)) = (\vec{q}^0,\vec{p}^0)$.

The \defn{canonical distribution} (also called Boltzmann-Gibbs distribution) $\sdist$
on the measure space $(\Real^n \times \Real^{n},\Sigma_{\Real^n \times \Real^{n}},\leb_{2n})$
corresponding to $H$
is given by the probability density function
\[
  \spdf(\vec{q},\vec{p}) := \frac{1}{Z} \exp{(-H(\vec{q},\vec{p}))} = \frac{1}{Z} \exp{(-U(\vec{q})-K(\vec{p}))}
  \qquad
  \text{where }Z := \shortint{\Real^n}{\tpdf}{\leb_n}
\]

\subsubsection{The Algorithm}
\label{sec:HMC algorithm}

Since computers cannot simulate continuous motions like Hamiltonian,
the equations of motion are generally numerically integrated by the \defn{leapfrog} method
(also called the velocity-Verlet algorithm \cite{PhysRev.159.98}):
\begin{align*}
  \vec{p}^{n+1/2} & = \vec{p}^n - \epsilon / 2 \cdot \grad{U}(\vec{q}^n) \\
  \vec{q}^{n+1} & = \vec{q}^n + \epsilon \cdot \vec{p}^{n+1/2} \\
  \vec{p}^{n+1} & = \vec{p}^{(n+1)/2} - \epsilon / 2 \cdot \grad{U}(\vec{q}^{n+1})
\end{align*}
where $\epsilon$ is the time step.

The \defn{integrator} ${\HMCint_n}:\Real^n \times \Real^n \to \Real^n \times \Real^n$
as given in \cref{alg:hmc integrator},
takes a state $(\vec{q},\vec{p})$ and
performs $L$ leapfrog steps with initial condition $(\vec{q}^0,\vec{p}^0) := (\vec{q},\vec{p})$ and time step $\epsilon$, and
return the state $(\vec{q}^{L},-\vec{p}^{L})$.

\begin{figure}[t!]
  \centering
  \vspace{-3mm}
  \begin{minipage}{0.45\linewidth}
    \begin{algorithm}[H]
      \caption{{HMC Integrator $\HMCint_n$}}
      \label{alg:hmc integrator}
      \begin{algorithmic}
        \STATE {\bfseries Input:}
          current state $(\vec{q_0}, \vec{p_0})$,
          potential energy $U$,
          step size $\epsilon$,
          number of steps {$L$}
        \STATE {\textbf{Output:}}
          new state $(\vec q, \vec p)$
        \algrule
        \STATE $(\seqa,\vec{p}) = (\vec{q_0}, \vec{p_0})$ \hfill \COMMENT{initialise}
        \FOR{$i=0$ {\bfseries to} $L$}
          \STATE $\vec{p}=\vec{p}-\frac{\epsilon}{2}\grad{U}(\seqa)$ \hfill\COMMENT{1/2 momentum step}
          \STATE $\seqa=\seqa+\epsilon \, \vec{p}$ \hfill\COMMENT{1 position step}
          \STATE $\vec{p}=\vec{p}-\frac{\epsilon}{2}\grad{U}(\seqa)$ \hfill\COMMENT{1/2 momentum step}
        \ENDFOR
        \STATE $\vec{p} = -\vec{p}$
        \STATE {\bfseries return} $(\seqa, \vec{p})$
      \end{algorithmic}
    \end{algorithm}
  \end{minipage}
  \hspace{0.05\linewidth}
  \begin{minipage}{0.45\linewidth}
    \begin{algorithm}[H]
      \caption{HMC Step}
      \label{alg:hmc}
      \begin{algorithmic}
        \STATE {\bfseries Input:}
          current sample $\vec{q_0}$,
          potential energy $U$,
          step size $\epsilon$,
          number of steps $L$
        \STATE \textbf{Output:}
          next \changed[fz]{sample $\vec{q}$}
        \algrule
        \STATE $\vec{p_0} \sim \Gau_{n}$
          \hfill \COMMENT{Kick}
        \STATE {$(\seqa, \vec{p}) =
          {\HMCint_n}((\vec{q_0}, \vec{p_0}), U, \epsilon, L)$}
          \hfill \COMMENT{Integrate}
        \IF{$\Uni(0,1)\footnote{$\Uni(0,1)$ is the standard uniform distribution.} < \min \set{1, \frac{\spdf(\vec{q},\vec{p})}{\spdf(\vec{q_0},\vec{p_0})}} $}
        \STATE {\bfseries return} $\seqa$
          \hfill \COMMENT{MH acceptance ratio}
        \ELSE
        \STATE {\bfseries return} $\vec{q_0}$
        \ENDIF
      \end{algorithmic}
    \end{algorithm}
  \end{minipage}
\end{figure}

\begin{proposition}[\citet{bou-rabee_sanz-serna_2018}, Theorem 4.1 and 4.2]
  \label{prop: hmc integrator is volume preserving and reversible}
  The integrator ${\HMCint_n}$ is volume preserving (i.e.~${\HMCint_n}_*\leb_{2n} = \leb_{2n}$) and reversible (i.e.~${\HMCint_n} = \inv{{\HMCint_n}}$) on $\Real^n \times \Real^n$.
\end{proposition}

\begin{proof}
  Let $\phi_k^P, \phi_k^Q:\Real^{2n}\to\Real^{2n}$ be the transition of momentum and position variables with step size $k$ respectively,
  i.e.~$\phi_k^P(\vec{q},\vec{p}) = (\vec{q},\vec{p}-k \grad{U}(\vec{q}))$, and
  $\phi_k^Q(\vec{q},\vec{p}) = (\vec{q} + k \grad{K}(\vec{p}),\vec{p})$.
  Hence, we can write the integrator ${\HMCint_n}$ as the composition $S \circ \phi_{\epsilon/2}^P \circ \phi_{\epsilon}^Q \circ \phi_{\epsilon/2}^P$,
  where $S(\vec{q},\vec{p}) := (\vec{q},-\vec{p})$.

  It is easy to see that
  $\inv{(\phi_k^P)} = S \circ \phi_k^P \circ S$ and
  $\inv{(\phi_k^Q)} = S \circ \phi_k^Q \circ S$.
  Hence,
  $
    \inv{{\HMCint_n}}
    =
    \inv{(\phi_{\epsilon/2}^P)} \circ \inv{(\phi_\epsilon^Q)} \circ \inv{(\phi_{\epsilon/2}^P)} \circ S
    =
    S \circ {\phi_{\epsilon/2}^P} \circ {\phi_\epsilon^Q} \circ {\phi_{\epsilon/2}^P}
    = {\HMCint_n}
  $
  and ${\HMCint_n}$ is reversible.

  Similarly it is easy to see that the \changed[fz]{shear} transformations $\phi_k^P$, $\phi_k^Q$ and momentum flip $S$ preserves measure on $\Real^{2n}$,
  i.e.~$\phi_k^P(D)$, $\phi_k^Q(D)$, $S(D)$ and $D$ have the same measure for all measurable set $D$ in $\Real^{2n}$.
  Hence,
  $
    ({\HMCint_n}_*\leb_{2n})(D)
    =
    \leb_{2n}(\inv{{\HMCint_n}}(D))
    =
    \leb_{2n}({\HMCint_n}(D))
    =
    \leb_{2n}(D)
  $
  and ${\HMCint_n}$ is volume preserving.
\end{proof}

\cref{alg:hmc} shows how HMC generates a sample from the current one $\vec{q_0}$.
It first performs leapfrog steps on $(\vec{q_0},\vec{p_0})$ via the integrator $\HMCint_n$
with a randomly chosen initial momentum $\vec{p_0}$.
The result $(\vec{q},\vec{p})$ of $\HMCint_n$ is then accepted with probability
$\min \set{1, \frac{\spdf(\vec{q},\vec{p})}{\spdf(\vec{q_0},\vec{p_0})}}$.
Note that if Hamiltonian is preserved
(i.e.~$H(\vec{q},\vec{p}) = H(\vec{q_0},\vec{p_0})$),
the acceptance probability is one and the proposal will always be accepted.

A Markov chain $\set{\vec{q_i}}_{i\in\Nat}$ is generated by iterating \cref{alg:hmc}.

\subsubsection{Correctness}
\label{sec:HMC correctness}

The HMC algorithm is only effective if its generated Markov chain $\set{\vec{q_i}}_{i\in\Nat}$ does converge to the target distribution $\tdist$.
Here we consider the typical convergence result of the total variation norm for the probability measure generated.

Formally, we say a Markov chain $\set{\vec{q_i}}_{i\in\Nat}$ converges to the target distribution $\tdist$ on $\Real^n$ if
\[
  \forall \vec{q}\in\Real^n, \quad \lim_{m\to\infty} \norm{Q^m(\vec{q},-)-\tdist} = 0,
\]
where
$Q^m(\vec{q},A)$ is the probability for which the Markov chain is in $A\in\Borel_n$ after $m$ steps starting at $\vec{q}\in\Real^n$ and
$\norm{-}$ denotes the total variation norm on $\Real^n$ (i.e.~$\norm{\mu} := \sup_{A \in \Borel_n} \mu(A) - \inf_{A \in \Borel_n} \mu(A)$).

Here we present the necessary conditions to prove such a result for the HMC algorithm.
Let $Q:\Real^n \times \Borel_n \to\pReal$ be the transition kernel specified by \cref{alg:hmc},
so that $Q(\vec{q},A)$ is the probability for which the next sample returned by \cref{alg:hmc} is in $A \in \Borel_n$ given the current sample is $\vec{q}\in\Real^n$.
We write $Q^m$ to be $m$ compositions of $Q$.
(i.e.~$Q^0(\vec{q},A) := [\vec{q} \in A]$; for $k > 0$, $Q^{k+1}(\vec{q},A) := \int_{\Real^n}{Q^{k}(\vec{q'},A)}\ {Q(\vec{q},\dif \vec{q'})}$).

First, we make sure that $\tdist$ is the invariant distribution of the Markov chain.

\begin{proposition}[\citet{bou-rabee_sanz-serna_2018}, Theorem 5.2]
  $\tdist$ is invariant against $Q$.
\end{proposition}

While showing $\tdist$ is the invariant distribution for the Markov chain is relatively simple,
we would be wrong to think that convergence follows trivially.
In fact, as shown in the following example, the Markov chain can easily be periodic.

\begin{example}[\citet{bou-rabee_sanz-serna_2018}, Example 5.1]
  \label{ex: not always ergodic}
  Consider the case where the target distribution is a (unnormalised) one-dim.~normal distribution.
  In particular say the potential energy is
  $U(q) := q^2/2$.
  Then, the Hamiltonian flow ($H(q,p) = U(q) + K(p) = q^2/2 + p^2/2$)
  is a rotation in the $(q,p)$-plane with period $2\pi$.
  If the duration of the simulation is $\pi$, the exact flow returns $q_1 = -q_0$.
\end{example}

\cm{Find a concrete $L$ and $\epsilon$ for this example.}

There are known conditions
for which HMC converges to the right distribution \cite{schutte1999}.
Here we follow the treatment given by \citet{CancesLS07}.

Results from \cite{Tierney94,DBLP:conf/icfp/BorgstromLGS16} tell us that it is enough to show that the transition kernel $Q$ is \defn{strongly $\tdist$-irreducible}: for all $a$ and $B$, $\tdist(B) > 0$ implies $Q(a,B) > 0$.

\begin{lemma}[\citet{CancesLS07}, Lemma 2 and 3 (Strong irreducibility)]
  \label{lemma: HMC irreducbile}
  Assume $U$ is continuously differentiable, bounded above on $\Real^n$ and
  $\grad{U}$ is globally Lipschitz.
  Then the transition kernel $Q$ is strongly $\tdist$-irreducible.
\end{lemma}

\begin{lemma}[\citet{DBLP:conf/icfp/BorgstromLGS16}, Lemma 33 (Aperiodicity)]
  \label{lemma: HMC aperiodic}
  A strongly $\tdist$-irreducible transition kernel is also $\tdist$-aperiodic.
\end{lemma}

\begin{restatable}[\citet{Tierney94}, Theorem 1 and Corollary 2]{lemma}{tierney}
  \label{lemma: Tieryney}
  If the transition kernel $Q$ with invariant distribution $\tdist$
  is $\tdist$-irreducible and $\tdist$-aperiodic, then
  for all $\vec{q}$, $\lim_{n\to\infty} \norm{Q^n(\vec{q},-)-\tdist} = 0$.
\end{restatable}

\begin{theorem}
  \label{thm: hmc converges}
  If $U$ is continuously differentiable, bounded above on $\Real^n$ and
  $\grad{U}$ is globally Lipschitz,
  the Markov chain generated by iterating \cref{alg:hmc} converges to the target distribution $\tdist$.
\end{theorem}

\subsection{HMC Variants}
\label{sec: hmc variants}

\subsubsection{Reflective/Refractive HMC}

\newcommand{\nextBoundary}{\ensuremath{\mathsf{nextBoundary}}}
\newcommand{\decompose}{\ensuremath{\mathsf{decompose}}}

Reflective/refractive HMC (RHMC) \cite{AfsharD15} is an extension of HMC that improves its behaviour for discontinuous density functions.
Standard HMC is correct for such distributions as well, but the acceptance probability may be very low and convergence extremely slow.

We need to quickly discuss what discontinuities mean in our setting:
In addition to discontinuities of each $U_n: \Real^n \to \Real$ itself, we also regard it as a discontinuity when $\vec q$ leaves the support of $U_n$, since this means that a different branch in the tree representing function is chosen.
The set of these discontinuities is $\partial \support{w}$, i.e. the boundary of the support of the density function.

Fortunately, the extension of RHMC to our nonparametric setting is straightforward.
The algorithm is described in \cref{alg:np-rhmc-integrator}.
The only relevant difference is the need for an $\extend$ call in the algorithm.

The rest of the algorithm is the same as \cite{AfsharD15}:
It uses two additional functions that deal with the discontinuities of $U$: \decompose{} and \nextBoundary.
Just like in \cite{AfsharD15}, we assume that these are given to the algorithm because their implementation depends on the kind of discontinuities in the density function.
In the original paper, they only consider discontinuities that are given by affine subspaces.

The function $\nextBoundary(\vec q, \vec p, T, U)$ takes a position $\vec q \in \Real^n$, a momentum $\vec p \in \Real^n$, a time limit $T > 0$, and family of potential energies $\set{U_n}_{n\in \Nat}$.
It then checks whether a particle starting at $\vec q$ moving with momentum $\vec p$ will hit a discontinuity of $U$ in time $\le T$.
If so, it returns the time $t$ of ``impact'', the position $\vec q_<$ just before the discontinuity and $\vec q_>$ just after the discontinuity.

The function $\decompose(\vec q, \vec p, U)$ takes a position $\vec q$ on the discontinuity, a momentum $\vec p$, and $U$ as before.
It then decomposes the momentum $\vec p$ into a component $\vec p_\parallel$ that is parallel to the discontinuity and $\vec p_\perp$ that is perpendicular to it.

The basic idea of the algorithm is inspired by reflection and refraction in physics.
We simulate the trajectory of a particle according to Hamiltonian dynamics.
When hitting a discontinuity, we compute the potential difference. If the kinetic energy is big enough to overcome it, refraction occurs: the perpendicular component of $\vec p$ is scaled down.
Otherwise, the particle is reflected.

The only difference to the original algorithm in \cite{AfsharD15} is the call to $\extend$.
Why is it necessary?
When hitting a discontinuity (and only then!), we may have to switch to a different branch on the tree representing the density function.
Hence we may have to extend the position $q_>$ just after the discontinuity, which is why we call extend on it.

\newcommand{\randomlyPermute}{\ensuremath{\mathsf{randomlyPermute}}}
\newcommand{\coordIntegrator}{\ensuremath{\mathsf{coordIntegrator}}}

\begin{figure}[t!]
  \centering
  \vspace{-3mm}
  \begin{minipage}{0.49\linewidth}
    \begin{algorithm}[H]
      \caption{NP-RHMC Integrator $\NPRInt$}
      \label{alg:np-rhmc-integrator}
      \begin{algorithmic}
        \STATE {\bfseries Input:}
          current state $(\vec{q_0}, \vec{p_0})$,
          family of potential energies $\set{U_n}_{n\in \Nat}$,
          step size $\epsilon$,
          number of steps {$L$}
        \STATE \textbf{Output:}
          new state $(\vec q, \vec p)$ computed according to Hamiltonian dynamics,
          extended initial state $(\vec q_0, \vec p_0)$
        \algrule
        \STATE $(\vec q,\vec{p}) = (\vec{q_0}, \vec{p_0})$ \hfill \COMMENT{initialise}
        \FOR{$i=0$ {\bfseries to} $L$}
          \STATE $\vec{p}=\vec{p}-\frac{\epsilon}{2}\grad{U_{\len{\vec{q_0}}}}(\seqa)$ \hfill\COMMENT{1/2 momentum step}
          \STATE $t = 0$ \hfill\COMMENT{start of position step}
          \WHILE{$\nextBoundary(\vec q, \vec p, \epsilon - t, U)$ exists}
            \STATE $(t', \vec q_<, \vec q_>) = \nextBoundary(\vec q, \vec p, \epsilon - t, U)$
            \STATE $t = t + t'$
            \STATE $((\vec q', \vec p'), (\vec q_0', \vec p_0')) = \extend((\vec q_>, \vec p), (\vec q_0, \vec p_0), i \epsilon + t, U)$
            \STATE $\Delta U = (U_{\len{\vec q'}}(\vec q') - U_{\len{\vec q_<}}(\vec q_<))$
            \IF{$\|\vec p_\perp\|^2 > 2 \Delta U$}
              \STATE $(\vec p_\parallel, \vec p_\perp) = \decompose(\vec q', \vec p', U)$
              \STATE $\vec p_\perp = \sqrt{\|\vec p_\perp\|^2 - 2 \Delta U} \frac{\vec p_\perp}{\|\vec p_\perp\|}$ \hfill\COMMENT{refraction}
              \STATE $\vec q = \vec q'$
            \ELSE
              \STATE $(\vec p_\parallel, \vec p_\perp) = \decompose(\vec q_<, \vec p, U)$
              \STATE $\vec p_\perp = -\vec p_\perp$\hfill\COMMENT{reflection}
              \STATE $\vec q = \vec q_<$
            \ENDIF
            \STATE $\vec p = \vec p_\perp + \vec p_\parallel$
          \ENDWHILE
          \STATE $\vec q =\vec q + (\epsilon - t)\vec{p}$ \hfill\COMMENT{rest of position step}
          \STATE $\vec{p}=\vec{p}-\frac{\epsilon}{2}\grad{U_{\len{\vec q}}}(\seqa)$ \hfill\COMMENT{1/2 momentum step}
        \ENDFOR
        \STATE $\vec{p} = -\vec{p}$
        \STATE {\bfseries return} $((\seqa, \vec{p}), (\vec{q_0}, \vec{p_0}))$
      \end{algorithmic}
    \end{algorithm}
    \begin{algorithm}[H]
      \caption{$\extend$ for NP-DHMC}
      \label{alg: extend for np-dhmc}
      \begin{algorithmic}
        \STATE {\bfseries Input:}
          current state $(\vec{q},\vec{p})$,
          initial state $(\vec{q_0},\vec{p_0})$,
          time $t$,
          family of potential energies $U = \set{U_n}_{n\in\Nat}$
          family of potential energies $\set{U_n}_{n\in \Nat}$,
          step size $\epsilon$,
          number of steps {$L$}
        \STATE \textbf{Output:}
          extended current state $(\vec q, \vec p)$,
          extended initial state $(\vec q_0, \vec p_0)$
        \algrule
          \WHILE{$\vec{q} \not\in \domain{U_{\len{\vec q}}}$}
            \STATE $x \sim \Gau(0,1)$
            \IF{$\len{q} + 1 \in C$}
              \STATE $y \sim \Gau(0,1)$ \hfill\COMMENT{Gaussian for continuous params}
              \STATE $(x_0, y_0) = (x - t \, y, y)$ \hfill\COMMENT{update to current time $t$}
            \ELSE
              \STATE $y_0 \sim \Lap(0,1)$ \hfill\COMMENT{Laplace for discontinuous ones}
              \STATE $(x_0, y_0) = (x - t \, \sign(y), y)$ \hfill\COMMENT{update to current time $t$}
            \ENDIF
            \STATE $(\vec{q_0},\vec{p_0})  = (\vec{q_0} \concat [x_0], \vec{p_0} \concat [y_0])$
            \STATE $(\vec{q}, \vec{p}) = (\vec{q} \concat [x], \vec{p} \concat [y])$ \hfill\COMMENT{increment dimension}
          \ENDWHILE
          \STATE \textbf{return} $((\vec{q},\vec{p}),(\vec{q_0},\vec{p_0}))$
      \end{algorithmic}
    \end{algorithm}
  \end{minipage}
  \hspace{0.01\linewidth}
  \begin{minipage}{0.49\linewidth}
    \begin{algorithm}[H]
      \caption{NP-DHMC Integrator $\NPDisInt$}
      \label{alg:np-dis-hmc-integrator}
      \begin{algorithmic}
        \STATE {\bfseries Input:}
          current state $(\vec{q_0}, \vec{p_0})$,
          family of potential energies $\set{U_n}_{n\in \Nat}$,
          step size $\epsilon$,
          number of steps {$L$},
          discontinuous coordinates $D$
        \STATE \textbf{Output:}
          new state $(\vec q, \vec p)$ computed according to Hamiltonian dynamics,
          extended initial state $(\vec q_0, \vec p_0)$
        \algrule
        \STATE $(\vec q,\vec{p}) = (\vec{q_0}, \vec{p_0})$ \hfill \COMMENT{initialise}
        \STATE $\vec{q}' = \vec{q}_0$
        \STATE $\vec{p}' = \vec{p}_0$
        \STATE $N = \len{\vec{q}_0}$
        \FOR{$i=0$ {\bfseries to} $L$}
          \STATE $\vec{p}_C = \vec{p}_C - \frac{\epsilon}{2}{\nabla_{\vec q_C} U_N}{(\vec q)}$
          \STATE $\vec{q}_C = \vec{q}_C + \frac{\epsilon}{2} {\vec{p}_C}$
          \FOR{$j \in \randomlyPermute(D)$}
            \IF{$j < \len{\vec q}$}
              \STATE \COMMENT{$\len{\vec q}$ may have changed, so must check $j < \len{\vec q}$}
              \STATE $((\vec q, \vec p), (\vec{q'}, \vec{p'})) = $
              \STATE \qquad$\coordIntegrator((\vec q, \vec p), (\vec{q'}, \vec{p'}), j, i\epsilon, \epsilon)$
            \ENDIF
          \ENDFOR
          \STATE $N = \len{\vec{q}}$
          \STATE $\vec{q}_C = \vec{q}_C + \frac{\epsilon}{2} {\vec{p}_C}$
          \STATE $\vec{p}_C = \vec{p}_C - \frac{\epsilon}{2}{\nabla_{\vec q_C} U_N}{(\vec q)}$
        \ENDFOR
        \STATE $\vec{p} = -\vec{p}$
        \STATE {\bfseries return} $((\seqa, \vec{p}), (\vec{q'}, \vec{p'}))$
      \STATE
      \STATE \textbf{function} $\coordIntegrator((\vec q, \vec p), (\vec{q'}, \vec{p'}), j, t, \epsilon)$
        \STATE $\vec{q^*} = \vec q$
        \STATE $q^*_j = q^*_j + \epsilon \sign(p_j)$
        \STATE $((\vec q^*, \vec p^*), (\vec{q'}^*, \vec{p'}^*)) = \extend((\vec q^*, \vec p^*), (\vec{q'}, \vec{p'}), t, U)$
        \STATE $\Delta U = U(\vec{q^*}) - U(\vec q)$
        \IF{$|p_j| > \Delta U$}
          \STATE $(\vec q, \vec p) = (\vec{q^*}, \vec{p^*})$ \hfill\COMMENT{enough kinetic energy to jump}
          \STATE $(\vec{q'}, \vec{p'}) = (\vec{q'}^*, \vec{p'}^*)$
          \STATE $p_j = p_j - \sign(p_j)\Delta U$
        \ELSE
          \STATE $p_j = -p_j$ \hfill\COMMENT{not enough kinetic energy, reflect}
        \ENDIF
        \STATE {\bfseries return} $((\vec q, \vec p), (\vec{q'}, \vec{p'}))$
      \end{algorithmic}
    \end{algorithm}
  \end{minipage}
\end{figure}

\subsubsection{Laplace Momentum and Discontinuous HMC}

The Hamiltonian Monte Carlo method usually uses Gaussian momentum because it corresponds to the physical interpretation of kinetic energy being $\frac12 \sum_i \vec{p}_i^2$ for a momentum vector $\vec p$.
\citet{NishimuraDL20} propose to use Laplace momentum where the kinetic energy for a momentum vector $\vec p$ is given by $\sum_i |\vec{p}_i|$.
This means that the momentum vector must follow a Laplace distribution, denoted as $\Lap(0,1)$, with density proportional to $\prod_i \exp(-|\vec{p}_i|)$.
Hamilton's equations have to be changed to
\[ \frac{\dif \vec q}{\dif t} = \sign(p), \quad \frac{\dif \vec p}{\dif t} = -\nabla_q U. \]
Note that the time derivative of $\vec q$ only depends on the sign of the $p_i$'s.
Hence, if the sign does not change, the change of $\vec q$ can be computed, irrespective of the intermediate values of $U_{\len{\vec q}}(\vec q)$.
The integrator of discontinuous HMC \cite{NishimuraDL20} takes advantage of this for ``discontinuous parameters'', i.e. parameters that $U$ is not continuous in. Thus it can jump through multiple discontinuities of $U$ without evaluating it at every boundary.

We adapt the integrator from \cite{NishimuraDL20} to NP-HMC.
Following them, we assume for simplicity that each coordinate of the position space either corresponds to a continuous or discontinuous parameter, irrespective of which path is chosen.
The set $C$ records all the continuous parameters and $D = \mathbb N \setminus C$ the discontinuous ones.
We use a Gaussian distribution for the continuous parameters of the momentum vector and a Laplace distribution for the discontinuous parameters.
Our integrator updates the continuous coordinates by half a step size just as before, but then the discontinuous ones are updated coordinate by coordinate, a technique called \emph{operator splitting}.
Afterwards, the continuous coordinates are updated by half a step size again.
Algorithm \ref{alg:np-dis-hmc-integrator} contains all the details.

Again, the main difference to the original algorithm is a call to $\extend$.
Note we also have to modify the $\extend$ function itself (given in \cref{alg: extend for np-dhmc}) because some momentum coordinates have to be sampled from a Laplace distribution, and not a Gaussian as before.

\changed[fz]{We also make the following modification: we update the $q_0$ position to current time $t$ instead of $q$ because this avoids having to re-run the probabilistic program.
If we update $q$, a re-run might be necessary if $q$ changed again after an extension, but for $q_0$ this is not the case because the extended part does not affect the weight.}

\subsection{Efficiency Improvements}
\label{sec:efficiency-improvements}

As touched upon in the main text, our implementation includes various performance improvements compared to the pseudocode presentation of NP-HMC.
\begin{asparaenum}[(i)]
\item The $\extend$ function (\cref{alg:extend}) as presented may seem inefficient.
While it terminates almost surely (thanks to Assumption~\aref{ass:3}{3}),
the expected number of iterations may be infinite.
In practice, however, the density function $w$ will arise from a probabilistic program, such as \cref{intro-program}.
Therefore, to evaluate $w$, it would be natural to run the program.
The length of $\vec q$ returned by $\extend$ is exactly the number of \pythoninline{sample} statements encountered during the program's execution.
In particular, if the program has finite expected running time, then the same is true of $\extend$.

\item On top of that, efficient implementations of NP-HMC will interleave the execution of the program with $\extend$, by gradually extending $\vec q$ (if necessary) at every encountered \pythoninline{sample} statement.
This way, $\extend$ increases the running time only by a small constant factor.

\item \changed[fz]{For this to work, we also make the following modification: we update the $q_0$ position to current time $t$ instead of $q$ because this avoids having to re-run the probabilistic program.
If we update $q$, a re-run might be necessary if $q$ changed again after an extension, but for $q_0$ this is not the case because the extended part does not affect the weight.}

\item In a similar vein, we do not have to compute the sum $w_{\le n}(\vec q) = \sum_{k=1}^n \btra(\seqrange{1}{k})$ each time $U_n = -\log w_{\le n}$ is accessed.
By the prefix property, only one of the summands of $w_{\le n}(\vec q)$ is actually nonzero.
Moreover, if $w$ is given by a probabilistic program, then the weight computed during the execution of the program on $\vec q$ is exactly this nonzero summand, assuming that the trace $\vec q$ is long enough for a successful run (which the $\extend$ function ensures).
\item
Another notable way our implementation differs from the algorithm presented above is that it not only extends a trace $\vec q$ in $\extend$ (if necessary), but also trims it (if necessary) to the unique prefix $\vec q'$ of $\vec q$ with positive $w(\vec q')$.
The dimension of $\vec p$ is adjusted accordingly.
This seems to work much better for certain examples, such as the geometric distribution described in \cref{sec:experiements}.
The reason is most likely that the unused suffix (which may have been adapted to the state \changed[lo]{before the current call of $\extend$) is a hindrance when trying to extend to a different state later on.}
\end{asparaenum}

\section{Proof of Correctness}
\label{appendix: correctness}

In this section,
we show that the NP-HMC algorithm is correct,
in the sense that the Markov chain generated by iterating \cref{alg:np-hmc}
converges to the target distribution
$\tdist:{A}\mapsto{\frac{1}{Z} \int_{A} w\ d\tmeasure}$
where
$Z := \int_{\traces}w\ d\tmeasure$.

Henceforth, we assume that \changed[cm]{the density function $\btra$ of the target distribution $\tdist$ is tree-representable and} satisfies Assumptions \aref{ass:1}{1}, \aref{ass:2}{2} and \aref{ass:3}{3}.

\subsection{An Equivalent Algorithm}

We write \cref{alg:np-hmc} as the program \pythoninline{NPHMCstep}
(\cref{alg:np-hmc integrator} as \pythoninline{NPint} and \cref{alg:extend} as \pythoninline{extend}) in \cref{python:np-hmc}.
We present
input sample as \pythoninline{q0};
the density function as \pythoninline{w} and
define potential energy $U$, which is a family of partial functions, as a function \pythoninline{U}, such that
\pythoninline{U(n)} is a partial function denoting $U_n$;
step size as \pythoninline{ep}; and
number of steps as \pythoninline{L}.
We also assume the following primitive functions are implemented:
\pythoninline{normal} is the sampling construct in the language which samples a real number from the standard normal distribution $\Gau_1$.
\pythoninline{domain(f)} gives the domain of the partial function \pythoninline{f}.
\pythoninline{pdfN(x,n)} gives the probability density of \pythoninline{x} on the standard \pythoninline{n}-dimensional normal distribution.
\pythoninline{cdfN(x)} gives the cumulative distribution of \pythoninline{x} on the standard normal distribution.
\pythoninline{grad(f,x)} gives the gradient of the partial function \pythoninline{f} at \pythoninline{x} if defined and \pythoninline{None} if not.

The program \pythoninline{NPHMC} generates a Markov chain on $\traces$ by iterating \pythoninline{NPHMCstep}.

\begin{figure*}
  \hspace{0.025\linewidth}
  \begin{minipage}{0.45\linewidth}
\begin{python}[caption={Python code for \pythoninline{NPHMC}},label={python:np-hmc}]
def extend((q,p),(q0,p0),t,U):
  while q not in domain(U(len(q))):
    x0 = normal
    y0 = normal
    x = x0 + t*y0
    y = y0
    q0.append(x0)
    p0.append(y0)
    q.append(x)
    p.append(y)
  return ((q,p),(q0,p0))

def NPint((q0,p0),U,ep,L):
  q = q0
  p = p0
  for i in range(L):
    p = p - ep/2*grad(U(len(q0)),q)
    q = q + ep*p
    ((q,p),(q0,p0)) =
      extend((q,p),(q0,p0),i*ep,U)
    p = p - ep/2*grad(U(len(q0)),q)
  return ((q,p),(q0,p0))

def NPHMCstep(q0,w,ep,L):
  # initialisation
  p0 = [normal for i in range(len(q0))]
  U = lambda n: lambda q:
    -log(sum([w(q[:i]) for i in range(n)]))
  # NP-HMC integration
  ((q,p),(q0,p0)) = NPint((q0,p0),U,ep,L)
  # MH acceptance
  if cdfN(normal) < accept((q,p),(q0,p0),w):
    return supported(q,w)
  else:
    return supported(q0,w)

def NPHMC(q0,w,ep,L,M):
  S = [q0]
  for i in range(M):
    S.append(NPHMCstep(S[i],w,ep,L))
  return S
\end{python}
\begin{python}[caption={Python code for helper functions},label={python:helper}]
# the MH acceptance ratio
def accept((q,p),(q0,p0),w):
  N = len(q)
  N_trunc = lambda q':
    sum([w(q'[:i]) for i in range(N)])
  weight = (N_trunc(q)*pdfN((q,p),2N))/
           (N_trunc(q0)*pdfN((q0,p0),2N))
  return min(1,weight)

# the w-supported prefix of q
def supported(q,w):
  k = 1
  while w(q[:k]) == 0 and k < len(q):
    k += 1
  return q[:k]
\end{python}
  \end{minipage}
  \hspace{0.025\linewidth}
  \begin{minipage}{0.45\linewidth}
\begin{python}[caption={Python code for \pythoninline{eNPHMC}},label={python:enp-hmc}]
def validstate((q0,p0),U,ep,L):
  q = q0
  p = p0
  for i in range(L):
    p = p - ep/2*grad(U,q)
    q = q + ep*p
    if q not in domain(U):
      return False
    p = p - ep/2*grad(U,q)
  return True

def HMCint((q0,p0),U,ep,L):
  q = q0
  p = p0
  for i in range(L):
    p = p - ep/2*grad(U,q)
    q = q + ep*p
    p = p - ep/2*grad(U,q)
  # momentum flip
  p = -p
  return (q,p)

def eNPHMCstep((q0,p0),w,ep,L):
  # initialisation (step 1)
  q0 = supported(q0,w)
  p0 = [normal for i in range(len(q0))]
  U = lambda n: lambda q:
    -log(sum([w(q[:i]) for i in range(n)]))
  # search (step 2)
  while not validstate((q0,p0),U(len(q0)),ep,L):
    x0 = normal
    y0 = normal
    q0.append(x0)
    p0.append(y0)
  # HMC integration (step 3)
  (q,p) = HMCint((q0,p0),U(len(q0)),ep,L)
  # MH acceptance (step 4)
  if cdfN(normal) < accept((q,p),(q0,p0),w):
    return (q,p)
  else:
    return (q0,p0)

def eNPHMC(q0,w,ep,L,M):
  mc = [(q0,0)]
  for i in range(M):
    mc.append(eNPHMCstep(mc[i],w,ep,L))
  # marginalisation
  S = [supported(q,w) for (q,p) in mc]
  return S
\end{python}
  \end{minipage}
\end{figure*}

Instead of a direct proof, we consider an auxiliary program \pythoninline{eNPHMC} \emph{equivalent} to \pythoninline{NPHMC} \changed[lo]{(in the sense of \cref{prop: we can move all sampling to the top of np-hmc})}, which does not increase the dimension dynamically;
instead it finds the smallest $N$ such that all intermediate positions during the $L$ leapfrog steps stay in the domain of $U_N$,
and performs leapfrog steps as in standard HMC.

The program \pythoninline{eNPHMC} is given in \cref{python:enp-hmc},
which iterates \pythoninline{eNPHMCstep} to generate a Markov chain on \emph{states}
and then marginalise it using the helper function \pythoninline{supported} to obtain a Markov chain on $\traces$.
The program \pythoninline{validstate} determines whether
the input state \pythoninline{(q0,p0)} goes beyond the domain of the potential energy \pythoninline{U} in \pythoninline{L} leapfrog steps, and
the program \pythoninline{HMCint} is the leapfrog integrator of the standard HMC algorithm.

\begin{remark}
  Programs in \cref{python:np-hmc,python:enp-hmc,python:helper} are given in Python syntax, but they can be translated into SPCF.
  First, note we can represent pairs and lists using Church encoding as follows:
  \begin{align*}
    \Pair{\typea,\typeb} & := \typea \to \typeb \to (\typea \to \typeb \to \PCFReal) \to \PCFReal &
    \List{\typea} & := (\typea \to \PCFReal \to \PCFReal) \to (\PCFReal \to \PCFReal) \\
    \anbr{\terma,\termb} & \equiv \lambda z .z\,\terma\,\termb &
    [\terma_1,\dots,\terma_\ell] & \equiv \lambda f x .f\,\terma_1(f\,\terma_2 \dots (f\,\terma_\ell\,\PCF{0}))
  \end{align*}
  Hence a state $(\vec{q},\vec{p}) \in \Real^\ell \times \Real^\ell$ can be encoded as a value
  $[\anbr{\PCF{\vec{q}_1},\PCF{\vec{p}_1}},\dots,\anbr{\PCF{\vec{q}_\ell},\PCF{\vec{p}_\ell}}]$
  with type
  $\List{\Pair{\PCFReal,\PCFReal}}$.

  Now we look at all the primitive functions used in the programs.
  It is easy to see that
  \pythoninline{cdfN}, \pythoninline{pdfN} and \pythoninline{log} are analytic functions.
  \pythoninline{len}, \pythoninline{append} and \pythoninline{sum} can be defined on Church lists.
  \pythoninline{grad} can be defined using the simple numerical differentiation method using analytic functions like subtraction and division.
  We can change \pythoninline{domain} in such a way that
  it takes \pythoninline{q} and \pythoninline{w} as inputs and
  tests whether \pythoninline{sum([w(q[:i]) for i in range(len(q))])} is zero (instead of testing whether \pythoninline{q} is in the domain of \pythoninline{U(len(q))}).

  Now we give a formal definition of equivalence.
  We say two SPCF programs are \defn{equivalent} if
  they induce the same value and weight functions,
  as specified in \cref{appendix: value and weight functions}.
\end{remark}

\begin{proposition}
  \label{prop: we can move all sampling to the top of np-hmc}
  \pythoninline{NPHMC} and \pythoninline{eNPHMC} are equivalent.
\end{proposition}

\begin{proof}
  We give an informal explanation here.

  First note that
  \pythoninline{NPHMCstep} is a Markov process on samples,
  and \pythoninline{eNPHMCstep} on states.
  However, it is easy to see that some minor changes to \pythoninline{NPHMCstep} and \pythoninline{NPHMC}
  make \pythoninline{NPHMCstep} a Markov process on states.
  \changed[lo]{Precisely, the following does not alter the meaning of program \pythoninline{NPHMC}:
  \begin{compactenum}[(1)]
    \item  Given a state \pythoninline{(q0,p0)} in \pythoninline{NPHMCstep}, apply \pythoninline{supported} to \pythoninline{q0} at the start of initialisation and
      return the state \pythoninline{(q0,p0)} or \pythoninline{(q,p)} at the MH acceptance step.
    \item In \pythoninline{NPHMC}, add the marginalisation step just like in \pythoninline{eNPHMC}.
  \end{compactenum}}
  Hence, it is enough to show that all steps in programs \pythoninline{NPHMCstep} and \pythoninline{eNPHMCstep} are equivalent, i.e.~they \changed[cm]{give the same weight and value functions}.

  \lo{@Carol: The preceding statement is problematic: the meaning of ``all steps in programs $A$ and $B$ are operationally equivalent'' is unclear.
  Is the following correct? I think it may be too strong.

  First observe that the input \pythoninline{p0} in \pythoninline{eNPHMCstep} is spurious (in that it is ignored).
  So modulo this input and the modification (1) above, \pythoninline{NPHMCstep} and \pythoninline{eNPHMCstep} are functions of the same type.
  The claim is that, given the same trace (meaning the sequence of random draws from \pythoninline{normal}), \pythoninline{NPHMCstep} and \pythoninline{eNPHMCstep} are extensionally equal, i.e., they define the same function.}

  After the modification,
  \pythoninline{NPHMCstep} and \pythoninline{eNPHMCstep} have the same
  initialisation and MH acceptance step.
  So it remains to show that the NP-HMC integration as described in \pythoninline{NPint} behaves the same as searching for a valid initial state (step 2) and HMC integration (step 3) in \pythoninline{eNPHMCstep}.

  In \pythoninline{NPHMCstep}, \pythoninline{((q,p),(q0,p0)) = NPint((q0,p0),U,ep,L)}
  ``integrates'' from the initial state \pythoninline{(q0,p0)}
  until it goes beyond the domain of \pythoninline{U(len(q0))}, at which moment
  it \pythoninline{extend}s.

  While in \pythoninline{eNPHMCstep}, it increments the dimension of the state \pythoninline{(q0,p0)} until
  it has \emph{just} enough dimension to ``integrate'' for time \pythoninline{ep*L} through \pythoninline{U(len(q0))} without going beyond the domain of \pythoninline{U(len(q0))}.
  This ensures the state \pythoninline{(q0,p0)} is safe to be an input to the standard HMC integrator \pythoninline{HMCint}.

  Notice that
  given the same values for the samples,
  the resulting initial state \pythoninline{(q0,p0)} in \pythoninline{NPHMCstep} would be the same as
  that in \pythoninline{eNPHMCstep}.
  Hence, the proposal state \pythoninline{(q,p)} in both programs would be the same.
\end{proof}

\begin{remark}
  The discussion in the proof of \cref{prop: we can move all sampling to the top of np-hmc} argues informally that \pythoninline{NPHMC} and \pythoninline{eNPHMC} are equivalent.
  We outline a formal proof here.
  To show that \pythoninline{NPHMC} and \pythoninline{eNPHMC} are equivalent, we first demonstrate that one program can be obtained form another by a series of meaning-preserving transformations (i.e.~transformations that preserves the value and weight functions).
  After that we show that the convergence result (\cref{thm: np-hmc converges}) is invariant over equivalent programs.
\end{remark}

Since \pythoninline{NPHMC} and \pythoninline{eNPHMC} are equivalent,
it is enough to show that \pythoninline{eNPHMC} is correct, i.e.~generates a Markov chain that converges to the target distribution.
We present a three-step proof.
\begin{compactenum}[1.]
  \item We first identify the invariant distribution $\sdist$ of the Markov chain $\chain$ generated by iterating \pythoninline{eNPHMCstep}. (\cref{eq: defn of state distribution})
  \item
  We then show that the \emph{marginalised} chain $\set{f\stateterm{i}}_{i\in\Nat}$ is invariant under the target distribution $\tdist$, where
  \changed[fz]{$f(\vec q, \vec p)$ is the unique prefix of $\vec q$ that has positive weight according to $w$.}
  (\cref{thm: marginalised distribution is the target distribution})
  \item Finally, we show this chain converges for a small enough step size $\epsilon$. (\cref{thm: np-hmc converges})
\end{compactenum}

\subsection{Invariant Distribution}

By iterating \pythoninline{eNPHMCstep}, a Markov chain $\chain$ is generated.
We now analyse this Markov chain by studying its invariant distribution $\sdist$ and
transition kernel.

Let $(\states,\Sigma_{\states},\smeasure)$ be the \defn{state space}
where
$\states := \biguplus_{n\in\Nat} (\Real^n\times \Real^n)$,
$\Sigma_\states := \set{\biguplus_{n\in\Nat} U_n \mid U_n \in  \Borel_{2n}}$ and
$\smeasure(\biguplus_{n\in\Nat} U_n) := \sum_{n\in\Nat} (\Gau_n\times \Gau_{n})(U_n)$.
It is easy to see that all output states in \pythoninline{eNPHMCstep}, and hence all elements of the Markov chain, is in $\states$.

However not all states have a positive weight.
In fact not even the union of the support of invariant distributions of the fixed dimension HMC on each of the truncations works.
This is because if \pythoninline{eNPHMCstep} returns $(\vec{q},\vec{p}) \in \Real^{2k}$, then it cannot return states of the form $(\vec{q}\concat \vec{q'},\vec{p} \concat\vec{p'}) \in \Real^{2n}$, which is a valid returning state for the fixed dimension HMC.
Hence we define a subset of states which precisely capture all possible returning states of \pythoninline{eNPHMCstep},
and define a distribution on it.

We say a state $(\vec{q},\vec{p})$ is {$(\epsilon, L)$-\defn{valid} (or simply \defn{valid} whenever the parameters $\epsilon$ and $L$ are clear from the context)} if
a particle starting from the state $(\vec{q},\vec{p})$
does not ``fall beyond'' the domain of $U_{\len{\vec{q}}} := -\log \trunc{\len{\vec{q}}} $
in the course of $L$ discrete leapfrog steps of size $\epsilon$,
and the states $(\vec{q}^{1\dots k},\vec{p}^{1\dots k})$ are not {$(\epsilon, L)$-valid} for all $k < n$.

Let $\validstates$ denote the set of all valid states
and $\validstates_n := \validstates \cap (\Real^n \times \Real^n)$ denote the the set of all $n$-dimension valid states.
The program \pythoninline{validstate} verifies valid states, i.e~
\pythoninline{validstate} always returns True when the input state is valid.

Let $\sdist$ be a distribution on $\states$ with density
$\spdf$ (with respect to $\smeasure$) given by
\begin{align} \label{eq: defn of state distribution}
  \spdf(\vec{q},\vec{p}) :=
  \MyCase{
    \frac{1}{Z} \trunc{\len{\vec{q}}}(\vec{q})
  }{(\vec{q},\vec{p})\in \validstates}{0}
\end{align}
Since the the position component of all valid states must have a $w$-supported prefix,
the set of valid states can be written as
\ifonecolumn
\[
  \validstates =
  \bigcup_{n=1}^{\infty}
  \bigcup_{m=n}^{\infty}
  \{
    (\vec{q}\concat\vec{x},\vec{y}) \in \validstates_m \mid
    \vec{q} \in \nsupport{w}{n},
    \vec{x} \in \Real^{m-n}, \vec{y} \in \Real^m
  \},
\]
\else
\begin{align*}
  \validstates =
  \bigcup_{n=1}^{\infty}
  \bigcup_{m=n}^{\infty}
  \{ &
    (\vec{q}\concat\vec{x},\vec{y}) \in \validstates_m \mid \\
    &\quad
    \vec{q} \in \nsupport{w}{n},
    \vec{x} \in \Real^{m-n}, \vec{y} \in \Real^m
  \},
\end{align*}
\fi
and hence the distribution $\sdist$ can be written as
\ifonecolumn
\begin{align} \label{eq: state distribution}
  \sdist: X & \mapsto
  \expint{X}{
  [(\vec{q},\vec{p}) \in \validstates] \cdot
  \frac{1}{Z} \trunc{\len{\vec{q}}}(\vec{q})}
  {\smeasure}{(\vec{q},\vec{p})}
  =
  \expint{X}{
  [(\vec{q},\vec{p}) \in \validstates] \cdot
  \frac{1}{Z} \sum_{n=1}^{\len{\vec{q}}} w(\vec{q}^{1\dots n})}
  {\smeasure}{(\vec{q},\vec{p})} \nonumber \\
  & =
  \sum_{n=1}^\infty
  \sum_{m=n}^\infty
  \expint{\Real^n}{
  \expint{\Real^{m-n}}{
  \expint{\Real^m}{
  [(\vec{q}\concat\vec{x},\vec{y}) \in X \cap \validstates_m] \cdot
  \frac{1}{Z} w(\vec{q})
  }{\Gau_m}{\vec{y}}
  }{\Gau_{m-n}}{\vec{x}}
  }{\Gau_n}{\vec{q}}
\end{align}
\else
\begin{align} \label{eq: state distribution}
  & \sdist: X \mapsto \nonumber \\
  &
  \expint{X}{
  [(\vec{q},\vec{p}) \in \validstates] \cdot
  \frac{1}{Z} \sum_{n=1}^{\len{\vec{q}}} w(\vec{q}^{1\dots n})}
  {\smeasure}{(\vec{q},\vec{p})} \nonumber \\
  & =\sum_{n=1}^\infty
  \sum_{m=n}^\infty
  \int_{\Real^n}
  \int_{\Real^{m-n}}
  \int_{\Real^m}
  \nonumber\\
  & \qquad\qquad\quad
  [(\vec{q}\concat\vec{x},\vec{y}) \in X \cap \validstates_m] \cdot
  \frac{1}{Z} w(\vec{q})
  \nonumber\\
  & \qquad\qquad\qquad\qquad\quad
  \Gau_m(\dif\vec{y})
  \Gau_{m-n}(\dif\vec{x})
  \Gau_n(\dif\vec{q}).
\end{align}
\fi
We claim that $\sdist$ is the \defn{invariant distribution} of the Markov chain determined by
\pythoninline{eNPHMCstep}.
The rest of this subsection is devoted to a proof of the claim.

For any state $(\vec{q},\vec{p}) \in \states$,
we write $\church{(\vec{q},\vec{p})}$ to be the term $[\anbr{\PCF{\vec{q}_1},\PCF{\vec{p}_1}},\dots,\anbr{\PCF{\vec{q}_{\len{\vec{q}}}},\PCF{\vec{p}_{\len{\vec{q}}}}}]$ of type $\List{\Pair{\PCFReal,\PCFReal}}$.
Take a SPCF term $\terma$ of type
$\set{x:\List{\Pair{\PCFReal,\PCFReal}} } \vdash \terma:\List{\Pair{\PCFReal,\PCFReal}}$.
We define a function $v_{\terma} :\states \times \traces \to \states$ such that
$\church{v_{\terma}({\vec{s}},\trace)} = \valuefn_{\terma[\church{\vec{s}}/x]}(\trace)$.
Then,
the \defn{transition kernel}
$\transkernel{\terma}: \states \times \Sigma_{\states} \rightarrow \states$ of $\terma$ given by
\begin{align*}
  \transkernel{\terma} ({\vec{s}},U) := \shortint{\inv{v_{\terma}({\vec{s}},-)}(U)}{\weightfn_{\terma[\church{\vec{s}}/x]}}{\tmeasure}.
\end{align*}
returns the probability of $\terma$ returning a state in $U$ given the input $\vec{s}$.

We say $\terma$ leaves the distribution $\mu$ on $\states$ invariant if
for all $U \in \Sigma_\states$,
$\expint{\states}{\transkernel{\terma}(\vec{s},U)}{\mu}{\vec{s}}
=
\mu(U).$

\subsubsection{Initialisation and Search (Steps 1 and 2)}

Given $(\vec{q_0},\vec{p_0}) \in \validstates$ and $X\in \Sigma_{\states}$,
where $w(\vec{q_0}^{1\dots n}) > 0$,
the initialisation (step 1) of \pythoninline{eNPHMCstep} returns
a pair of the $w$-supported prefix of $\vec{q_0}$ and a randomly drawn momentum.
Hence, its transition kernel $\transkernel{1}$ is given by
$\transkernel{1}((\vec{q_0},\vec{p_0}),X) := \expint{\traces}{[(\vec{q_0}^{1\dots n},\vec{t}) \in X]}{\tmeasure}{\vec{t}}$.
Note that
$\vec p_0$ (of the input state $(\vec q_0, \vec p_0)$) is ignored by \pythoninline{eNPHMCstep}.

If the input state $(\vec{q_0},\vec{p_0})$ is not a valid state,
we have $\transkernel{1}((\vec{q_0},\vec{p_0}),X) = 0$.
This is required for technical reasons but is excluded in the program \pythoninline{eNPHMCstep} for ease of readability.
At it stands in \cref{python:enp-hmc}, \pythoninline{eNPHMCstep} does not care whether the input state is valid as long as it has a prefix which is $\btra$-supported.
To define such a transition kernel for \pythoninline{eNPHMCstep},
we can simply call \pythoninline{validstate} on the input state at the start of initialisation
and fail this execution if the input state is not valid.

After that,
given $(\vec{q_0},\vec{p_0}) \in \states$ and $X\in \Sigma_{\states}$
where $w(\vec{q_0}^{1\dots n}) > 0$,
step 2 of \pythoninline{eNPHMCstep} searches for a valid state by repeating drawing from the standard normal distribution.
We can write its transition kernel $\transkernel{2}$ as
$\transkernel{2}((\vec{q_0},\vec{p_0}),X) :=
\expint{\traces}{[(\vec{q_0}\concat\vec{t}^{\mathsf{odd}},\vec{p_0} \concat \vec{t}^{\mathsf{even}}) \in X \cap \validstates]}{\tmeasure}{\vec{t}}$
where
$\vec{t}^{\mathsf{odd}}$ and $\vec{t}^{\mathsf{even}}$ are subsequences of $\vec{t}$ containing the values of odd and even indexes respectively.

For any $X \in \Sigma_{\traces}$,
the (combined) transition kernel $\transkernel{1,2}$ of steps 1 and 2 of \pythoninline{eNPHMCstep} is given by
\begin{align*}
  \transkernel{1,2}((\vec{q_0},\vec{p_0}),X)
  & =
  \expint{\traces}{
  \expint{\traces}{
  [(\vec{q_0}^{1\dots n}\concat\vec{t'}^{\mathsf{odd}},\vec{t} \concat \vec{t'}^{\mathsf{even}}) \in X \cap \validstates]
  }{\tmeasure}{\vec{t'}}
  }{\tmeasure}{\vec{t}} \\
  & =
  \expint{\Real^n}{
  \sum_{m=n}^{\infty}
  \expint{\Real^{m-n}}{
  \expint{\Real^{m-n}}{
  [(\vec{q_0}^{1\dots n}\concat\vec{t''},\vec{t} \concat \vec{t'}) \in X \cap \validstates]
  }{\Gau_{m-n}}{\vec{t''}}
  }{\Gau_{m-n}}{\vec{t'}}
  }{\Gau_n}{\vec{t}} \\
  & =
  \sum_{m=n}^\infty
  \expint{\Real^m}{
  \expint{\Real^{m-n}}{
  [(\vec{q_0}^{1\dots n} \concat \vec{x}, \vec{y})\in X\cap \validstates]
  }{\Gau_{m-n}}{\vec{x}}
  }{\Gau_m}{\vec{y}}
\end{align*}
if $(\vec{q_0},\vec{p_0}) \in \validstates$;
and $\transkernel{1,2}((\vec{q_0},\vec{p_0}),X) = 0 $ otherwise.

\begin{proposition}
  \label{prop: step 1 and 2 is probabilistic}
  The transition kernel is probabilistic,
  i.e.~
  $\transkernel{1,2}((\vec{q_0},\vec{p_0}),\states) = \transkernel{1,2}((\vec{q_0},\vec{p_0}),\validstates) = 1$ for any valid state
  $(\vec{q_0},\vec{p_0}) \in \validstates$.
\end{proposition}

\begin{proof}
  Let $(\vec{q_0},\vec{p_0}) \in \validstates$.
  We can see $\transkernel{1,2}((\vec{q_0},\vec{p_0}), -)$ as the value measure of steps 1 and 2 of \pythoninline{eNPHMCstep} (with the initial states substituted by $\church{(\vec{q_0},\vec{p_0})}$) which does not contain $\Score{-}$ as a subterm.
  Moreover, Assumption \aref{ass:3}{3} ensures step 2 almost always terminates and returns a valid state.
  Hence, \cref{prop: AST SPCF term gives probability measure} tells us that $\transkernel{1,2}((\vec{q_0},\vec{p_0}), -)$ is probabilistic and
  $\transkernel{1,2}((\vec{q_0},\vec{p_0}),\states) = \transkernel{1,2}((\vec{q_0},\vec{p_0}),\validstates) = 1$.
\end{proof}


\begin{proposition}
  \label{prop: np-hmc-equivalent step 12 inv}
  $\sdist$ is invariant {with respect to}
  step 1 and 2 of \pythoninline{eNPHMCstep}.
\end{proposition}

\begin{proof}

  We aim to show:
  \(\expint{\states}{\transkernel{1,2}((\vec{q_0},\vec{p_0}),X)}{\sdist}{(\vec{q_0},\vec{p_0})} =
  \sdist(X)\)
  for any measurable set $X \in \Sigma_{\states}$.
  \begin{calculation}
    \displaystyle
    \expint{\states}{\transkernel{1,2}((\vec{q_0},\vec{p_0}),X)}{\sdist}{(\vec{q_0},\vec{p_0})}
    =
    \expint{\validstates}{\transkernel{1,2}((\vec{q_0},\vec{p_0}),X)}{\sdist}{(\vec{q_0},\vec{p_0})}
    \step[=]{
      \cref{eq: state distribution}, definition of $\transkernel{1,2}$ and
      writing $(\vec{q_0},\vec{p_0})\in \validstates$ as $(\vec{q} \concat \vec{x} , \vec{y})$ where
      $\vec{q} \in \support{w}$
    }
    \displaystyle
    \sum_{n=1}^\infty
    \sum_{m=n}^\infty
    \int_{\Real^n}
    \int_{\Real^{m-n}}
    \int_{\Real^m}
    \bigg(
      \sum_{k=n}^\infty
      \expint{\Real^k}{
      \expint{\Real^{k-n}}
      {[(\vec{q} \concat \vec{x'} , \vec{y'})\in X\cap \validstates]}
      {\Gau_{k-n}}{\vec{x'}}
      }
      {\Gau_k}{\vec{y'}}
    \bigg)
    \cdot \\
    \qquad\qquad\qquad\qquad\qquad\qquad\
    \displaystyle
    \bigg(
      [(\vec{q}\concat\vec{x},\vec{y}) \in \validstates]\cdot
      \frac{1}{Z} w(\vec{q})
    \bigg)
    \
    \Gau_{m}(\dif\vec{y})
    \Gau_{m-n}(\dif\vec{x})
    \Gau_{n}(\dif\vec{q})
    \step[=]{ Rearranging (allowed because everything is nonnegative) }
    \displaystyle
    \sum_{n=1}^\infty
    \sum_{k=n}^\infty
    \int_{\Real^n}
    \int_{\Real^{k-n}}
    \int_{\Real^k}
    [(\vec{q} \concat \vec{x'} , \vec{y'})\in X\cap \validstates]\cdot
    \frac{1}{Z} w(\vec{q})\\
    \qquad\quad\
    \displaystyle
    \bigg(
      \sum_{m=n}^\infty
      \expint{\Real^{m-n}}{
      \expint{\Real^m}
      {[(\vec{q}\concat\vec{x},\vec{y}) \in \validstates]}
      {\Gau_{m}}{\vec{y}}
      }
      {\Gau_{m-n}}{\vec{x}}
    \bigg)
    \
    \Gau_k(\dif\vec{y'})
    \Gau_{k-n}(\dif\vec{x'})
    \Gau_{n}(\dif\vec{q})
    \step[=]{Definition of $\transkernel{1,2}$ where $(\hat{\vec{q}},\hat{\vec{p}})$ is an arbitrary valid state such that $\hat{\vec{q}}^{1\dots n} = \vec{q}$}
    \displaystyle
    \sum_{n=1}^\infty
    \sum_{k=n}^\infty
    \expint{\Real^n}{
    \expint{\Real^{k-n}}{
    \expint{\Real^k}{
    [(\vec{q} \concat \vec{x'} , \vec{y'})\in X\cap \validstates]\cdot
    \frac{1}{Z} w(\vec{q})\cdot
    \transkernel{1,2}((\hat{\vec{q}},\hat{\vec{p}}),\validstates)
    \
    }{\Gau_k}{\vec{y'}}
    }{\Gau_{k-n}}{\vec{x'}}
    }{\Gau_{n}}{\vec{q}}
    \step[=]{
      Definition of $\spdf$ and
      \cref{prop: step 1 and 2 is probabilistic} for some valid state $(\hat{\vec{q}},\hat{\vec{p}})$
    }
    \displaystyle
    \shortint{X}{\spdf}{\smeasure}
  \end{calculation}
\end{proof}

\subsubsection{Integration and Acceptance (Steps 3 and 4)}

Let $(\vec{q_0},\vec{p_0}) \in \states$ and $X\in \Sigma_{\states}$.
Now we check that the HMC integration (step 3) and acceptance (step 4) preserve the invariant distribution $\sdist$.

Similar to HMC, the transition kernel for steps 3 and 4 is given by
\[
  \transkernel{3,4}((\vec{q_0},\vec{p_0}),X)
  =
  \MyCase{\alpha(\vec{q_0},\vec{p_0})\cdot[\HMCint_{\len{\vec{q_0}}}(\vec{q_0},\vec{p_0}) \in X] + (1-\alpha(\vec{q_0},\vec{p_0}))\cdot [(\vec{q_0},\vec{p_0}) \in X]}{(\vec{q_0},\vec{p_0}) \in \validstates}
  {0}
\]
where
$\alpha(\vec{q_0},\vec{p_0}) = \min \set{
  1,
  \frac
    {\trunc{N}(\vec{q})\cdot \pdfGau_{2N}(\vec{q},\vec{p}) }
    {\trunc{N}(\vec{q_0})\cdot \pdfGau_{2N}(\vec{q_0},\vec{p_0}) }
}$
for
$N = \len{\vec{q_0}}$ and
$(\vec{q},\vec{p})= \HMCint_{N}(\vec{q_0},\vec{p_0})$.

\begin{proposition}
  \label{prop: hmc integrator is volume preserving and reversible on valid states}
  The HMC integrator ${\HMCint_n}$ with respect to the potential energy $U_n$ is
  volume preserving with respect to $\leb_{2n}$ (i.e.~${\HMCint_n}_*\leb_{2n} = \leb_{2n}$) and
  reversible (i.e.~${\HMCint_n} = \inv{{\HMCint_n}}$) on $\validstates_n$.
\end{proposition}

\begin{proof}
  Since measurable subsets of and states in $\validstates_n$ are also in the $n$-dimension Euclidean Space,
  and $\HMCint_n$ always map valid states to valid states,
  \cref{prop: hmc integrator is volume preserving and reversible} is sufficient.
\end{proof}


\begin{proposition}
  \label{prop: np-hmc-equivalent step 34 inv}
  $\sdist$ is invariant against integration and acceptance (steps 3 and 4) of \pythoninline{eNPHMCstep}.
\end{proposition}

\begin{proof}
  We aim to show:
  $\expint{\states}{\transkernel{3,4}(x,X)}{\sdist}{x} = \sdist(X)$
  for all $X \in \Sigma_{\states}$.
  By \cref{prop: hmc integrator is volume preserving and reversible on valid states},
  for all $n$, HMC integrator $\HMCint_n$ is volume preserving against $\leb_{2n}$ and reversible on $\validstates_n$. Hence, we have
  \begin{align*}
    &\expint{\states}{\transkernel{3,4}(x,X)}{\sdist}{x}
    =
    \expint{\validstates}{\transkernel{3,4}(x,X)}{\sdist}{x}
    =
    \sum_{n=1}^{\infty} \expint{\validstates_n}{\transkernel{3,4}(x,X)\cdot \spdf(x)}{(\Gau_n\times\Gau_{n})}{x} \\
    &=
    \shortint{X}{\spdf}{\smeasure}
    +
    \sum_{n=1}^{\infty}
    \bigg(
      \expint{\validstates_n}{[\HMCint_n(x) \in X \cap \validstates_n] \cdot \alpha(x) \cdot\spdf(x)
      \cdot \pdfGau_{2n}(x)}{\leb_{2n}}{x} \\
    & \qquad\qquad\qquad\qquad\qquad
      -
      \expint{\validstates_n}{[x \in X \cap \validstates_n]\cdot \alpha(x) \cdot\spdf(x)
      \cdot \pdfGau_{2n}(x)}{\leb_{2n}}{x}
    \bigg)
  \end{align*}
  The second and third integrals are the same since
  the pushforward measure of $\leb_{2n}$ along the integrator $\HMCint_n$ is the same as $\leb_{2n}$ ($\HMCint_n$ is volume preserving on $\validstates_n$) for all $n$ and
  $\alpha(x)\cdot\spdf(x)\cdot \pdfGau_{2n}(x) = \alpha(\HMCint_n(x))\cdot\spdf(\HMCint_n(x))\cdot \pdfGau_{2n}(\HMCint_n(x))$ for all $x \in \validstates_n$ (all $\HMCint_n$ are reversible on $\validstates_n$).
\end{proof}

Since the transition kernel $P$ of \pythoninline{eNPHMCstep} is the composition of $\transkernel{1,2}$ and $\transkernel{3,4}$,
i.e.~$P(x,X) := \int_{\states} \transkernel{3,4}(x',X)\ \transkernel{1,2}(x,\dif x') $ for $x \in \states$ and $X \in \Sigma_{\states}$,
and both $\transkernel{1,2}$ and $\transkernel{3,4}$ are invariant against $\sdist$ (\cref{prop: np-hmc-equivalent step 12 inv,prop: np-hmc-equivalent step 34 inv}),
we conclude with the following lemma.

\begin{lemma}
  \label{lemma: pi is the invariant distribution of np-hmc}
  $\sdist$ is the invariant distribution of the Markov chain generated by iterating \pythoninline{eNPHMCstep}.
\end{lemma}

\subsection{Marginalised Markov Chains}

It is important to notice that the Markov chain
$\set{(\vec{q_i},\vec{p_i})}_{i\in\Nat} $
generated by iterating \pythoninline{eNPHMCstep}
with invariant distribution $\sdist$
is \emph{not} the samples we are seeking.
The chain we are in fact interested in is
the \emph{marginalised} chain
$\set{f(\vec{q_i},\vec{p_i})}_{i\in\Nat}$
where the measurable\footnote{For any measurable set $A \in \Sigma_{\traces}$,
  $f^{-1}(A) = \big(\bigcup_{n=1}^\infty \bigcup_{m=n}^\infty ((A \cap \Real^n) \times \Real^{m-n})\times \Real^m\big) \cap \validstates$ is measurable in $\states$.}
function $f$ finds the prefix of $\vec{q}$ which is $w$-supported, formally defined as
\begin{align*}
  {f}: \quad {\validstates} & \longrightarrow{\traces} \\
  {(\vec{q},\vec{p})} & \longmapsto
    \vec{q}^{1\dots n} \quad
    \text{for } \vec{q}^{1\dots n} \in \support{w}.
\end{align*}
This function is realised by the \pythoninline{supported} program in \cref{python:helper}.

In this section we show that this marginalised chain has the target distribution $\tdist$ as its invariant distribution.
Let $Q: \support{w} \times \Sigma_{\traces} \to \pReal$ be the transition kernel of this marginalised chain.
We can write it as
$Q(f(x), A) = P(x, \inv{f}(A))$
for $x \in \validstates$ and $A \in \Sigma_{\traces}$.

\begin{remark}
  In the standard HMC algorithm, the function $f$ would simply be the first projection,
  and it is trivial to check that the pushforward of the invariant distribution along the first projection is exactly the target distribution.
  Hence this step tends to be skipped in the correctness proof of HMC \cite{Neal2011,bou-rabee_sanz-serna_2018}.
\end{remark}

\begin{lemma}
  \label{lemma: relationships between pi and pi_n}
  Writing
  $\validstates_{\leq n} := \bigcup_{k=1}^n \validstates_k$,
  we let $\sdist_n$ be a probability distribution on measurable space
  $(\Real^{2n}, \Borel^{2n},\Gau_{2n})$ given by
  \[
    \sdist_n(X) := \expint{X}{\frac{1}{Z_n} \trunc{n}(\vec{q})}{\Gau_{2n}}{(\vec{q},\vec{p})}
    \qquad
    \text{where }
    Z_n := \shortint{\Real^n}{\trunc{n}}{\Gau_n}
    \text{ and }
    X \in \Borel_{2n}.
  \]
  \begin{compactenum}[(1)]
    \item
      $\sdist(\states \setminus {\validstates_{\leq n}}) \to 0$ as $n\to \infty$.
    \item
      For $m \geq n$,      $Z_n\cdot\sdist_n = Z_m\cdot e^{(m,n)}_* \sdist_m $ on $\validstates_n$
      where
      $e^{(m,n)} : \Real^m \times \Real^{m} \to \Real^n \times \Real^{n}$ with
      $e^{(m,n)}(\vec{q},\vec{p}) = (\vec{q}^{1\dots n} ,\vec{p}^{1\dots n}) $.
    \item
      $Z\cdot\sdist = Z_n\cdot g^{(n)}_*\sdist_n$ on ${\validstates_{\leq n}}$ where
      $g^{(n)} : \Real^n \times \Real^{n}\partialto {\validstates_{\leq n}}$ such that
      $g^{(n)}(\vec{q},\vec{p}) = (\vec{q}^{1\dots k} ,\vec{p}^{1\dots k}) \in {\validstates_{\leq n}}$.
  \end{compactenum}
\end{lemma}

\begin{proof}
  \begin{compactenum}[(1)]
    \item
      $\sdist$ is an invariant distribution,
      and hence it is probabilistic.
      The sum
      $\sum_{n=1}^{\infty} \sdist(\validstates_n)$
      which equals
      $\sdist(\bigcup_{n=1}^{\infty} \validstates_n) = \sdist(\validstates)$ must converge.
      Hence
      $
      \sdist(\states \setminus {\validstates_{\leq n}})
      = \sum_{i=n+1}^{\infty} \sdist(\validstates_i)
      \to 0
      $ as $n\to \infty$.

    \item Simple to show.

    \item Let $X$ be a measurable subset of ${\validstates_{\leq n}}$.
    Then,
    \begin{align*}
      Z\cdot \sdist(X)
      & = \sum_{k=1}^n Z_k \cdot \sdist_k(X\cap \validstates_k) = Z_n \sum_{k=1}^n e^{(n,k)}_* \sdist_n (X\cap \validstates_k) \\
      & = Z_n \cdot \sdist_n (\bigcup_{k=1}^n
          \set{ (\vec{q},\vec{p}) \in \Real^{2n} \mid (\vec{q}^{1\dots k},\vec{p}^{1\dots k}) \in X\cap \validstates_k}) \\
      & = Z_n \cdot g^{(n)}_*\sdist_n (X).
    \end{align*}
  \end{compactenum}
\end{proof}

\invariant*

\begin{proof}
  For any $A \in \Sigma_{\traces}$,
  if (1) $\tdist = f_*\sdist$ on $\traces$
  and (2) $\tmeasure = f_*\smeasure$ on $\support{w}$,
  then
  \begin{align*}
    \tdist(A) & = f_*\sdist(A)
    = \expint{\states}{P(x, \inv{f}(A))}{\smeasure}{x} \tag{\cref{lemma: pi is the invariant distribution of np-hmc}} \\
    & = \expint{\validstates}{P(x, \inv{f}(A))}{\smeasure}{x}
    = \expint{\validstates}{Q(f(x), A)}{\smeasure}{x} \\
    & = \expint{\support{\btra}}{Q(q, A)}{f_*\smeasure}{q}
    = \expint{\support{\btra}}{Q(q, A)}{\tmeasure}{q}
    = \expint{\traces}{Q(q, A)}{\tmeasure}{q}.
  \end{align*}
  Hence it is enough to show (1) and (2).

  \begin{compactenum}[(1)]
    \item
      Let $A \subseteq \Real^n$ be a measurable set on $\traces$
      and $\delta > 0$.
      Then partitioning
      $\inv{f}(A) = \set{(\vec{q},\vec{p}) \in \validstates
        \mid
      \vec{q}^{1\dots n} \in A}$
      using $\validstates_k$, we have {for sufficiently large $m$},
      \begin{align*}
        f_*\sdist(A)
        & = \sdist\left(\bigcup_{k=1}^m \inv{f}(A)\cap \validstates_k\right)
          + \sdist\left(\bigcup_{k=m+1}^\infty \inv{f}(A)\cap \validstates_k\right) \\
        & < \frac{Z_m}{Z} \cdot
          g^{(m)}_*\sdist_m\left(\bigcup_{k=1}^m \inv{f}(A)\cap \validstates_k\right) + \delta
        \tag{by \cref{lemma: relationships between pi and pi_n} (1) and (3)} \\
        & \leq \frac{Z_m}{Z} \cdot
          \sdist_m(A\times\Real^{m-n} \times\Real^m) + \delta \\
        & = \tdist(A) + \delta.
      \end{align*}
      For any measurable set $A \in \Sigma_{\traces}$,
      we have
      $
      f_*\sdist(A)
      =
      \sum_{n=1}^\infty f_*\sdist(A\cap \Real^n)
      \leq
      \sum_{n=1}^\infty \tdist(A\cap \Real^n)
      =
      \tdist(A)$.
      Since both $\tdist$ and $\sdist$ are {probability} distributions, we also have
      $\tdist(A) = 1 - \tdist(\traces \setminus A)
      \leq 1 - f_*\sdist(\traces \setminus A)
      = 1 - (1- f_*\sdist(A))
      = f_*\sdist(A)$.
      Hence
      $f_*\sdist = \tdist$ on $\traces$.
    \item
      Similarly, let $A \subseteq \nsupport{\btra}{n}$ be a measurable set on $\traces$
      and $\delta > 0$.
      Then {for sufficiently large $m$}, we must have
      $\smeasure(\bigcup_{k=m+1}^\infty \validstates_k) = \smeasure(\validstates \setminus \validstates_{\leq m}) < \delta$.
      Hence,
      \begin{align*}
        f_*\smeasure(A)
        & = \smeasure\left(\bigcup_{k=1}^m \inv{f}(A)\cap \validstates_k\right)
          + \smeasure\left(\bigcup_{k=m+1}^\infty \inv{f}(A)\cap \validstates_k\right) \\
        & < \sum_{k=1}^m \Gau_{2k} (\inv{f}(A)\cap \validstates_k) + \delta \\
        & = \sum_{k=1}^m \Gau_{2m} (\set{(\vec{q},\vec{p})\in\Real^{2m} \mid (\vec{q}^{1\dots k},\vec{p}^{1\dots k}) \in \inv{f}(A)\cap \validstates_k}) + \delta \\
        & = \Gau_{2m} (\bigcup_{k=1}^m \set{(\vec{q},\vec{p})\in\Real^{2m} \mid (\vec{q}^{1\dots k},\vec{p}^{1\dots k}) \in \inv{f}(A)\cap \validstates_k}) + \delta \\
        & \leq \Gau_{2m}(A\times\Real^{m-n} \times\Real^m) + \delta \\
        & = \tmeasure(A) + \delta.
      \end{align*}
      Then the proof proceeds as in (1).

  \end{compactenum}
\end{proof}

\subsection{Convergence}
\label{sec:convergence}

Last but not least, we check for the convergence of the marginalised chain to the target distribution $\tdist$.

As shown in \cref{ex: not always ergodic},
it is not trivial that the standard HMC algorithm converges.
The same can be said of the NP-HMC algorithm.
Recall the conditions on the transition kernel to ensure convergence.

\tierney*

Recall $Q$ is the transition kernel of the Markov chain generated by iterating \cref{alg:np-hmc}
on $\support{w}$.
In \cref{thm: marginalised distribution is the target distribution}, we have shown that $Q$ has invariant distribution $\tdist$.
Hence, most of this section is devoted to searching for sufficient conditions (\cref{ass: convergence}) in order to show that the transition kernel $Q$ is $\tdist$-irreducible (\cref{lemma: irreducible}) and aperiodic (\cref{lemma: aperiodic}).
We conclude in \cref{thm: np-hmc converges} that this Markov chain converges to the target distribution $\tdist$.

We start by
extending the result in \cite{CancesLS07}
in two ways:
\begin{compactenum}
  \item
    The density function is only continuously differentiable \emph{almost everywhere}.
  \item
    The position space is the target space $\traces$.
\end{compactenum}

Let $\mathcal{U}$ be the collection of measurable subsets of $\traces$ with the property that their boundary has measure zero.
Formally,
$\mathcal{U} := \set{A \in \Sigma_{\traces} \mid \tmeasure(\boundary{A}) = 0 }$.
Not every set in $\Sigma_{\traces}$ satisfies this property.
A typical example would be the fat Cantor set.
It is easy to see that $\mathcal{U}$ is closed under complementation.
Moreover, for any non-null set $A$ in $\mathcal{U}$,
its interior $\interior{A}$ is non-empty.

We assume the density function $w:\traces \to \pReal$ is continuously differentiable on a non-null set $A \in \mathcal{U}$.
We start by showing that the Markov chain can almost surely move between $\btra$-supported elements in $A$.

\begin{lemma}
  \label{lemma: bthmc reachable in A}
  Assume $w$ is continuously differentiable on a non-null set $A \in \mathcal{U}$ and
  $\set{U_n}$ is uniformly bounded above (i.e.~there is an upper bound $M$, where $U_n(\vec{q}) < M$ for all $\vec{q} \in \domain{U_n}$ for all $n \in \Nat$).
  For \changed[cm]{almost all} $\vec{a}, \vec{b} \in A \cap \support{\btra}$,
  there
  exists some $k \geq \max{\set{\len{\vec{a}},\len{\vec{b}}}}$ and $\vec{p} \in \Real^k$ such that
  $\proj{1}(\HMCint_k(\vec{a}\concat \vec{0}^{1\dots k-\len{\vec{a}}},\vec{p}))^{1\dots \len{\vec{b}}} = \vec{b}$,
  where $\proj{1}(\vec{q},\vec{p}) = \vec{q}$.
\end{lemma}

\begin{proof}
  Define a function $V$ on the sequence space $\Real^\omega$, which is a Fréchet space with a family of semi-norms $\set{\norm{-}_k}_{k\in\Nat}$ where $\norm{\vec{x}}_k = |\vec{x}_k|$, as
  \begin{align*}
    {V}:\quad {\Real^{\omega}} & \longrightarrow{\pReal} \\
    {\vec{x}} & \longmapsto{-\log \sum_{k=1}^\infty w(\vec{x}^{1\dots k}).}
  \end{align*}
  $V$ is well-defined thanks to Assumption \aref{ass:3}{3}.
  Since $w$ is continuously differentiable on $A$,
  $V$ is continuously differentiable on the non-empty open set $\hat{A} := \bigcup_{n=1}^\infty (\interior{A} \cap \Real^n) \times \Real^\omega$.
  Moreover, $V$ must be bounded above, say by some $M$.

  Now we consider the minimization of the function $S_\epsilon: \changed[cm]{(\Real^{\omega})^{L+1}} \to \Real^{\omega}$ \changed[fz]{where $\epsilon$ is the leapfrog step size},
  \[
    (S_\epsilon (\vec{q}^0,\dots,\vec{q}^{L}))_k
    :=
    \epsilon \sum_{i=0}^{L-1} \bigg(
      \frac{1}{2} \Big(
        \frac{\vec{q}_k^{i+1} - \vec{q}_k^{i}}{\epsilon}
      \Big)^2
      -
      \frac{V(\vec{q}^{i+1})+V(\vec{q}^i)}{2}
    \bigg)
    \qquad\text{for all }k\in\Nat
  \]
  where
  $\vec{q}^0 = \vec{a} \concat \vec{0}$ and
  $\vec{q}^L = \vec{b}\concat \vec{0}$.
  Since $V$ is bounded above by $M$,
  for all $\phi \in (\Real^\omega)^{L+1}$,
  each component of ${S_\epsilon(\phi)}\in\Real^{\omega}$ is bounded below by $-\epsilon (L-1)M$
  (i.e.~$\forall k\in\Nat$,
  ${S_\epsilon(\phi)}_k > -\epsilon (L-1)M$).
  Hence, $S_\epsilon$ is bounded below.
  By the completeness of $\Real^\omega$, $\inf{S_\epsilon} \in \Real^\omega$ exists.

  Consider \changed[fz]{a} minimising sequence $\set{\phi_n}_{n\in\Nat} $ on $\changed[cm]{(\Real^\omega)}^{L+1}$ where
  ${S_\epsilon(\phi_{n+1})}_k < {S_\epsilon(\phi_n)}_k$ for all $n,k \in\Nat$ and
  $S_\epsilon(\phi_n) \to \inf{S_\epsilon}$ as $n \to \infty$.
  Writing the sequence as $\set{(\vec{q}^{0,n},\dots, \vec{q}^{L,n})}_{n\in\Nat}$,
  we say it is bounded on $(\Real^\omega)^{L+1}$ if and only if
  for each $i = 0,\dots, L$, $\set{\vec{q}^{i,n}}_{n\in\Nat}$ is a bounded set on $\Real^\omega$ which is equivalent to saying that
  for each $i= 0,\dots, L$ and for all $k \in \Nat$, $\set{\norm{\vec{q}^{i,n}}_k}_{n\in\Nat}$ is bounded on $\Real$.
  It is easy to see that for all $n\in\Nat$ and $i= 1,\dots, L$,
  $\norm{\vec{q}^{i+1,n}-\vec{q}^{i,n}}_k \leq 2\epsilon S_\epsilon(\phi_0) + 2\epsilon^2 LM$ and
  $\norm{\vec{q}^{1,n}}_k \leq 2\epsilon S_\epsilon(\phi_0) + 2\epsilon^2 LM + \norm{\vec{q}^0}_k$,
  so for any $i = 0,\dots, L$ and $k \in \Nat$,
  $\set{\norm{\vec{q}^{i,n}}_k}_{n\in\Nat}$ is bounded
  and hence
  the sequence $\set{\phi_n}_{n\in\Nat}$ is bounded.
  Moreover,
  \changed[cm]{its closure $\Phi := \closure{\set{\phi_n}_{n\in\Nat}}$ is bounded and closed.}

  Note that the Fréchet space $\Real^\omega$ is a quasi-complete nuclear space and has the Heine–Borel property\cm{according to Wikipedia}, i.e.~all closed and bounded set is compact.
  So, the \changed[cm]{set $\Phi$ is compact.}
  Moreover, since $\Real^\omega$ is completely metrisable, the compact \changed[cm]{set $\Phi$} is also sequentially compact, i.e.~\changed[cm]{every sequence in $\Phi$ has a subsequence converging to a point in $\Phi$.}
  Hence $\set{\phi_n}_{n\in\Nat} \subseteq \Phi$ must have a subsequence $\set{\phi_{n_k}}_{k\in\Nat}$ which converges to some point $\bar{\phi}$ in $\Phi$.

  \changed[cm]{
  We claim that $\bar{\phi}$ is almost surely in $\hat{A}^{L+1}$.
  \fz{again, "almost surely" requires a probability measure. Do you mean for almost all a,b, we have $\bar\phi$ in $\hat{A}^{L+1}$?}
  We show that the set $(\Real^\omega)^{L+1} \setminus \hat{A}^{L+1}$ has measure zero.
  First note that by Assumption \aref{ass:2}{2}, $\btra$ is continuously differentiable almost everywhere and hence $\traces \setminus A$ is a null set.
  Moreover, by the definition of $A \in \mathcal{U}$, $\traces \setminus \interior{A}$ is also a null set.
  Then this implies the set of infinite sequences with no prefixes in $\interior{A}$ has measure zero, i.e.~$\Real^\omega\setminus \hat{A}$ is a null set.
  Hence $(\Real^\omega)^{L+1} \setminus \hat{A}^{L+1} = \set{(\vec{q}^0, \dots, \vec{q}^L) \in (\Real^\omega)^{L+1} \mid \exists i\ .\ \vec{q}^i \not\in\hat{A}} = \bigcup_{i=0}^L (\Real^\omega)^{i} \times (\Real^\omega \setminus \hat{A} ) \times (\Real^\omega)^{L-i} $ has zero measure.

  Since $\bar{\phi}$ is constrained by $\vec{q}^0 = \vec{a} \concat \vec{0}$ and $\vec{q}^L = \vec{b} \concat \vec{0}$,
  there can only be a null set of
  $\vec{a}, \vec{b} \in A \cap \support{w}$ which induces $\bar{\phi}$ in the null set $(\Real^\omega)^{L+1} \setminus \hat{A}^{L+1}$.
  Hence $\bar{\phi}$ is almost surely in $\hat{A}^{L+1}$.

  Assume $\bar{\phi}$ is in $\hat{A}^{L+1}$.
  Since $V$ is continuously differentiable on $\hat{A}$, so is $S_\epsilon$ on $\hat{A}^{L+1}$.
  By the continuity of $S_\epsilon$, we have
  $
    \inf{S_\epsilon}
    =
    \lim_{k\to\infty} S_\epsilon(\phi_{n_k})
    =
    S_\epsilon(\lim_{k\to\infty}\phi_{n_k})
    =
    S_\epsilon(\bar{\phi})
  $,
  so $S_\epsilon$ attains its infimum on $\hat{A}^{L+1}$.
  }

  By \cref{prop: R^omega turns on inf}, the gradient of $S_\epsilon$ at its infimum $\bar{\phi} = (\vec{\bar{q}}^0,\dots, \vec{\bar{q}}^L) $ is $\vec{0}$.
  Hence $\vec{\bar{q}}^0 = \vec{a} \concat \vec{0}$,
  $\vec{\bar{q}}^L = \vec{b}\concat \vec{0}$ and
  \[
    \vec{\bar{q}}^{i+1} = 2\vec{\bar{q}}^i - \vec{\bar{q}}^{i-1} - \epsilon^2 \grad{V}(\vec{\bar{q}}^i)
    \qquad
    \text{for }i = 1,\dots, L-1
  \]
  which is the solution to the leapfrog steps.
  In other words, the infimum $\bar{\phi}$ gives a path from $\vec{a}\concat \vec{0}$ to $\vec{b}\concat\vec{0}$ via the leapfrog trajectory with initial momentum
  $\vec{p} = \frac{1}{\epsilon}(\vec{\bar{q}}^1-\vec{a}\concat\vec{0}) + \frac{\epsilon}{2}\grad{V}(\vec{a}\concat\vec{0})$.

  Last but not least, let $k$ be the maximum of $k_i$'s where $w({\vec{\bar{q}}^i}^{1\dots k_i}) > 0$ for all $i = 0,\dots, L$.
  Then it is easy to see that
  $\proj{1}(\HMCint_k(\vec{a}\concat \vec{0}^{1\dots k-\len{\vec{a}}}),\vec{p}^{1\dots k})^{1\dots \len{\vec{b}}} = \vec{b}$.
\end{proof}

\begin{proposition}
  \label{prop: R^omega turns on inf}
  Let $f:\Real^\omega \to \Real^\omega$ be a function with infimum\fz{what does this mean if the range is $\Real^\omega$?}\cm{It is defined below. If $x_0$ is an infimum of $f$, then
  ${f(x)}_{\ell} \geq {f(x_0)}_{\ell}$
  for all $\ell\in\Nat$ and $x \in \Real^\omega$.} at $x_0 \in \Real^\omega$ and
  is continuously differentiable on $A \subseteq \Real^\omega$ where $x_0 \in A$,
  then $\grad{f}(x_0)$ is the zero map, i.e.~$\grad{f}(x_0)(h) = \vec{0}$ for all $h \in \Real^\omega$.
\end{proposition}

\begin{proof}
  First note that $f$ is continuously differentiable at $x_0 \in A$ means that
  for any $\epsilon > 0$ there exists an $\delta > 0$ such that for any $k \in\Nat$ and $x \in \Real^\omega$ such that $\norm{x-x_0}_k < \delta $ implies
  $\frac{\norm{f(x)-f(x_0)-L(x-x_0)}_\ell}{\norm{x-x_0}_k} < \epsilon $ for all $\ell \in\Nat$,
  where $L:\Real^\omega \to \Real^\omega$ is the bounded linear map defined as $L := (Df)(x_0)$.
  \footnote{
    This can be easily seen by substituting $h$ by $\frac{x-x_0}{\norm{x-x_0}_k}$ in the standard definition of continuously differentiable functions $f$ on $A \subseteq \Real^\omega$.
  }

  Assume for contradiction that $L$ is not a zero map.
  i.e.~There exists some $h \in \Real^\omega$ such that $Lh \not= 0$.
  Let $k$ be the coordinate such that ${(Lh)}_k \not= 0$
  and $\epsilon > 0$.

  Since $x_0$ is an infimum of $f$,
  ${f(x)}_{\ell} \geq {f(x_0)}_{\ell}$
  for all $\ell\in\Nat$ and $x \in \Real^\omega$.
  Moreover, $f$ is continuously differentiable at $x_0$ so
  there exists an $\delta > 0$ such that for any $x \in \Real^\omega$, $\norm{x-x_0}_k < \delta $ implies
  $\frac{\norm{f(x)-f(x_0)-L(x-x_0)}_\ell}{\norm{x-x_0}_k} < \epsilon $ for all $\ell \in\Nat$.

  Consider the sequence $\set{y_n}_{n\in\Nat}$ defined as
  $y_n := x_0 - \frac{1}{n}\frac{Lh}{\norm{Lh}_k}\cdot h$.
  The distance between $y_n$ and $x_0$ is
  $
    \norm{y_n-x_0}_{k}
    =
    \bignorm{\frac{-1}{n}\frac{Lh}{\norm{Lh}_k}\cdot h}_{k}
    =
    {\frac{1}{n}\norm{h}_k}.
  $
  So for large enough $n$,
  $\norm{y_n-x_0}_{k} < \delta$.

  Hence,
  \[
    0
    \leq
    \frac{{(f(y_n)-f(x_0))}_k}{\norm{y_n-x_0}_{k}}
    <
    \frac{{L(y_n-x_0)}_k}{\norm{y_n-x_0}_{k}} + \epsilon
    =
    \frac{n}{\norm{h}_{k}}\cdot \bigg(\frac{-1}{n}\frac{(Lh)^2}{\norm{Lh}_k}\bigg)_k + \epsilon
    =
    -\frac{\norm{Lh}_k}{\norm{h}_k} + \epsilon
  \]
  which implies
  $\norm{Lh}_k < \norm{h}_k\epsilon$.
  Since $\epsilon$ is arbitrary, we have $\norm{Lh}_k \leq 0$ which implies $(Lh)_k = 0$ and contradicts our assumption.
\end{proof}

Now we show that the Markov chain can move to any measurable set with positive measure on $A$ from almost all $\btra$-supported element in $A$.

\begin{lemma}
  \label{lemma:bthmc irreducible in A}
  Assuming $w$ is continuously differentiable on a non-null set $A \in \mathcal{U}$ and
  $\set{U_n}$ is uniformly bounded above (i.e.~there is an upper bound $M$, where $U_n(\vec{q}) < M$ for all $\vec{q} \in \domain{U_n}$ for all $n \in \Nat$)
  and
  $\grad{U_n}$ is Lipschitz on $A \cap \domain{U_n}$.
  For \changed[cm]{almost all} $\vec{a} \in A \cap \support{w}$ and measurable subset $B \subseteq A$,
  $\tdist(B) > 0$
  implies
  $Q(\vec{a},B) > 0$.
\end{lemma}

\begin{proof}
  It is enough to prove the statement for
  a non-null measurable set $B \subseteq A \cap \Real^n$ where
  all elements of $B$ have positive weight since
  all measurable subset $B$ of $A$ with $\tdist(B)>0$ must contains such a subset.
  \changed[cm]{Moreover we restrict $B$ to the elements where the statement in \cref{lemma: bthmc reachable in A} always hold w.r.t.~$\vec{a}$.}

  Say $m = \len{\vec{a}}$ and
  $M = \max\set{m,n}$.
  Let
  $I_{\vec{a}}(B)  = \{
  {\vec{p} \in \Real^k} \mid
  {k \geq M}$ and
  all intermediate leapfrog steps starting from $(\vec{a} \concat \vec{0}^{1\dots k-m},\vec{p})\in\validstates$ are in $A \cap \domain{U_k}$ and
  $\proj{1}(\HMCint_k(\vec{a}\concat \vec{0}^{1\dots k-m},\vec{p}))^{1\dots n} \in B\}$.
  It is enough to show that
  $\sum_{k = M}^\infty \leb_k(I_{\vec{a}}(B) \cap \Real^k) > 0$.

  Let $\theta: I_{\vec{a}}(B) \to B$ be the function where
  $\theta(\vec{p})$ gives
  the next sample in $B$ after $L$ HMC leapfrog steps starting with initial state
  $(\vec{a}\concat \vec{0}^{1\dots \len{\vec{p}}-m},\vec{p})$.
  By \cref{lemma: bthmc reachable in A}, $\theta$ is subjective.

  We write $I_{\vec{a}}^k(B) = I_{\vec{a}}(B) \cap \Real^k$ and
  show that $\theta_k: I_{\vec{a}}^k(B) \to B$ is Lipschitz.
  By assumption for any $\vec{p}^0 \in I_{\vec{a}}^k(B)$,
  all the intermediate positions are in $\domain{U_k} \cap A$.
  Hence, we can write $\theta_k(\vec{p}) := \proj{1}(\HMCint_k(\vec{a}\concat \vec{0}^{1\dots k-m},\vec{p})) = \vec{q}^L$
  as
  \[
    \vec{q}^{0} + \epsilon L \vec{p}^{0} -
    \epsilon^2\big(
      \frac{L}{2}\grad{U_k}(\vec{q}^{0}) +
      \sum_{k=1}^{L-1} k \grad{U_k}(\vec{q}^{L-k})
    \big).
  \]
  Let $\vec{p}, \vec{p'} \in I_{\vec{a}}^k(B)$,
  and $\vec{q}^i, \vec{q'}^i$ be the position of the state after $i$ leapfrog steps with momentum kick $\vec{p},\vec{p'}$ respectively. Then,
  \begin{align*}
    |\theta_k(\vec{p}) - \theta_k(\vec{p'})| =
    |\vec{q}^L-\vec{q'}^L|
    & \leq \epsilon L|\vec{p}-\vec{p'}| +
    \epsilon^2 \sum_{i=1}^{L-1}i|\grad{U_k}(\vec{q}^{L-i})-\grad{U_k}(\vec{q'}^{L-i})| \\
    & \leq \epsilon L|\vec{p}-\vec{p'}| +
    \epsilon^2 \sum_{i=1}^{L-1}i|\vec{q}^{L-i}-\vec{q'}^{L-i}| \tag{$U_k$ is Lipschitz on $A \cap \domain{U_k}$}
  \end{align*}
  hence
  $|\theta_k(\vec{p}) - \theta_k(\vec{p'})| \leq c |\vec{p}-\vec{p'}|$ for some constant $c$
  and $\theta_k$ is Lipschitz.

  Assume for contradiction that
  $\sum_{k = M}^\infty \leb_k(I_{\vec{a}}(B) \cap \Real^k) = 0$
  which means that
  for all ${k \geq M}$, $\leb_k(I_{\vec{a}}^k(B)) = 0$.
  However,
  \begin{align*}
    \leb_n(B)
    & =
    \leb_n(\theta(I_{\vec{a}}(B)))
    =
    \leb_n(\theta(\bigcup_{k=M}^\infty I_{\vec{a}}^k(B)))
    =
    \leb_n(\bigcup_{k=M}^\infty \theta_k(I_{\vec{a}}^k(B))) \\
    & \leq
    \sum_{k=M}^\infty \leb_n(\theta_k(I_{\vec{a}}^k(B)))
    \leq
    \sum_{k=M}^\infty \mathsf{Lip}(\theta_k)^{3N}\cdot \leb_n(I_{\vec{a}}^k(B))
    = 0
  \end{align*}
  implies that
  $\leb_n(B) = 0$
  which gives a contradiction.
  \cm{check the Lip measure again}
\end{proof}

\begin{lemma}
  \label{lemma: bthmc reach A}
  Assuming $w$ is continuously differentiable on a non-null set $A$ where $A \in \mathcal{U}$ and
  $\set{\grad{U_n}}$ is uniformly bounded above and below (i.e.~there are bounds $M_1,M_2$, where $M_1 \leq \grad{U_n}(\vec{q}) \leq M_2$ for all $\vec{q} \in \domain{\grad{U_n}}$ for all $n \in \Nat$).
  Then {there exists a step size $\epsilon$ such that}
  for any sequence $\vec{q} \in \support{w}$,
  $\tdist(A) > 0$ implies
  $Q(\vec{q},A) > 0$.
\end{lemma}

\begin{proof}
  Let $\vec{q} \in \Real^m$ be $\btra$-supported.
  Since $A \in \mathcal{U}$, its interior $\interior{A}$ is an non-empty open set.
  Hence for some $n$,
  there is an non-empty open subset $\prod_{i=1}^n (a_i,b_i)$ of $A \cap \Real^n$.

  Now we
  consider the conditions on the starting momentum $\vec{p}^0$ in order for the position $\vec{q}^L$ at the end of the trajectory of the leapfrog steps to be in $A$
  assuming that the position of the intermediate states never leave the domain of $U_k$
  for some $k \geq M := \max\set{m,n}$.
  \begin{align*}
    \vec{q}^L \in \prod_{i=1}^n (a_i,b_i) \times \Real^{k-n}
    & \quad \Leftrightarrow
      \quad \forall i = 1,\dots,n \quad
      \vec{q}_i^{0} + \epsilon L \vec{p}_i^{0} -
      \epsilon^2\big(
        \frac{L}{2}\grad{U_k}(\vec{q}^{0}) +
        \sum_{k=1}^{L-1} k \grad{U_k}(\vec{q}^{L-k})
      \big)
      \in (a_i, b_i) \\
    & \quad \Leftarrow
      \quad \forall i = 1,\dots,n \quad
      \vec{p}_i^0 \in
      \Big(\frac{1}{\epsilon L}(a_i -\vec{q}^{0}_i + \frac{(\epsilon L)^2}{2} M_2),
      \frac{1}{\epsilon L}(b_i -\vec{q}^{0}_i + \frac{(\epsilon L)^2}{2} M_1)\Big) =: I_i
  \end{align*}
  For any $\vec{p} \in \prod_{i=1}^n I_i$,
  the union
  $\bigcup_{k=M}^{\infty} \set{
    \vec{p'}\in \Real^{k-n} \mid
    (\vec{q}\concat \vec{0}^{k-m},\vec{p} \concat \vec{p'})\in \validstates
  } $
  is non-null.
  This is because
  the measure of the union can be seen as the
  value measure of the almost surely terminating probabilistic program
  which given $\vec{q} \in \nsupport{\btra}{m}$ and $\vec{p}\in \prod_{i=1}^n I_i$ returns $\vec{p'}\in \Real^{k-n}$ such that
  $(\vec{q}\concat \vec{0}^{k-m},\vec{p} \concat \vec{p'})$ is a valid state,

  For
  $\epsilon < \frac{1}{L} \sqrt{\frac{2(b_i-a_i)}{M_2-M_1}}$ for all $i$,
  the intervals $\set{I_i}$ are non-empty and hence
  $Q(\vec{q},A)
  \geq \sum_{k=M}^{\infty} \Gau_k(\set{
    \vec{p'}\in \Real^{k-n} \mid
    (\vec{q}\concat \vec{0}^{k-m},\vec{p} \concat \vec{p'})\in \validstates,
    \vec{p} \in \prod_{i=1}^n I_i
  })
  > 0$.
\end{proof}

\begin{definition}
  \label{ass: convergence}
  We gather all the conditions so far.
  \begin{compactenum}[(C1)]
  \item $w$ is continuously differentiable on a non-null set $A$ with measure-zero boundary.
  \item \changed[fz]{$w|_{\support{w}}$} is bounded below by a positive constant.
  \item For each $n$, the function $\frac{\grad{w_{\le n}}}{w_{\le n}}$ is uniformly bounded from above and below on $\support{w_{\le n}} \cap A$.
  \item For each $n$, the function $\frac{\grad{w_{\le n}}}{w_{\le n}}$ is Lipschitz \changed[fz]{continuous} on $\support{w_{\le n}} \cap A$.
  \end{compactenum}
\end{definition}
Note that
\begin{compactenum}[(C1) {implies}]
  \item
    $w$ is continuously differentiable on a non-null set $A \in \mathcal{U}$.
  \item
    $\set{U_n}$ is uniformly bounded above (i.e.~there is an upper bound $M$, where $U_n(\vec{q}) < M$ for all $\vec{q} \in \domain{U_n}$ for all $n \in \Nat$).
  \item
    $\set{\grad{U_n}}$ is uniformly bounded above and below (i.e.~there are bounds $M_1,M_2$, where $M_1 \leq \grad{U_n}(\vec{q}) \leq M_2$ for all $\vec{q} \in \domain{\grad{U_n}}$ for all $n \in \Nat$).
  \item
    $\grad{U_n}$ is Lipschitz on $A \cap \domain{U_n}$.
\end{compactenum}

Now we are ready to prove irreducibility.

\begin{lemma}[Irreducible]
  \label{lemma: irreducible}
  If Assumptions (C1)--(C4) are satisfied,
  {there exists a step size $\epsilon$ such that}
  for any sequence $\vec{q} \in \support{w}$
  and measurable set $B \in \Sigma_{\traces}$,
  $\tdist(B) > 0$ implies
  $Q^i(\vec{q},B) > 0$ for $i \in \set{1,2}$.
\end{lemma}

\begin{proof}
  Let $A$ be the non-null set in $\mathcal{U}$ where
  $\btra$ is continuously differentiable on $A$ \emph{and}
  $\tmeasure(\traces \setminus A) = 0$
  \emph{and}
  \cref{lemma:bthmc irreducible in A} holds for all elements in $A$.
  Such $A$ must exist by Assumption \aref{ass:2}{2} and (C1).

  First note that $\tdist(A \cap B) > 0$.
  Otherwise, we must have $\tdist((\traces \setminus A) \cap B) > 0$. But this implies
  $\tmeasure(\traces \setminus A) \geq \tmeasure((\traces \setminus A) \cap B) > 0$
  which contradicts the assumption.

  We do case analysis on $\vec{q}\in \traces$.
  \begin{itemize}
    \item If $\vec{q} \in A$, then by \cref{lemma:bthmc irreducible in A}, $Q(\vec{q},A \cap B) > 0$.
    \item If $\vec{q} \not\in A$, then
      by \cref{lemma: bthmc reach A}, $Q(\vec{q},A) > 0$ and so
    \begin{align*}
      Q^2(\vec{q},B)
      & \geq Q^2(\vec{q}, A \cap B) \\
      & = \int_{\traces} Q(\vec{q'}, A\cap B)\ Q(\vec{q},\dif\vec{q'}) \\
      & \geq \int_{A} Q(\vec{q'},A \cap B)\ Q(\vec{q},\dif\vec{q'})
      > 0.
    \end{align*}
  \end{itemize}
\end{proof}

\begin{lemma}[Aperiodic]
  \label{lemma: aperiodic}
  If Assumptions (C1)--(C4) are satisfied,
  $Q$ is aperiodic.
\end{lemma}

\begin{proof}
  Assume for contradiction that $Q$ is not aperiodic.
  Then, there exists disjoint $B_0,\dots, B_d$ for $d \geq 1$ such that $\tdist(B_0) > 0$ and
  $x \in B_i$ implies $Q(x,B_{(i+1) \mod (d+1)}) = 1$ for all $i = 0,\dots, d$.

  Let $A$ be the non-null set in $\mathcal{U}$ where
  $\btra$ is continuously differentiable on $A$ \emph{and}
  $\tmeasure(\traces \setminus A) = 0$
  \emph{and}
  \cref{lemma: bthmc reach A} holds for all elements in $A$.
  Such $A$ must exist by Assumption \aref{ass:2}{2} and (C1).
  Let $C_i := B_i \cap A$ for all $i = 0,\dots, d$.
  Hence, $\tdist(C_0) > 0$ and
  $x \in C_i$ implies $Q(x,C_{(i+1) \mod (d+1)}) = 1$ for all $i = 0,\dots, d$.

  Let $x \in C_0$ be a $w$-supported sequence. Such an $x$ must exist as $\tdist(C_0) > 0$.
  Then, $Q(x,C_1) = 1$ implies $Q(x,C_0) \leq Q(x,\traces\setminus C_1) = 0$
  which contradicts with \cref{lemma: irreducible}
  as $x \in A$.

\end{proof}

Finally by Tierney's Theorem (\cref{lemma: Tieryney}), the
$\tdist$-irreducible (\cref{lemma: irreducible}) and
$\tdist$-aperiodic (\cref{lemma: aperiodic})
transition kernel $Q$ with invariant distribution $\tdist$ (\cref{thm: marginalised distribution is the target distribution})
converges to $\tdist$.

\convergence*


\section{Experiments}
\label{appendix: experiments}

\subsection{Details on the Experimental Setup}

For our experimental evaluation, we implemented the algorithms in Python, using PyTorch for tensor and gradient computations.
The source code for our implementation and experiments is available at \url{https://github.com/fzaiser/nonparametric-hmc} and archived as \cite{code}.

\paragraph{Inference algorithms}
The four inference algorithms we compared were:
\begin{enumerate}
\item NP-DHMC (ours): the nonparametric adaptation of \cite{NishimuraDL20}, explained in \cref{sec: hmc variants}, using the efficiency improvements from \cref{sec:efficiency-improvements}.
\item Lightweight Metropolis-Hastings (LMH),
\item Particle Gibbs (PGibbs) and
\item Random walk lightweight Metropolis-Hastings (RMH).

\end{enumerate}
We used the Anglican implementations of the latter three algorithms.

\paragraph{Models}
\changed[fz]{For NP-DHMC, the models were given to the algorithm as probabilistic programs in the form of a Python function with a context argument for NP-DHMC.
The context allows probabilistic primitives and records the trace and weight for the inference algorithms.
This way, evaluating the density function $w$ amounts to running the probabilistic programs.
For LMH, PGibbs, and RMH, the Python models were translated to Clojure programs using Anglican's probabilistic programming constructs.}
The pseudocode for the geometric example and the random walk example can be found in the main text.
The Gaussian and Dirichlet process mixture model is explained there as well, using \changed[lo]{statistical} notation.
\changed[lo]{Sampling from $\mathrm{DP}(\alpha, \mathrm{Uniform}([0,1]^3))$ is implemented using the stick-breaking procedure \cite{Sethuraman94}.}
We use a cutoff of $\epsilon = 0.01$ for the stick size as explained in the text.
In pseudocode, it looks as follows:
\begin{python}
def dp(alpha, H):
    stick = 1.0
    beta = 0.0
    cumulative_product = 1.0
    weights = []
    means = []
    while stick > 0.01:
        cumulative_product *= 1 - beta
        beta = sample(Beta(1, alpha))
        theta = sample(H)
        weights.append(beta * cumulative_product)
        means.append(theta)
        stick -= beta * cumulative_product
    return weights, means
\end{python}

\paragraph{ESS computation}
For the random walk example, we computed the effective sample size.
For this we used NumPyro's \cite{pyro} \pythoninline{diagnostics.effective_sample_size} function.
It is designed to estimate the effective sample size for MCMC samplers using autocorrelation \cite{Gelman2014}. 
\changed[lo]{For importance samples used as the ground truth, we used the importance weights directly to compute the ESS: given importance weights $w_1,\dots,w_n$, the ESS is $\frac{(\sum_{i=1}^n w_i)^2}{\sum_{i=1}^n w_i^2}$.
We also computed the (autocorrelation-based) MCMC ESS for the importance samples and we obtained very similar results.}

\paragraph{Hyperparameter choices}
We produced 10 runs with 1000 samples each for every example except the last, Dirichlet process mixture model (DPMM).
For the DPMM example, we only produced 100 samples in each run because of the forbidding computational cost.
We set the number of burn-in samples that are discarded to 10\% of the total number of samples, i.e. 100 samples for each run.
Since each run of the DPMM only had 100 samples, we set the burn-in higher there, namely to 50.
We did not vary this hyperparameter much because higher values did not seem to make a difference.
For the number of leapfrog steps we tried values $L \in \{5, 20, 50, 100 \}$, and for the step size we tried values $\epsilon \in \{0.01, 0.05, 0.1, 0.5\}$.
Generally, the simple geometric distribution example already works for very rough hyperparameters ($L = 5, \epsilon = 0.1$).
Finer steps work as well, but are not necessary.
However, more complex models generally require finer steps (GMM: $L = 50, \epsilon = 0.05$).
The other inference algorithms we tested don't have any hyperparameters that need to be set.

\paragraph{Thinning}
Since NP-DHMC performs more computation than its competitors for each sample because it evaluates the density function in each of the $L$ leapfrog steps, not just once like the other inference algorithms.
To equalise the computation budgets, we generate $L$ times as many samples for each competitor algorithm, and apply thinning (taking every $L$-th sample) to get a comparable sample size.

\subsection{Additional Plots and Data}


In addition to the ESS and LPPD computations, we also plotted both as a variable of the number of samples computed.
The results can be seen in \cref{fig:walk-ess-plot,fig:gmm-dpmm-lppd-plot}.
\changed[fz]{As we can see, NP-DHMC performs the best consistently over the course of the inference, not just in terms of the final result.}

\begin{figure}
\centering
\includegraphics[width=0.6\textwidth]{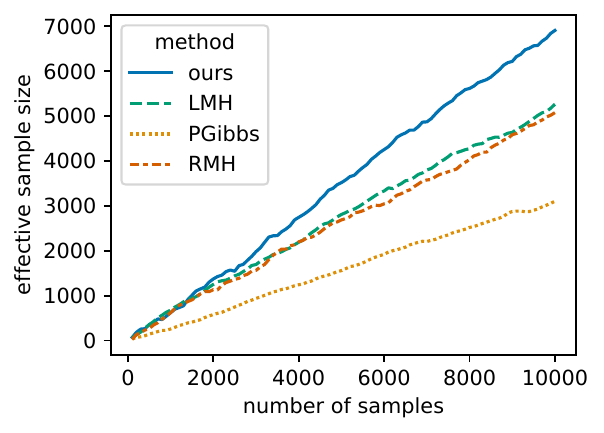}
\caption{ESS for the random walk example in terms of number of samples}
\label{fig:walk-ess-plot}
\end{figure}
\begin{figure}
\centering
\includegraphics[width=0.49\textwidth]{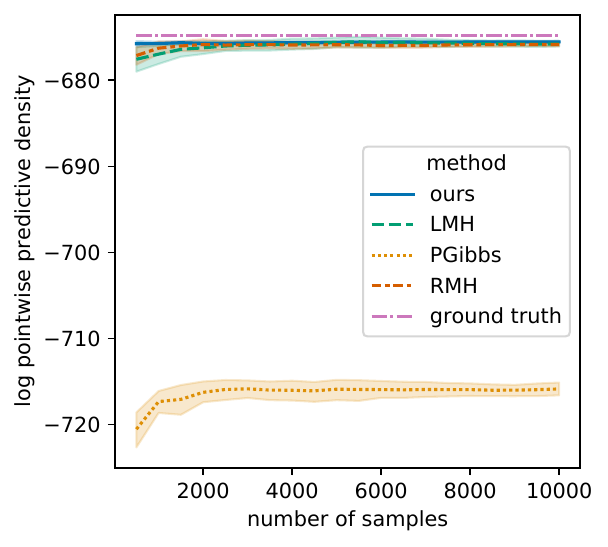}
\includegraphics[width=0.49\textwidth]{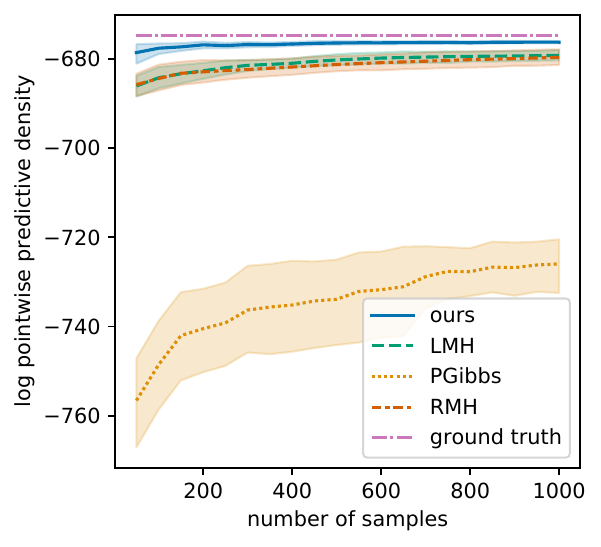}
\caption{LPPD for the GMM and DP mixture model in terms of the number of samples from 10 runs. The shaded area is one standard deviation.
These are the full plots of \cref{fig:gmm-lppd-plot-zoomed,fig:dpmm-lppd-plot-zoomed}, respectively.
}
\label{fig:gmm-dpmm-lppd-plot}
\end{figure}

\paragraph{Running time}
We report the wall-clock times for the different algorithms.
Experiments were carried out on a computer with \cm{@Fabian, are these still true?}\fz{yes} an Intel Core i7-8700 CPU @ 3.20 GHz x 12 and 16 GB RAM, running Ubuntu 20.04.
The results are presented in \cref{table:running-times}.

NP-DHMC is significantly slower than the competition in the geometric and random walk examples, faster for GMM and comparable for DPMM.
Due to the nature of the coordinate integrator of discontinuous HMC \cite{NishimuraDL20},
NP-DHMC has to run the model $L \times d$ times per sample where $d$ is the number of discontinuous variables in the model.
We could improve the algorithm by only updating a subset of the discontinuous variables per iteration.
In addition, NP-DHMC computes gradients and simulates Hamiltonian dynamics, which is computationally expensive.
\changed[fz]{On the random walk example we also ran Pyro HMC and NUTS, as mentioned before.
Both of them were a lot slower than our implementation, which speaks to the fact that HMC methods simply have an unavoidable performance overhead.}
Finally, the implementation of NP-DHMC is a research prototype, so it is not optimal and there is a lot of room for improvement.

\begin{table}
\begin{center}
\footnotesize
\caption{Running times for the different inference algorithms in seconds per sample.}
\label{table:running-times}
\vspace{0.1in}
\begin{tabular}{lllllll}
\toprule
method &ours &LMH &PGibbs &RMH &Pyro HMC &Pyro NUTS \\
\midrule
geometric example &0.0418 &0.0003 &0.0001 &0.0005 &n/a &n/a \\
random walk example &0.2266 &0.0077 &0.0051 &0.0095 &$\approx 0.41$ &$\approx 5.7$ \\
GMM example &0.1879 &1.6572 &1.6835 &1.6376 &n/a &n/a \\
DPMM example &1.8516 &2.1491 &1.7855 &2.0584 &n/a &n/a \\
\bottomrule
\end{tabular}
\end{center}
\end{table}


\else
\fi
\end{document}

